\documentclass[12pt]{colt2021} 

\usepackage{graphics}

 \graphicspath{{./img/}}

\usepackage{amsmath, amssymb, amsfonts}
\usepackage{enumerate}
\usepackage{bbm}
\usepackage[shortlabels]{enumitem}
\usepackage{comment}
\usepackage{upgreek}
\usepackage{dsfont}
 \usepackage{xargs}
\usepackage{nicefrac}

\usepackage{aliascnt}
\usepackage{cleveref}

\newtheorem{assum}{A\hspace{-2pt}}

\newtheorem{assumUE}{UE\hspace{-2pt}}
\newtheorem{assumSUP}{M\hspace{-2pt}}

\crefname{assum}{A\hspace{-2pt}}{A\hspace{-2pt}}
\crefname{assumb}{B\hspace{-2pt}}{B\hspace{-2pt}}
\crefname{assumUGE}{UGE\hspace{-1pt}}{UGE\hspace{-1pt}}
\crefname{assumUE}{UE\hspace{-1pt}}{UE\hspace{-1pt}}
\crefname{assumSUP}{M\hspace{-1pt}}{M\hspace{-1pt}}
\crefname{lemma}{lemma}{lemmas}
\Crefname{lemma}{Lemma}{Lemmas}
\crefname{proposition}{proposition}{propositions}
\Crefname{proposition}{Proposition}{Propositions}
\crefname{corollary}{corollary}{corollaries}
\Crefname{corollary}{Corollary}{Corollaries}
\usepackage[textwidth=4cm,textsize=footnotesize]{todonotes}

\newlist{renumerate}{enumerate}{3}
\setlist[renumerate]{wide, labelwidth=!, labelindent=0pt,label=(\roman*)}

\newlist{aenumerate}{enumerate}{3}
\setlist[aenumerate]{wide, labelwidth=!, labelindent=0pt,label=(\arabic*)}

\newlist{aaenumerate}{enumerate}{3}
\setlist[aaenumerate]{wide, labelwidth=!, labelindent=0pt,label=(\alph*)}

\newlist{aenumerateSpace}{enumerate}{3}
\setlist[aenumerateSpace]{wide, labelwidth=!,label=(\arabic*)}

\newlist{benumerate}{enumerate}{3}
\setlist[benumerate]{wide, labelwidth=!, labelindent=0pt,label=$\bullet$}

\newcommand{\PE}{\mathbb{E}}
\newcommand{\PP}{\mathbb{P}}

\def\Xset{\mathsf{X}}
\def\Xsigma{\mathcal{X}}
\def\Zset{\mathsf{Z}}
\def\Zsigma{\mathcal{Z}}
\def\LL{\mathsf{K}}

\def\rset{\mathbb{R}}

\def\nset{\ensuremath{\mathbb{N}}}
\def\nsets{\ensuremath{\mathbb{N}^*}}

\newcommand{\round}[1]{\ensuremath{\lfloor#1\rfloor}}

\newcommand{\msi}{\mathsf{I}}
\newcommand{\msj}{\mathsf{J}}

\newcommand{\mrl}{\mathrm{L}}

\newcommand{\Const}[1]{\operatorname{C}_{{#1}}}

\newcommand{\bConst}[1]{\bar{\operatorname{C}}_{{#1}}}
\newcommand{\smallConst}[1]{\operatorname{c}_{{#1}}}

\def\Constb{\mathrm{B}}

\newcommandx\sequence[3][2=,3=]
{\ifthenelse{\equal{#3}{}}{\ensuremath{\{ #1_{#2}\}}}{\ensuremath{\{ #1_{#2}, \eqsp #2 \in #3 \}}}}
\newcommandx\sequencet[4]
{\ensuremath{\{ #1{#2_{#3}}, \eqsp #3 \in #4 \}}}
\def\PE{\mathbb{E}}
\def\P{\mathbb{P}}
\def\ProdB{\Gamma}

\def\BlockDet{B}
\def\gPois{\hat{g}}
\def\BlockB{\overline{B}}

\newcommandx{\PVar}[1][1=]{\ensuremath{\operatorname{Var}_{#1}}}

\def\MK{{\rm P}}
\def\MKQ{{\rm Q}}

\newcommand{\abs}[1]{\vert #1\vert}
\newcommand{\absLigne}[1]{\vert #1\vert}

\newcommandx{\norm}[2][2=]{\Vert#1 \Vert_{{#2}}}
\newcommandx{\normop}[2][2=]{\Vert{#1}\Vert_{{#2}}}
\newcommandx{\normopLigne}[2][2=]{\Vert{#1}\Vert_{{#2}}}
\newcommandx{\normopLine}[2][2=]{\Vert{#1}\Vert_{{#2}}}
\newcommandx{\osc}[2][1=]{\mathrm{osc}_{#1}(#2)}

\newcommand{\iid}{i.i.d.}

\newcommandx{\as}[1][1=\PP]{\ensuremath{#1\, -\mathrm{a.s.}}}

\newcommand{\ie}{i.e.}

\newcommand{\eg}{e.g.}

\newcommand{\eqsp}{\;}

\newcommand{\Id}{\mathrm{I}}
\def\ttheta{\tilde{\theta}}
\def\utheta{\tilde{\theta}^{\sf (tr)}}
\def\vtheta{\tilde{\theta}^{\sf (fl)}}

\newcommand{\coint}[1]{\left[#1\right)}
\newcommand{\ocint}[1]{\left(#1\right]}
\newcommand{\ooint}[1]{\left(#1\right)}
\newcommand{\ccint}[1]{\left[#1\right]}
\def\LVset{\operatorname{L}_\infty^V}

\newcommand{\VnormD}[2]{\ensuremath{\left\Vert#1\right\Vert_{#2}}}

\newcommand{\rmi}{\mathrm{i}}
\newcommand{\rme}{\mathrm{e}}
\newcommand{\rmd}{\mathrm{d}}
\def\funcAw{\bar{A}}
\newcommand{\funcA}[1]{\funcAw(#1)}
\def\funcbw{\bar{b}}
\newcommand{\funcb}[1]{\funcbw(#1)}
\def\funnoisew{\bar{\varepsilon}}
\newcommand{\funnoise}[1]{\funnoisew(#1)}

\def\funcAtilde{\widetilde{A}}
\newcommand{\funcAt}[1]{\funcAtilde(#1)}
\newcommand{\frobnorm}[1]{\left\Vert #1 \right\Vert_{\mathrm{F}}}

\def\qcond{\kappa_{\mathsf{Q}}}
\def\myqcond{\kappa}
\def\State{Z}
\def\Stateset{{\sf \State}}

\def\eligibility{\varphi}

\newcommand{\1}{\mathbbm{1}}

\def\expo{\operatorname{t}}
\def\bb{\operatorname{b}}

\def\gain{\mathrm{R}}

\def\msz{\mathsf{Z}}
\def\mcz{\mathcal{Z}}

\def\msc{\mathsf{C}}

\def\msd{\mathsf{D}}

\def\plusinfty{+\infty}

\DeclareMathAlphabet{\mathpzc}{OT1}{pzc}{m}{it}
\def\lyapV{\mathpzc{V}}
\def\lyapW{\mathpzc{W}}
\def\Cf{C_f}
\def\CW{C_{\lyapW}}

\newcommand{\txts}{\textstyle}

\newcommandx\probaMarkovTilde[2][2=]
{\ifthenelse{\equal{#2}{}}{{\widetilde{\mathbb{P}}_{#1}}}{\widetilde{\mathbb{P}}_{#1}\left[ #2\right]}}

\newcommand{\expe}[1]{\PE \left[ #1 \right]}

\newcommand{\expeLigne}[1]{\PE [ #1 ]}

\newcommand{\expeMarkov}[2]{\PE_{#1} \left[ #2 \right]}

\newcommand{\expeMarkovLigne}[2]{\PE_{#1} [ #2 ]}

\newcommand{\Cros}[1]{C_{\mathrm{Ros},#1}}
\newcommand{\Dros}[1]{D_{\mathrm{Ros},#1}}

\def\hg{\gPois}
\def\mcf{\mathcal{F}}

\newcommand{\parenthese}[1]{\left(#1 \right)}
\newcommand{\parentheseLigne}[1]{(#1 )}
\newcommand{\parentheseDeux}[1]{\left[ #1 \right]}
\newcommand{\parentheseDeuxLigne}[1]{[ #1 ]}

\newcommand{\defEnsLigne}[1]{\lbrace #1 \rbrace }

\newcommand{\indi}[1]{\1_{#1}}
\def\cupgamma{c_{\upgamma}}
\def\Rupgamma{R_{\upgamma}}
\def\bupgamma{\bb_{\upgamma}}

\newcommand\cupgammaD[1]{c_{#1}}
\newcommand\RupgammaD[1]{R_{#1}}
\newcommand\tRupgammaD[1]{\tilde{R}_{#1}}
\newcommand\bupgammaD[1]{\bb_{#1}}

\def\tupgamma{\tilde{\upgamma}}
\def\tuptau{\tilde{\uptau}}

\def\Rtuptau{R_{\tuptau}}

\def\rmY{\mathrm{Y}}
\def\bvarepsilon{\bar{\varepsilon}}

\newcommand{\ceil}[1]{\left\lceil #1 \right\rceil}
\def\half{\nicefrac{1}{2}}
\def\tS{\tilde{S}}
\def\bA{\bar{A}}

\title[On the Stability of Random Matrix Product]{On the Stability of Random Matrix Product with Markovian Noise: Application to Linear Stochastic Approximation and TD Learning}
\usepackage{times}

\coltauthor{%
 \Name{Alain Durmus} \Email{alain.durmus@ens-paris-saclay.fr}\\
 \addr ENS Paris-Saclay
 \AND
 \Name{Eric Moulines} \Email{eric.moulines@polytechnique.edu}\\
 \addr Ecole Polytechnique 
 \AND
 \Name{Alexey Naumov} \Email{anaumov@hse.ru}\\
 \addr HSE University 
 \AND
 \Name{Sergey Samsonov} \Email{svsamsonov@hse.ru} \\
 \addr HSE University 
 \AND
 \Name{Hoi-To Wai} \Email{htwai@se.cuhk.edu.hk}\\
 \addr The Chinese University of Hong Kong \thanks{Authors listed in alphabetical order.}
 }

\begin{document}

\maketitle

\begin{abstract}
  This paper studies the exponential stability of  random matrix products driven by a general (possibly unbounded) state space Markov chain. It is a
  cornerstone in the analysis of stochastic algorithms in machine
  learning (\eg\ for parameter tracking in online-learning or reinforcement learning). The existing results impose strong
  conditions  such as uniform boundedness of the matrix-valued functions and
  uniform ergodicity of the Markov chains.  Our main contribution is an exponential stability result
  for the $p$-th moment of random matrix product, provided
  that {\sf (i)} the underlying Markov chain satisfies a
  super-Lyapunov drift condition, {\sf (ii)} the growth of the matrix-valued functions
  is controlled by an appropriately defined function (related to the drift condition).
 Using this result, we give finite-time $p$-th moment bounds for constant and decreasing stepsize linear stochastic approximation schemes with Markovian noise on general state space. We illustrate these findings for linear value-function estimation in reinforcement learning. We provide finite-time $p$-th moment bound for various members of temporal difference (TD) family of algorithms.
 \end{abstract}


\section{Introduction}
Consider the following linear stochastic approximation (LSA) recursion: for  $n \in \nset$,
\begin{equation} \label{eq:lsa}
    \theta_{n+1} = \theta_n + \alpha_{n+1} \{ - \funcA{Z_{n+1}} \theta_n + \funcb{Z_{n+1}} \}\eqsp,
\end{equation}
where $( \alpha_{i})_{i \in \nsets}$ is a sequence of positive step sizes, $\funcAw: \Stateset \rightarrow \rset^{d \times d}$,  $\funcbw: \Stateset \rightarrow \rset^d$ are measurable functions on the state space $\Stateset$, and $( Z_i)_{i \in \nsets}$ is a sequence of random variables on $\Stateset$. The LSA recursion \eqref{eq:lsa} encompasses a wide range of algorithms.  LSA  is central to the analysis of identification algorithms and control of linear systems. Early results have focused on these two applications and studied both the asymptotic behaviour of the sequence $(\theta_n)_{n \in\nset}$ and the tracking error;  see  \cite{eweda:macchi:1983,guo1994stability,guo1995performance,ljung2002recursive} and the references therein.

LSA is also a cornerstone in the analysis of linear value-function estimation (LVE) that are popular in reinforcement learning  \citep{sutton:td:1988,bertsekas:tsitsiklis:96}. Seminal works on this topic \citep{bertsekas:tsitsiklis:96,tsitsiklis:td:1997, benveniste2012adaptive} established conditions for asymptotic convergence. Finite-time bound for LVE (and more generally LSA) has attracted a renewed interest. In the case when $(Z_i)_{i \in \nsets}$ is an i.i.d.~sequence, \citep{lakshminarayanan2018linear,dalal:td0:2017} have investigated mean-squared error bounds for LSA. Recent developments \citep{bhandari2018finite, srikant:1tsbounds:2019,chen2020explicit} have considered the setting that $( Z_i)_{i \in \nsets}$ is a Markov chain, and provided finite-time analysis. On a related subject, \citep{gupta2019finite, xu2019two, doan2019finite, kaledin2020finite} considered linear two-timescale stochastic approximation that involves coupled LSA recursions.

Most of the existing results on LSA are limited by strong conditions such as {\sf (i)} uniform geometric ergodicity ($\mathsf{UGE}$) on the Markov chain and/or {\sf (ii)} uniformly bounded $\funcAw, \funcbw$, \ie~$\sup_{z \in \msz} \{\norm{\funcA{z}} + \norm{\funcbw(z)}\} < \plusinfty$. These conditions are restrictive since the {\sf UGE} condition typically requires the state space to be finite or compact and do not extend to general (unbounded) state space. This is of course a limitation because many applications involve general unbounded state space; see \eg\ \cite{ljung2002recursive} and \cite[p.~305]{bertsekas:tsitsiklis:96}.

In this paper, we aim to provide high-order moment bounds on the LSA with Markovian noise. Our results are applicable under the relaxed conditions: {\sf (i)} $(Z_i)_{i \in \nsets}$ is a  Markov chain on a general (possibly unbounded) state-space satisfying a super-Lyapunov drift condition, and {\sf (ii)} for some constant $C \geq 0$, for any $z \in \msz$, $\norm{\funcA{z}} \leq C \mathrm{W}_1(z)$, $ \norm{\funcb{z}} \leq C \mathrm{W}_2(z)$, with  $\mathrm{W}_1,\mathrm{W}_2 : \rset_+ \to \coint{1,\plusinfty}$  deduced from the drift condition in {\sf (i)}. They are strictly weaker than the conditions required in previously reported works. In particular, $\bA,\funcbw$ can be potentially unbounded.

For $m,n \in \nset$, $m < n$ and $z_{m+1:n} = (z_{m+1},\ldots,z_n)\in\msz^{n-m}$,  we define
\begin{equation} \label{eq:rand_prod}
 \txts \Gamma_{m+1:n}(z_{m+1:n}) = \prod\nonumber_{i=m+1}^n \{\Id_d - \alpha_i \funcA{z_i}\} \eqsp.
 \end{equation}
A key property used for deriving our bounds is an exponential stability result on the matrix product above, $\Gamma_{m+1:n}(Z_{m+1:n})$, for $m,n \in \nset$, $m< n$. To motivate why this is relevant to LSA, suppose that the Markov chain $( Z_n )_{n \in \nsets}$ is ergodic so that, for all $z \in \Zset$, the limits $A = \lim_{n \rightarrow \infty} \PE_z[ \funcA{Z_n} ]$, $b = \lim_{n \rightarrow \infty} \PE_z[ \funcb{Z_n} ]$ exist. Assume that there exists a unique solution $\theta^\star$ to the linear system $A \theta^\star = b$. The $n$-th error vector $\ttheta_{n} = \theta_{n} - \theta^\star$ may be expressed, for all $n \in \nset$, by 
\begin{equation}
\label{eq:LSA-recursion} \txts
\ttheta_{n} = \sum_{j=1}^{n} \alpha_{j} \Gamma_{j+1:n}(Z_{j+1:n}) \funnoise{Z_j} + \Gamma_{1:n} (Z_{1:n}) \ttheta_0  \eqsp,
\end{equation}
where $\funnoise{Z_{j}} = \funcb{Z_{j}} - b - \{ \funcA{Z_{j}} - A \} \theta^\star$.
Obtaining a bound on $p$-th moments for $\{\norm{\ttheta_{n}}\}_{n \in \nset}$ naturally requires that the sequence of random matrices
$\{ \funcA{Z_i} \}_{i \in \nsets}$
to be $(\operatorname{V},q)$-\emph{exponentially stable}.
Recall that for $q \geq 1$ and a function $\operatorname{V}: \Zset \to \coint{1,\infty}$, $\{ \funcA{Z_i} \}_{i \in \nsets}$ is said to be
$(\operatorname{V},q)$-exponentially stable if there exists $\mathsf{a}_q, \Const{q} > 0$ and $\alpha_{\infty,q} < \infty$ such that, for any sequence of positive step sizes $(\alpha_i)_{i \in \nset^*}$ satisfying $\sup_{i \in \nsets} \alpha_i \leq \alpha_{\infty,q}$,  $z \in \Zset$, $m,n \in \nset$, $m < n$,
\begin{equation}
\label{eq:L_V_q-exponential-stability} \txts
\PE_z[ \| \Gamma_{m+1:n}(Z_{m+1:n}) \|^q ] \leq \Const{q} \exp\left( - \mathsf{a}_q \sum_{i=m+1}^n \alpha_i \right) \operatorname{V}(z) \eqsp.
\end{equation}
Intuitively, $(\operatorname{V},q)$-exponential stability means that the $q$-th moment of the product of random matrices $\Gamma_{m+1:n}(Z_{m+1:n})$ behaves similarly to that of the product of \emph{deterministic} matrices $G_{m+1:n} = \prod_{i=m+1}^n (\Id_d - \alpha_i A)$, provided that the matrix $-A$ is Hurwitz (the real parts of its eigenvalues are strictly negative). 

Fix $p,q,r \in \nsets$ such that $p^{-1}= q^{-1} + r^{-1}$.
Assume that the sequence $\{ \funcA{Z_i} \}_{i \in \nsets}$ is
$(\operatorname{V},q)$-exponentially stable for some $q > 1$, the $r$-th moments of the noise term $\norm{\funnoise{Z_n} }$ and initialization error $\ttheta_0$ are bounded. Using \eqref{eq:LSA-recursion}, we can readily derive bounds for the $p$-th moment, $\PE^{1/p}_z[\norm{\ttheta_n}^p]$ by applying the Hölder's inequality. Note that the $r$-th moment bound for the "noise" terms may follow from classical Lyapunov drift conditions, which is implied by super-Lyapunov drift conditions.

\paragraph{Contributions and Organization} The contributions of this paper are three-fold:
\begin{itemize}[leftmargin=5mm, noitemsep]
    \item We establish $(\operatorname{V},q)$-exponential stability of the sequence of matrices $\{ \funcA{Z_k} \}_{k \in \nsets}$, and provide explicit expression for constants appearing in \eqref{eq:L_V_q-exponential-stability}; see \Cref{th:expconvproducts}. Compared to the prior works, our result can be applied to the settings where the function $\funcA{\cdot}$ is unbounded, not symmetric and $(Z_k)_{k \in \nset^*}$ is a Markov chain on a general (unbounded) state-space not constrained to be uniformly geometrically ergodic.
    A discussion of how our results relax the restrictive conditions in previously reported works is given after the statement of \Cref{th:expconvproducts}.
    \item We provide finite-time bound and  first-order expansion for  the $p$-th moment of the error $(\ttheta_n)_{n \in \nsets}$ for LSA recursion \eqref{eq:LSA-recursion}. More precisely, we show that $\PE_z^{1/p}[ \| \ttheta_n \|^p ] = {\cal O}(\alpha_n^{1/2}) \operatorname{V}_p(z)$ both for constant  $\alpha_n \equiv \alpha$ (where $\alpha$ is sufficiently small) or nonincreasing stepsizes under weak additional conditions including $\alpha_n = C / ( n + n_0)^{\expo}$, for any $\expo \in (0,1]$; see \Cref{th:approximation_error}. From our analysis on the LSA error $\ttheta_n$, we identify a leading term, denoted $J_n^{(0)}$, which is a weighted additive linear functional of the error process  $(\funnoise{Z_n})_{n \in \nset^*}$. Furthermore, the leading term $J_n^{(0)}$ and its remainder $H_n^{(0)}= \ttheta_n - J_n^{(0)}$ admit a  separation of scales. For example, when $\alpha_n = C / ( n + n_0)$, the leading term has a $p$-th moment bound of ${\cal O}(n^{-1/2}) \operatorname{V}_p(z)$, and the remainder has a $p$-th moment bound of ${\cal O}(n^{-1} \log(n)) \operatorname{V}_p(z)$; see \Cref{th:approximation_expansion}. 
     
    \item Finally, we apply our results to TD-learning for LVE. We give sufficient conditions for a Markov Reward Process on general (unbounded) state space  (with unbounded reward and feature functions) to satisfy the assumptions of \Cref{th:approximation_error} and \Cref{th:approximation_expansion}. Therefore, the convergence bounds  we derive hold for these algorithms.
\end{itemize}
The rest of this paper is organized as follows. \Cref{sec:main} introduces the formal conditions required for $(\mathrm{V},q)$-exponential stability on $\{\bA(Z_{k})\}_{k \in \nsets}$ and states our main theorem. \Cref{sec:proof} outlines the major steps in the proof. We use this result in \Cref{sec:lsa} to obtain upper bound on the $p$-th moments for the error vector \eqref{eq:LSA-recursion}; finally, we illustrate our results for LVE in TD learning framework.

\paragraph{Notations}
Denote $\nsets = \nset \setminus \{0\}$.
Let $d \in \nsets$ and $Q$ be a symmetric positive definite $d \times d$ matrix.  Denote by $\Id_d$ the $d$-dimensional identity matrix. For $x \in \rset^d$, we denote $\norm{x}[Q]= \{x^\top Q x\}^{\half}$. For brevity, we set $\norm{x}= \norm{x}[\Id_d]$. We denote $\normop{A}[Q]= \max_{\norm{x}[Q]=1} \norm{Ax}[Q]$, and the subscriptless norm $\normop{A} = \normop{A}[\Id]$ is the standard spectral norm.
Let $A_{1},\ldots,A_N$ be $d$-dimensional matrices. We denote $\prod_{\ell=i}^j A_\ell = A_j \ldots A_i$ if $i\leq j$ and with the convention $\prod_{\ell=i}^j A_\ell = \Id_d$ if $i >j$.

Throughout this paper, we let $\Zset$ be a Polish space equipped with sigma-algebra $\mcz$ and fix a measurable function $V: \Zset \to \coint{1,\infty}$. For a measurable function $g: \Zset \rightarrow \rset$, we define its $V$-norm as $\VnormD{g}{V} = \sup_{z \in \Zset} | g(z) | / V(z)$. Furthermore, $\LVset$ denotes the set of all measurable functions $g: \Zset \to \rset$ satisfying $\VnormD{g}{V} < \infty$. Let $\MK : \Zset \times \mcz \to \rset_+$ be a Markov kernel and $V : \Zset \to \rset_+$ be a measurable function, the function $\MK V : \Zset \to \rset_+$ is defined as $\MK V(z) = \int_{\Zset} V(z') \MK( z, {\rm d} z' )$. For a measure $\mu$ on $(\Zset,\mcz)$ and a function $V : \Zset \to \rset_+$ we define $\VnormD{\mu}{V} = \sup_{f: \VnormD{f}{V} \leq 1}\int_{\Zset}f(z)\mu(\rmd z)$.

\section{Main Results} \label[section]{sec:main}

Consider the Markov chain $(Z_k)_{k \in\nset}$.
We assume without loss of generality that $(Z_k)_{k \in\nset}$ is the canonical process corresponding to $\MK$ on $(\msz^{\nset},\mcz^{\otimes \nset})$. We denote by $\PP_{\mu}$ and $\PE_{\mu}$ the corresponding probability distribution and expectation with initial distribution $\mu$. In the case $\mu= \updelta_z$, $z \in \msz$, $\PP_{\mu}$ and $\PE_{\mu}$ are  denoted by $\PP_{z}$ and $\PE_{z}$. In addition, throughout this paper, we assume
\begin{assumUE}
\label{assum:drift}
The  Markov kernel $\MK: \Zset \times \mcz \to \rset_+$ is irreducible and aperiodic. There exist $c > 0, \bb > 0, \delta \in (1/2,1]$, $R_0 \geq 0$,  and $V: \Zset \to \coint{\rme,\infty}$ such that by setting $W = \log V$, $\msc_0 = \{z : W(z) \leq R_0\}$, $\msc_0^{\complement} = \{z : W(z) > R_0\}$, we have
\begin{equation}
\label{eq:drift-condition-improv}
\MK V(z) \le \exp\parentheseDeuxLigne{-c W^{\delta}(z)} V(z)  \indi{\msc_0^{\complement}}(z) + \bb \indi{\msc_0}(z) \eqsp,
\end{equation}
in addition, for any $R \geq 1$, the level sets $\{ z : W(z) \leq R\}$ are $(m_R,\varepsilon_R\nu)$-small for $\MK$, with $m_{R} \in \nsets$, $\varepsilon_R \in \ocint{0,1}$ and $\nu$ being a probability measure on $(\Zset,\mcz)$.
\end{assumUE}
Since $(\Zset,\Zsigma)$ is a general state-space, irreducibility here means that the Markov kernel $\MK$ admits an accessible small set; see \cite[Chapter~9]{douc:moulines:priouret:2018}.
The condition \eqref{eq:drift-condition-improv} in \Cref{assum:drift} is referred to as a multiplicative or super-Lyapunov drift condition and plays a key role in studying the large deviations of additive functionals of Markov chains; see \cite{varadhan:1984,kontoyiannis:meyn:2003,kontoyiannis:meyn:2005}  and the references therein. \Cref{assum:drift} is satisfied for Gaussian linear vector auto-regressive process and also non-linear auto-regressive process under exponential moment condition for innovation process, see e.g. \cite{priouret1998remark}.

Eq.~\eqref{eq:drift-condition-improv} implies the classical Foster-Lyapunov drift condition, $\MK V(z) \le \lambda V(z) + \bb \indi{\msc_0}(z)$ with
\begin{equation}
  \label{eq:def_lambda}
\txts  \lambda = \exp(-c \inf_{\msc^{\complement}_0}W^{\delta}) \leq  \exp(-c) < 1 \eqsp.
\end{equation}
It follows from~\citep[Theorem 15.2.4]{douc:moulines:priouret:2018} that under \Cref{assum:drift}
the Markov kernel $\MK$ is $V$-uniformly geometrically ergodic and admits a unique stationary distribution $\pi$, i.e. there exists $\rho \in \ooint{0,1}$ and $\Constb_{V} < \infty$ such that for each $z \in \Zset$ and $n \in \nset$, \vspace{-.1cm}
\begin{equation} \label{eq:drift_conseq}
\normop{\MK^n(z, \cdot) - \pi}[V]  \leq \Constb_{V}  \rho^n V(z) \eqsp.\vspace{-.1cm}
\end{equation}
We also impose some constraints on $\funcAw$.\vspace{-.1cm}
\begin{assum}\label{assum:almost_bounded}
Given $\varepsilon \in \ooint{0,1}$ there exists $\Const{A}>0$ such that for any $1 \leq i,j \leq d$, the $(i,j)$-th element of $\funcAw$ satisfies $\VnormD{[\funcAw]_{i,j}}{W^\beta} \le \Const{A}$, where $\beta < \min(2\delta-1,\delta/(1+\varepsilon))$ and $\delta$ is given in \Cref{assum:drift}. \vspace{-.2cm}
\end{assum}
\begin{assum}
\label{assum:Hurwitzmatrices}
The square  matrix $-A= -\PE_\pi[\bar A(Z_0)]$ is Hurwitz.\vspace{-.1cm}
\end{assum}
\Cref{assum:almost_bounded}, \Cref{assum:Hurwitzmatrices} are standard conditions on the parameter matrices in LSA.
Under \Cref{assum:Hurwitzmatrices}, there exists a positive definite matrix $Q$ satisfying the Lyapunov equation [cf.~\Cref{lem:lyapunov}]\vspace{-.05cm}
\begin{equation}\label{eq: Lyapunov eq}
A^\top Q + Q A =  \Id_d, \quad \text{and we define} \quad \qcond = \lambda_{\sf min}^{-1}( Q )\lambda_{\sf max}( Q ),~~a= \normop{Q}^{-1}/2.\vspace{-.1cm}
\end{equation}
Consequently, we have $\normop{\Id - \alpha A}[Q] \leq 1 - a \alpha/2$ for $\alpha \in [0, \normop{Q}^{-1} \normop{A}[Q]^{-2}/2]$ [cf.~\Cref{lem:Hurwitzstability}].

Our aim is to establish $(\operatorname{V},q)$-exponential stability of the sequence $\{\funcA{Z_k}\}_{k \in \nsets}$.
The following example illustrates that for the matrix product to be exponentially stable, it is necessary for the Markov chain $(\State_k)_{k \in \nset}$ to be geometrically ergodic.

\begin{example}
\label{ex:counterexample}
Set $\Stateset = \nset^\star$ and consider the forward recurrence time chain on $\Stateset$ starting from $\State_0 = 1$ and defined based on an \iid~sequence $(\rmY_i)_{i \in\nset}$, $\rmY_i \in \Stateset$ by $\State_{k+1}=   \State_k -1$,  if $\State_k >1$ and $  \State_{k+1}= \rmY_{k+1}$, if $\State_k=1$. \citet[Proposition 8.1.5]{douc:moulines:priouret:2018} shows that if
$\PP(\rmY_1=z) > 0$ for $z \in \Stateset$ and $m=\sum_{z \in \Stateset} z \PP(\rmY_1=z)< \plusinfty$, then $(\State_k)_{k \in\nset}$ admits a
unique stationary distribution $\pi$. For any  $\varepsilon >0$, set $A_{\varepsilon}(1)= 1$, and 
$A_{\varepsilon}(z)=-\varepsilon$ for $z \in \Zset \setminus \{1\}$. 
If $\varepsilon \in \ooint{0,\pi(1)}$ then $\sum_{z \in \Stateset} \pi(z) A_\varepsilon(z)= \pi(1) - \epsilon \{1-\pi(1)\} > 0$, so that both conditions \Cref{assum:almost_bounded}, \Cref{assum:Hurwitzmatrices} are satisfied.

Consider the sequence defined recursively as 
$\theta_{n+1}^{\varepsilon}=\{1-\alpha
A_{\varepsilon}(\State_{n+1})\}\theta_{n}^{\varepsilon}$ with $\theta_0^{\varepsilon} > 0$. We show in
\Cref{theo:counterexample}, \Cref{sec:proof:theo:counterexample}, that as $(\State_k)_{k \in \nset}$ is not
geometrically ergodic, for any $\varepsilon \in \ooint{0,\pi(1)}$ and $\alpha \in \ooint{0,1}$, the sequence
$u_n = \expeLigne{\absLigne{\theta_n^{\varepsilon}}} =
\theta_0\expeLigne{\prod_{k=0}^{n-1} \{1-\alpha A_{\varepsilon}(\State_{k+1})\}}$ is
not bounded.
\end{example}


The following theorem establishes the $({\rm V},p)$-exponential stability of the sequence $\{\funcA{Z_k}\}_{k \in \nsets}$. For ease of notation, we simply denote $\Gamma_{m+1:n} = \Gamma_{m+1:n}(Z_{m+1:n})$.
\begin{theorem}
\label[theorem]{th:expconvproducts}
Assume \Cref{assum:drift},  \Cref{assum:almost_bounded} and \Cref{assum:Hurwitzmatrices}.  Then for any $p \geq 1$, there exists $\alpha_{\infty,p} >0$, given in  \eqref{eq:alpha_infty_main}, such that for any  non-increasing sequence $(\alpha_k)_{k \in \nsets}$ satisfying $\alpha_1 \in (0, \alpha_{\infty,p})$,  $z_0 \in \Zset$ and $m,n \in \nset$, $m<n$, it holds
\begin{align}
\label{eq:exponential-stability}
\PE_{z_0}^{1/p}[ \normop{\ProdB_{m+1:n}}^{p} ] \leq \Const{\mathsf{st},p} \rme^{ - (a/4) \sum_{\ell=m+1}^{n}   \alpha_{\ell}}V^{1/2p}(z_0)\eqsp,
\end{align}
where $a$, $\Const{\mathsf{st},p}$, and $h$ are defined in \eqref{eq: Lyapunov eq},    \eqref{eq:C_st_p_proof}, and \eqref{eq:h_def}, respectively.
\end{theorem}
The theorem shows that provided $(\alpha_k)_{k \in\nsets}$ satisfies $\sum_{k \in \nsets} \alpha_k = \plusinfty$, $\PE_z^{1/p}[ \normop{\ProdB_{m+1:n}}^{p} ] \rightarrow 0$ as $(n-m) \to \infty$ for any $p \geq 1$. Specifically, it has a similar convergence rate as the deterministic matrix product $\normop{  G_{m+1:n} } = \normop{  \prod_{i=m+1}^n (\Id_d - \alpha_i A) } \lesssim \rme^{-a \sum_{\ell=m+1}^n \alpha_\ell}$.

\Cref{th:expconvproducts} generalizes previously reported works. 
\citet{guo1994stability,guo1995exponential} used a slightly different definitions allowing to consider  non-Markovian processes satisfying more general mixing conditions (like $\phi$- or $\beta$-mixing). As we will see later, when specialized to Markov chains, the results we obtain significantly improve the results reported in these works. \citet{priouret1998remark} established $(\operatorname{V},q)$-exponential stability for general state-space Markov chain under a super-Lyapunov drift condition (similar to \Cref{assum:drift}). However,  the results in \citet{priouret1998remark} assume constant stepsize and $\bA(z)$ being symmetric and non-negative definite for any $z\in \msz$. Non-negative definiteness plays a key role in the arguments: in such case, for any $z \in \msz$, the spectral norm $\| \Id_d - \alpha \funcA{z} \| \leq 1$ provided that $\| \funcA{z} \| \leq \alpha^{-1}$ for $\alpha > 0$ which is no longer true for general matrix-valued function $\funcA{z}$. Similar results, also under the condition that  $\funcA{z}$ is symmetric for any $z \in \msz$, were obtained by \citet{Delyon_Yuditsky1999} based on perturbation theory for linear operators in Banach space and spectral theory. However, the bounds provided in \citet{Delyon_Yuditsky1999} are only qualitative.  The restrictions imposed on these prior works have limited their applications to more general algorithms, in particular to most RL algorithms. As we will see below, the application to linear value-function estimation in temporal difference learning involve non-symmetric matrix function $\funcAw$. In contrast, our result (cf.~\Cref{th:expconvproducts}) can be applied to the setting where for some $z \in \msz$, $\funcA{z}$ is not necessary non-negative symmetric but only Hurwitz. 

Notice that the case of uniformly geometric ergodic Markov chain is covered by \Cref{assum:drift}. In this case the set $\Zset$ is small and drift function $V$ can be chosen to be constant. Together with the assumption of bounded $\funcA{\cdot}$, the exponential stability of product of random matrices has been implicitly established in \citep{srikant:1tsbounds:2019, doan2019finite, kaledin2020finite, chen2020explicit}. In particular, their results on LSA can be applied on the recursion $y_0 = y$, $y_{n+1} = \{ \Id_d - \alpha_{n+1} \funcA{Z_{n+1}} \} y_n$, $n \in \nset$. Through studying the decomposition: 
\begin{equation} \label{eq:alt_LSA}
    y_{n+1} = \{ \Id_d - \alpha_{n+1} A \} y_n - \alpha_{n+1}  (\funcA{Z_{n+1}} - A) y_n,~\forall~n \in \nset,
\end{equation}
they derived bounds on $\PE_{z_0}[ \normop{ y_{n+1} }^p ] = \PE_{z_0}[ \normop{ \Gamma_{1:n+1} y }^p ]$.
However, generalizing this approach for other classes of Markov chains (e.g., \Cref{assum:drift}) or unbounded function appears to be impossible.
\subsection{Proof of \Cref{th:expconvproducts}} \label{sec:proof}

First note that for any $z_0 \in \msz$, by the Markov property,
\begin{equation}
  \label{proof:main_theo_1_markov_use}
\PE_{z_0}[ \normop{\ProdB_{m+1:n}}^{p}]  = \PE_{z_0}[ \normop{\ProdB_{m+1:n}(Z_{m+1:n})}^{p}]  = \PE_{z_0}[ \PE_{\State_m}[ \normop{\ProdB_{m+1:n}(Z_{1:n-m})}^{p} ] \eqsp.
\end{equation}
The first step is to fix some value $\State_m = z_m \in \msz$ and to derive a bound on $\PE_{z_m}[ \normopLigne{\ProdB_{m+1:n}(Z_{1:n-m})}^{p}]$. We denote by $\myqcond= \qcond^{\half}$ where $\qcond$ is defined in \eqref{eq: Lyapunov eq}.

\paragraph{Step 1: Extracting the deterministic matrix product and a block decomposition}

Consider a block length $h\in\nset$ [to be defined in \eqref{eq:alpha_infty_main}] and define the sequence $j_0 =m, \, j_{\ell+1} = \min(j_\ell + h, n)$ such that $j_{\ell+1} - j_{\ell} \le h$. Let $N= \ceil{(n-m)/h}$, where $\ceil{\cdot}$ is the ceiling function so that $j_{\ell} = j_N =n$ for any $\ell \geq N$.
Then, we introduce the decomposition
\begin{equation}
  \label{eq:decomp_Gamma_proof_main}
 \ProdB_{m+1:n}(Z_{1:n-m}) = \prod_{\ell=1}^N \BlockB_{\ell} \quad \text{where} \quad \BlockB_{\ell} :=
\prod_{i=j_{\ell-1}+1}^{j_\ell} (\Id_d - \alpha_i
\funcA{\State_{i-m}}), ~~\ell \in \{0,\ldots,N\} \eqsp.
  \end{equation}
Using that $(Z_k)_{k \in\nset}$ satisfies
\Cref{assum:drift}, it can be shown that if $m$ is sufficiently large,
then $\BlockB_{\ell}$ is  close in $\mathrm{L}^p$ to the deterministic matrix
$\BlockDet_{\ell} = \prod_{i=j_{\ell-1}+1}^{j_\ell} (\Id_d - \alpha_i
A)$. However, it is not sufficient to conclude because we need to deal
with the product of these terms in
\eqref{eq:decomp_Gamma_proof_main}.
 Therefore, we consider
\begin{equation} \textstyle
    \normop{ \ProdB_{m+1:n}(Z_{1:n-m}) } \leq \myqcond \norm{\ProdB_{m+1:n}(Z_{1:n-m})}[Q] \leq \myqcond \prod_{\ell=1}^N \{\norm{\BlockDet_{\ell}}[Q] + \norm{\BlockDet_{\ell}-\BlockB_{\ell}}[Q] \},
\end{equation}
where the last inequality follows from $\BlockB_\ell = \BlockDet_{\ell} + (\BlockB_{\ell} - \BlockDet_{\ell})$.  Using \Cref{assum:Hurwitzmatrices}, we have
\begin{align*}
    \normop{ \ProdB_{m+1:n}(Z_{1:n-m}) } & \txts \overset{(a)}{\leq} \myqcond \prod_{\ell=1}^N \{\prod_{i=j_{\ell-1}+1}^{j_{\ell}} (1-\alpha_i a/2) + \norm{\BlockDet_{\ell}-\BlockB_{\ell}}[Q] \} \\
    & \txts \overset{(b)}{\leq} \myqcond\{1+ \myqcond\norm{\BlockDet_{N}-\BlockB_{N}}\}  \prod_{\ell=1}^{N-1} \{ (1-\alpha_{j_{\ell}} a/2)^h + \myqcond \norm{\BlockDet_{\ell}-\BlockB_{\ell}} \} \eqsp,
\end{align*}
where (a) is due to \Cref{lem:Hurwitzstability} and  we assumed that $\sup_{i \in \nsets} \alpha_i \leq \alpha_{\infty,p} \leq (1/2)\normopLigne{A}[Q]^{-2}\normopLigne{Q}^{-1}$, and (b) is due to the assumption $\alpha_{i+1} \leq \alpha_i$. Assuming $a \alpha_{\infty,p} \leq 1$ and $h \alpha_{\infty,p} \leq 1$, we get $(1 - a \alpha_{j_\ell}/2)^{-h} \leq \rme^{a}$ since for any $t \in \ccint{0,1/2}$, $(1-t)^{-1} \leq 1+2t \leq \rme^{2t}$, therefore, we obtain
\begin{equation*}
\txts  \norm{  \ProdB_{m+1:n}(Z_{1:n-m}) } \leq \myqcond \prod_{\ell=1}^{N-1} (1 - a \alpha_{j_\ell}/2 )^h \prod_{\ell'=1}^N \{ 1 + \myqcond \rme^{a} \normopLigne{ \BlockB_{\ell'} - \BlockDet_{\ell'} } \}  \eqsp.
\end{equation*}
Taking expectation leads to
 \begin{align}
   \nonumber
&   \txts    \PE_{z_m}^{1/p} [ \norm{  \ProdB_{m+1:n}(Z_{1:n-m}) }^p ] \leq \myqcond \prod_{\ell=1}^{N-1} (1 - a \alpha_{j_\ell} /2)^h \, \PE_{z_m}^{1/p} \Big[ \prod_{\ell'=1}^N \{ 1 + \myqcond \rme^{a} \normopLigne{ \BlockB_{\ell'} - \BlockDet_{\ell'} } \}^p \Big] \\
   \label{eq:after_exp}
   & \txts \qquad \qquad \leq   \myqcond \rme^{a \alpha_{\infty,p}h}  \exp\parentheseLigne{- (a/2) \sum_{i=m+1}^n \alpha_i} \PE_{z_m}^{1/p} [ \prod_{\ell'=1}^N \{ 1 + \myqcond \rme^{a} \normopLigne{ \BlockB_{\ell'} - \BlockDet_{\ell'} } \}^p ]  \eqsp,
\end{align}
since  $\prod_{\ell=1}^{N-1} (1 - a \alpha_{j_\ell}/2 )^h \leq C \rme^{- (a/2) \sum_{i=m+1}^n \alpha_i}$, with $C = \rme^{a \alpha_{\infty,p}h}$, using $\sup_{i \in\nsets} \alpha_i \leq \alpha_{\infty,p}$ and $(\alpha_i)_{i \in \nsets}$ is non-increasing. In order to complete the proof, our next step is to show that the last term in \eqref{eq:after_exp} grows in the order ${\cal O}( \rme^{(a/4) \sum_{i=m+1}^n \alpha_i} )$.

\paragraph{Step 2: Bounding the product of differences $\BlockB_{\ell} - \BlockDet_{\ell}$:} We now tackle the last term in \eqref{eq:after_exp}. Note that for any sequence of square matrices  $\{ C_i \}_{i=1}^N$,
$\prod_{i=1}^n\{ \Id + C_i\} = \sum_{r=0}^N \sum_{(i_1,\ldots,i_r) \in \msj_r} \prod_{k=1}^r  C_{i_k}$, where $\msj_r = \{(i_1,\ldots,i_r) \in \{1,\ldots,N\}^r\, : \, i_1 < \cdots < i_r \}$, with the convention $\prod_{\emptyset} = 1$. Using this expansion, we may therefore decompose the difference $\BlockB_{\ell} - \BlockDet_{\ell}$  as follows:
\begin{equation} \label{eq:split_main}
    \BlockB_{\ell} - \BlockDet_{\ell} = S_\ell + R_\ell - \bar{R}_\ell \eqsp,
\end{equation}
where $S_{\ell} = \sum_{k = j_{\ell-1}+1}^{j_{\ell}}\alpha_{k}\bigl\{\funcA{\State_{k-m}} - A\bigr\}$ is  linear ($r=1$) and the  remainders collect the higher-order terms ($r \geq 2$) in the products
\vspace{-.1cm}
\begin{equation} \label{eq:RlRlbar_def}
    \bar{R}_{\ell} = \sum_{r=2}^{h}(-1)^{r}\sum_{(i_1,\dots,i_r)\in\msi_r^\ell}\prod_{u=1}^{r}\alpha_{i_u}\funcA{\State_{i_u-m}},~ R_{\ell} = \sum_{r=2}^{h}(-1)^{r}\sum_{(i_1,\dots,i_r)\in\msi_r^\ell} \prod_{u=1}^{r}\alpha_{i_u}  A^{r}, \vspace{-.1cm}
\end{equation}
where we have set  $\msi_r^{\ell} = \{(i_1,\ldots,i_r) \in \{j_{\ell-1}+1,\ldots,j_{\ell}\}^r\, : \, i_1 < \cdots < i_r \}$.
Since for $\{ a_i \}_{i=1}^N \subset \rset_+$, $(1 + \sum_{i=1}^N a_i) \leq \prod_{i=1}^N(1 + a_i)$,  the H\"{o}lder's inequality implies
\begin{align}
& \txts\PE_{z_m}^{1/p} \Big[ \prod_{\ell=1}^N \{ 1 + \myqcond \rme^{a} \normopLigne{ \BlockB_{\ell} - \BlockDet_{\ell} } \}^p \Big] \leq
\prod_{\ell=1}^N (1 + \myqcond \rme^{a} \normopLigne{ R_\ell } ) \label{eq:factorofthree_main}
\\
&\qquad \txts \times  \{\PE_{z_m}\bigl[\prod_{\ell=1}^{N}(1+ \myqcond \rme^{a} \normopLigne{\bar{R}_{\ell}})^{2p}\bigr]  \}^{1/(2p)} \{\PE_{z_m}\bigl[\prod_{\ell=1}^{N}(1+ \myqcond \rme^{a} \normopLigne{S_{\ell}})^{2p}\bigr]  \}^{1/(2p)} \eqsp. \notag
\end{align}
Consider first the two terms involving  $\{R_{\ell},\bar{R}_{\ell} \, : \, \ell \in \{1,\ldots,N\}\}$.
From \eqref{eq:RlRlbar_def}, we observe that the order of terms in $R_\ell, \bar{R}_\ell$ is at least quadratic in the step size. As such, a crude estimate suffices to establish that the relevant terms in \eqref{eq:factorofthree_main} grow slowly with $N$ as shown in \Cref{lem:deterministic_part_control,lem:remainder_bound} (postponed to the appendix):
\begin{equation} \label{eq:barRlbd_main}
\txts    \prod_{\ell=1}^N (1 + \myqcond \rme^{a} \normopLigne{ R_\ell } )^p \leq \exp \{ p C^{(0)}  h^2 \sum_{\ell=1}^N \alpha_{j_{\ell-1}+1}^2 \} \eqsp,
\end{equation}
\begin{equation} \label{eq:Rlbd_main}
\txts    \PE_{z_m}\Bigl[\prod_{\ell=1}^{N}(1+ \myqcond \rme^{a} \normopLigne{\bar{R}_{\ell}})^{2p}\Bigr] \leq \PE_{z_m}[ \exp \{ 2p C^{(1)} 2^{h} \sum_{\ell=1}^N \alpha_{j_{\ell-1}+1}^{1+\varepsilon} \sum_{k=j_{\ell-1}+1}^{j_\ell} W^{\delta} (\State_{k-m}) \} ] \eqsp,
\end{equation}
where $C^{(0)}, C^{(1)}$ are defined in \eqref{eq:c_0_definition}, \eqref{eq:c_1_definition}, respectively.
The exponents in \eqref{eq:barRlbd_main}, \eqref{eq:Rlbd_main} are of the order ${\cal O}( \sum_{\ell=1}^N \alpha_{j_{\ell-1}+1}^2 )$, ${\cal O}( \sum_{\ell=1}^N \alpha_{j_{\ell-1}+1}^{1+\varepsilon} )$, respectively, which are desirable for us.

However, similar crude estimates are not sufficient for controlling the last term of \eqref{eq:factorofthree_main} which involves the linear term $S_\ell$. We first apply the following useful bound (of independent interest):
\begin{lemma}
\label[lemma]{lem:trick_2_main}
(\Cref{lem:trick_2}) Let $(\mathfrak F_\ell)_{\ell \geq 0}$ be some filtration and a sequence of non-negative random variables $(\xi_\ell)_{\ell \geq 0}$ which is $(\mathfrak F_{\ell})_{\ell \geq 0}$-adapted.  For any $P \in \nset$, it holds\vspace{-.05cm}
\begin{equation}
\label{eq:cond_expectation_bound_product_main}
\textstyle
 \PE \big[ \prod_{\ell=1}^{P} \xi_{\ell} \big] \le \big\{ \PE \big[ \prod_{\ell=1}^{P} \PE[ \xi_\ell^{2} | \mathfrak F_{\ell-1} ] \big] \big\}^{\half}  \eqsp.\vspace{-.05cm}
\end{equation}
\end{lemma}
By the Markov property, the previous Lemma allow us to write:
\begin{equation} \label{eq:Slbd_main0} \textstyle
\PE_{z_m}^{1/(2p)} \Bigl[\prod_{\ell=1}^{N}\bigl(1 + \myqcond \rme^{a} \normopLigne{S_\ell}\bigr)^{2p}\Bigr] \leq \PE_{z_m}^{1/(4p)} \Big[\prod_{\ell=1}^{N} \PE_{\State_{j_{\ell-1}}}[(1+ \myqcond \rme^{a}\normopLigne{S_\ell})^{4p}] \Big] \eqsp.
\end{equation}
Each of the conditional expectation on the r.h.s.~can be controlled through studying the $p$-th moment of the linear statistics $\PE_{Z_{j_{\ell-1}}} [ \normop{S_\ell}^{4p} ]$. A tight bound can be obtained through applying the Rosenthal's inequalities derived in \Cref{sec:Rosenthal}. Formally, this is done by \Cref{coro:rosenthal_part} in the appendix. Namely, for any $\ell=1,...,N$, it holds
\begin{equation} \label{eq:Slbd_main1}
  \PE_{\State_{j_{\ell-1}}}[(1+ \myqcond \rme^{a}\normopLigne{S_\ell})^{4p}] \leq \exp \big\{4p C_p^{(2)} h^{\half} \alpha_{j_{\ell-1}+1} W^\delta( \State_{j_{\ell-1}}) \big\} \eqsp,
\end{equation}
where $C^{(2)}_p$ is defined in \eqref{eq:c_2_p_definition}.
Note that the exponent on the r.h.s.~has a sublinear growth rate with respect to the block size $h$.
Combining \eqref{eq:barRlbd_main}-\eqref{eq:Rlbd_main}-\eqref{eq:Slbd_main0}-\eqref{eq:Slbd_main1}  lead to the upper bound:
\begin{align}
    \label{eq:conc_step_2_1}
    &\txts \PE^{1/(2p)}_{z_m}[\prod_{\ell=1}^{N} \PE_{\State_{j_{\ell-1}}}[(1+ \myqcond \rme^{a}\normopLigne{\BlockB_{\ell} - \BlockDet_{\ell}})^{2p}] ] \leq \exp \Big\{ C^{(0)} h^2 \sum_{\ell=1}^N \alpha_{j_{\ell-1}+1}^2 \Big\} \cdot T_1  \cdot T_2 \eqsp ,
\end{align}
where $T_1$, $T_2$ are defined as
\begin{align}
    \nonumber
    &\txts  T_1   = \PE_{z_m}^{1/(2p)}[ \exp \{2p  C^{(1)} 2^{h} \sum_{\ell=1}^N \alpha_{j_{\ell-1}+1}^{1+\varepsilon}   \sum_{k=j_{\ell-1}+1}^{j_\ell} W^\delta (\State_{k-m}) \} ] \eqsp , \\
    \nonumber
&\txts T_2 =  \PE_{z_m}^{1/(4p)} [ \exp\{4p C_p^{(2)} h^{\half} \sum_{\ell=1}^N \alpha_{j_{\ell-1}+1} W^\delta( \State_{j_{\ell-1}-m}) \} ]\eqsp.
\end{align}
Constructing an appropriately defined supermartingale (that we deduce from the super-Lyapunov drift condition)  and assuming that $2^{h+1}p C^{(1)}  \alpha_{\infty,p}^{1+\varepsilon} \leq c$, $4p C_p^{(2)} h^{\half} \alpha_{\infty,p}  \leq c$, in \Cref{lem:supermartingale,lem:sceleton_supermartingale} we show that $T_1$, $T_2$ can be bounded by
  \begin{equation}
          \label{eq:conc_step_2_2}
    \begin{aligned}
    T_1 & \textstyle \leq \exp\{ {C}^{(1)} 2^{h} ( \alpha_{\infty,p}^{1+\varepsilon}W(z_m) + \tilde{\bb} h \sum_{\ell=1}^{N}\alpha_{j_{\ell-1}+1}^{1+\varepsilon} ) \} \eqsp, \\
    T_2 & \textstyle \leq \exp\{ {C}_p^{(2)} h^{\half} ( \alpha_{\infty,p}W(z_m) + ( \tilde{\bb} - \log(1-\lambda) ) \sum_{\ell=1}^{N}\alpha_{j_{\ell-1}+1} ) \}\eqsp,
  \end{aligned}
\end{equation}
where $\tilde{\bb} = \log \bb + \sup_{r > 0}\{cr^{\delta}-r\}$.

\paragraph{Step 3: Collecting Terms} The proof is concluded by adjusting the block size and combining upper bounds on $\alpha_{\infty,p}$. The technical details are given in \Cref{subsec:complete-proof}.

\section{Application to Linear Stochastic Approximation}\label[section]{sec:lsa}

This section illustrates how to apply \Cref{th:expconvproducts} to analyze LSA schemes with Markovian noise. First, we state the assumptions on $\funcb{\cdot}$ and step sizes which can be either constant or diminishing. For $\LL \in \nset^*$, consider the following assumption:
\begin{assum}[$\LL$]
\label{assum:funcb}
There exists $\Const{b, \LL} > 0$ such that $\max_{ 1 \leq \ell \leq d} \VnormD{\bar b_\ell}{V^{1/\LL}} \le \Const{b,\LL}$, where $\bar{b}_\ell$ is the $\ell$-th component of $\bar{b}$.
\end{assum}
\begin{assum}
\label{assum:stepsize}
There exists a constant $0<\smallConst{\alpha} \le a/16$ such that for $k \in \nset$,
$ \alpha_{k}/\alpha_{k+1} \le 1 + \alpha_{k+1} \smallConst{\alpha}$.
\end{assum}
It is easy to check that \Cref{assum:stepsize} is satisfied by diminishing step sizes $\alpha_n = \Const{a} (n+n_0)^{-\expo}$, $\expo \in (0,1]$ and constant step sizes.

\begin{theorem}
\label{th:approximation_error}
Let $\LL \geq 8$. Assume \Cref{assum:drift}, \Cref{assum:almost_bounded}, \Cref{assum:Hurwitzmatrices} and \Cref{assum:funcb}($\LL$). 
For any $2 \leq p \leq K/4$, there exists $\alpha_{\infty, p}^{(0)}$ defined in \eqref{eq:definition-alpha-infty-0} such that for  any non-increasing sequence $(\alpha_k)_{k \in \nsets}$ satisfying $\alpha_1 \in (0, \alpha_{\infty, p}^{(0)})$ and~\Cref{assum:stepsize},  $z \in \Zset$, and $n \in \nset$, it holds
\begin{equation} \label{eq:approx_bound_main}
\textstyle{    \PE_z^{1/p}[ \| \ttheta_{n} \|^p ] \leq \operatorname{M}_0 \Const{\mathsf{st},2p} \rme^{ - (a/4) \sum_{\ell=1}^{n}   \alpha_{\ell}} V^{1/(4p)}(z) + (\Const{\mathsf{J},p}^{(0)} + \Const{\mathsf{H},p}^{(0)}) \sqrt{\alpha_n} V^{2/\LL+1/(4p)}(z),
}
\end{equation}
where $\operatorname{M}_0= \PE_z^{1/(2p)} [\| \ttheta_0 \|^{2p}]$ and  $\Const{\mathsf{J},p}^{(0)}, \Const{\mathsf{H},p}^{(0)}$ are defined in \eqref{eq:J0_main}, \eqref{eq:H0_main}, respectively.
\end{theorem}
Most often, the distribution of the initial value $\ttheta_0$  does not depend on the initial value of the Markov chain $z$. In this case $\PE_z^{1/(2p)} [\| \ttheta_0 \|^{2p}]$ is a constant.  With a sufficiently small step size, \Cref{th:approximation_error} shows that the ${\rm L}_p$ norm of error vector converges under  \Cref{assum:drift} for the Markov chain.
Compared to \citep{srikant:1tsbounds:2019}, we consider relaxed conditions on the Markov chain and allow for diminishing step sizes in the LSA.

\paragraph{Finite-time $L_p$ error bound of LSA} \emph{[Proof of \Cref{th:approximation_error}]}\quad
Define the following constraint on the step size
\begin{equation}
\label{eq:definition-alpha-infty-0}
\alpha_{\infty, p}^{(0)}: =\alpha_{\infty, 2p} \wedge \rho \wedge \rme^{-1},
\end{equation}
where $\alpha_{\infty, 2p}$ and $\rho$ are defined in \eqref{eq:alpha_infty_main} and \eqref{eq:drift_conseq} respectively.
Below, we show that the finite-time ${\rm L}_p$ error bound can be derived through applying the stability of random matrix product (see  \Cref{th:expconvproducts}). We recall that the error vector $\ttheta_{n+1} = \theta_{n+1} - \theta^\star$ may be expressed as
\begin{equation} \textstyle
    \ttheta_{n+1} = \ProdB_{1:n+1} \ttheta_0 + \sum_{j=1}^{n+1} \alpha_j \ProdB_{j+1:n+1} \funnoise{Z_j} \equiv \utheta_{n+1} + \vtheta_{n+1} \eqsp.
\end{equation}
Using the H\"{o}lder's inequality and \Cref{th:expconvproducts}, the transient term $\utheta_{n+1}$ can be bounded as follows
{\small \begin{equation}
\label{eq: transient term for theorem}
   \PE_z^{1/p}[\norm{\utheta_{n+1}}^p] \le \PE_z^{1/(2p)}[\normop{\ProdB_{1:n+1}}^{2p}] \PE_z^{1/(2p)}[ \normop{ \ttheta_0 }^{2p} ] \leq \operatorname{M}_0 \Const{\mathsf{st},2p} \rme^{ - (a/4) \sum_{\ell=1}^{n+1}   \alpha_{\ell}} V^{1/(4p)}(z).
\end{equation}}As for the fluctuation term $\vtheta_{n+1}$, it can be verified that $\vtheta_{n+1} = J_{n+1}^{(0)} + H_{n+1}^{(0)}$, where the latter terms are defined by the following pair of recursions:
\begin{equation} \label{eq:jn0_main}
\begin{array}{ll}
J_{n+1}^{(0)} =\left(\Id_d - \alpha_{n+1} A\right) J_{n}^{(0)}+ \alpha_{n+1} \funnoise{\State_{n+1}}, & J_{0}^{(0)}=0, \\[.1cm]
H_{n+1}^{(0)} =\left( \Id_d - \alpha_{n+1} \funcA{\State_{n+1}} \right) H_{n}^{(0)} - \alpha_{n+1} \funcAt{\State_{n+1}} J_{n}^{(0)}, & H_{0}^{(0)}=0 ,
\end{array}
\end{equation}
and $\funcAt{z} = \funcA{z} - A$. Furthermore, we observe that
\begin{equation} \label{eq:jh_recur_main} \textstyle
    J_{n+1}^{(0)} = \sum_{j=1}^{n+1} \alpha_j G_{j+1:n+1} \funnoise{\State_j}, \quad H_{n+1}^{(0)} = - \sum_{j=1}^{n+1} \alpha_j \ProdB_{j+1:n+1} \funcAt{Z_j} J_{j-1}^{(0)} \eqsp.
\end{equation}
From \eqref{eq:jh_recur_main}, we observe that $J_{n+1}^{(0)}$ is an additive functional of $\{ \funnoise{Z_j} \}_{j=1}^{n+1}$ whose $L_p$ norm can be bounded using a Rosenthal-type inequality for Markov chains (see  \Cref{propo:rosenthal_adapt_fort_moulines}). We obtain the following estimate for the function $\bar \varepsilon(\cdot)$ and the coefficients $\alpha_k G_{k+1:n+1}$.
By \Cref{assum:almost_bounded}, \Cref{assum:funcb}($\LL$), we have
\begin{equation} \txts
    \max_{\ell \in \{1, \ldots, d\}} \VnormD{\bar \varepsilon_\ell}{V^{1/\LL}} \le \Const{\bar \varepsilon} = \sqrt d\Const{b, \LL} + 2 d (\beta \LL/\rme)^\beta\Const{A} \norm{\theta^\star}.
\end{equation}
From \Cref{assum:Hurwitzmatrices}, we recall that $\normop{G_{k+1:n+1}} \le \myqcond \prod_{\ell=k+1}^{n+1} \sqrt{1 - a \alpha_\ell}$ [cf.~\Cref{lem:Hurwitzstability}].
Together with \Cref{assum:stepsize}, this implies that \vspace{-.2cm}
\begin{equation} \txts
\normop{\alpha_k G_{k+1:n+1} - \alpha_{k+1} G_{k+2:n+1}} \le \myqcond (\smallConst{\alpha}  + 2 \normop{A} ) \alpha_{k+1}^2 \prod_{\ell=k+1}^{n+1}\sqrt{1 - a \alpha_\ell} \eqsp, \vspace{-.2cm}
\end{equation}
By \Cref{assum:stepsize}, we also have $\alpha_1 \normop{G_{2:n+1}} \le  \myqcond \alpha_{n+1} \prod_{j=2}^{n+1} (1 + \smallConst{\alpha} \alpha_{j})(1 - a \alpha_j/2) \le \myqcond \alpha_{n+1}$. We can now apply the Rosenthal inequality (see  \Cref{propo:rosenthal_adapt_fort_moulines}) to obtain the following estimate:
{\small \begin{multline*}
\PE_z^{1/p}[\norm{J_{n+1}^{(0)}}^p]  \leq d \Const{\bar\varepsilon} \Cros{p}^{1/p} V^{1/\LL}(z) \Big\{  \big[ \myqcond + 1 \big] \alpha_{n+1}\\
+  \Big( \myqcond^2 \sum_{k=1}^{n+1} \alpha_k^2 \prod_{\ell=k+1}^{n+1} (1-\alpha_\ell a) \Big)^{1/2}
+ \myqcond (\smallConst{\alpha} + 2 \normop{A} ) \sum_{k=1}^{n+1} \alpha_{k+1}^2 \prod_{\ell=k+1}^{n+1}\sqrt{1 - a \alpha_\ell} \Big\} \eqsp.
\end{multline*}}Using the inequality $\sum_{k=1}^{n+1} \alpha_{k+1}^2 \prod_{\ell=k+1}^{n+1}\sqrt{1 - a \alpha_\ell} \leq (4/a) \alpha_{n+1}$ [cf.~\Cref{lem:bsum2}] yields that
\begin{equation} \label{eq:J0_main}
    \PE_z^{1/p}[\norm{J_{n+1}^{(0)}}^p]  \le \Const{\mathsf{J},p}^{(0)} \sqrt{\alpha_{n+1}} V^{1/\LL}(z) \eqsp, 
    \Const{\mathsf{J},p}^{(0)}= d \myqcond \Const{\bar \varepsilon} (2 + 4(\smallConst{\alpha} + 2 \normop{A})/a + 2/\sqrt{a})) \Cros{p}^{1/p};
\end{equation}
Finally, to analyze $H_{n+1}^{(0)}$, from \eqref{eq:jh_recur_main} we apply the H\"older's inequality twice to get
\begin{equation} \label{eq:holder_Hn0} \txts
    \PE_z^{1/p} [ \normop{ H_{n+1}^{(0)} }^p ] \leq \sum_{j=1}^{n+1} \alpha_j \PE_z^{1/(2p)} [ \normop{ \ProdB_{j+1:n+1}}^{2p} ] \, \PE_z^{1/(4p)} [ \normop{ \funcAt{Z_j} }^{4p} ] \, \PE_z^{1/(4p)}[ \normop{ J_{j-1}^{(0)} }^{4p} ] .
\end{equation}
Notice that $\PE_z^{1/(4p)} [ \normop{ \funcAt{Z_j} }^{4p} ] \leq \bConst{A} V^{1/\LL}(z)$ where $\bConst{A}$ is defined in \eqref{eq:AepsBd2} [cf.~\Cref{lem: moments of vareps}]. Using \Cref{th:expconvproducts} and \eqref{eq:J0_main}, we obtain
\begin{equation} \label{eq:H0_main}
\begin{split}
    & \txts \PE_z^{1/p} [ \normop{ H_{n+1}^{(0)} }^p ] \leq \Const{\mathsf{st},2p} \Const{\mathsf{J},4p}^{(0)} \bConst{A}  \sum_{j=1}^{n+1} \alpha_j \sqrt{\alpha_{j-1}} \rme^{- (a/4) \sum_{\ell=j+1}^{n+1} \alpha_\ell} V^{2/\LL + 1/(4p)}(z) \\
    & \txts  \overset{(a)}{\leq} \sqrt{1 + \alpha_{\infty,p}^{(1)} \smallConst{\alpha}} \Const{\mathsf{st},2p} \Const{\mathsf{J},4p}^{(0)} \bConst{A}  \sum_{j=1}^{n+1} \alpha_j^{3/2} \prod_{\ell=j+1}^{n+1} \big( 1 - \alpha_\ell a/8 \big) V^{2/\LL + 1/(4p)}(z) \\
    & \overset{(b)}{\leq} \Const{\mathsf{H},p}^{(0)} \sqrt{ \alpha_{n+1} }  V^{2/\LL + 1/(4p)}(z),
\end{split}
\end{equation}
with $\Const{\mathsf{H},p}^{(0)} =   16 \sqrt{1 + \alpha_{\infty,p}^{(1)} \smallConst{\alpha}} \Const{\mathsf{st},2p} \Const{\mathsf{J},4p}^{(0)} \bConst{A} /a$. In the above, (a) is due to \Cref{assum:stepsize} and the inequality $\rme^{-\alpha_j a / 4} \leq 1 - \alpha_j a/8$ since $\alpha_j a / 4 \leq 1$, (b) is due to the inequality $\sum_{j=1}^{n+1} \alpha_j^{3/2} \prod_{\ell=j+1}^{n+1} \big( 1 - {\alpha_\ell a}/{8}) \leq (16/a) \sqrt{\alpha_{n+1}}$ [cf.~\Cref{lem:bsum2}].
By observing that
\begin{equation}
    \ttheta_{n+1} = \utheta_{n+1} + \vtheta_{n+1} = \utheta_{n+1} + J_{n+1}^{(0)} + H_{n+1}^{(0)} \eqsp,
\end{equation}
applying Minkowski's inequality yields the bound in \eqref{eq:approx_bound_main}.\vspace{.1cm}

\paragraph{Refining the error bound $\PE_z^{1/p}[\normop{\vtheta_n}^p]$} It is possible to obtain a bound on $\PE_z^{1/p}[ \normop{H_n^{(0)}}^p ]$ tighter than ${\cal O}( \sqrt{\alpha_n} )$  obtained in \eqref{eq:H0_main}. This establishes in particular that $J_n^{(0)}$ is the leading term in the decomposition of the fluctuation term $\vtheta_{n+1} = J_{n+1}^{(0)} + H_{n+1}^{(0)}$.
To this end, we rely on an extra decomposition step similar to \eqref{eq:jn0_main}. 
We may further decompose the error term $H_n^{(0)}$ as $H_n^{(0)} = J_n^{(1)} + H_n^{(1)}$ such that
\begin{equation}
\begin{array}{ll}
\label{eq:expansion_recur_gen}
J_{n+1}^{(1)} = (\Id_d - \alpha_{n+1} A) J_n^{(1)} - \alpha_{n+1} \funcAt{Z_{n+1}} J_n^{(0)}, & J_{0}^{(1)}=0, \\[.1cm]
H_{n+1}^{(1)} = (\Id_d - \alpha_{n+1} \funcA{Z_{n+1}} ) H_n^{(1)} - \alpha_{n+1} \funcAt{Z_{n+1}} J_n^{(1)}, & H_0^{(1)} = 0,
\end{array}
\end{equation}
where $J_n^{(0)}$ is defined in \eqref{eq:jn0_main}.
For diminishing step sizes, here we should strengthen the previous assumption \Cref{assum:stepsize} as:
\begin{assum}
\label{assum:stepsize_2}
We have $\mathcal A_0 < \infty$, where $\mathcal A_n= \sum_{\ell=n}^\infty \alpha_\ell^2$. There exists a constant $0<\smallConst{\alpha} \le a/32$ such that for $k \in \nset$,
$ \alpha_{k}/\alpha_{k+1} \le 1 + \alpha_{k+1} \smallConst{\alpha}$ and $\alpha_k/\mathcal A_{k+1} \le (2/3)\smallConst{\alpha}$.
\end{assum}
It is easy to check that \Cref{assum:stepsize_2} is satisfied by diminishing step sizes $\alpha_n = \Const{a} (n+n_0)^{-\expo}$, $\expo \in (\frac{1}{2},1]$.

Using the decomposition in \eqref{eq:expansion_recur_gen}, we obtain the the following result:
\begin{theorem}
\label{th:approximation_expansion}
(\Cref{th:approximation_expansion_supp})
Let $\LL \geq 32$ and assume \Cref{assum:drift}, \Cref{assum:almost_bounded}, \Cref{assum:Hurwitzmatrices}, and \Cref{assum:funcb}($\LL$). For any $2 \le p \le \LL/16$ and any non-increasing sequence $(\alpha_k)_{k \in \nset}$ satisfying $\alpha_0 \in (0, \alpha_{\infty, p}^{(1)})$ such that $\alpha_k \equiv \alpha$ or \Cref{assum:stepsize_2} holds.
For any $z \in \Zset$, $n \in \nset$, it holds
\begin{equation} \label{eq:hn0_tight_main}
\textstyle{ \PE_z^{1/p}[ \| H_{n}^{(0)} \|^p ] \leq V^{3/\LL + 9/(16p)}(z)
    \begin{cases}
     \Const{p}^{(\mathsf{f})} \alpha \sqrt{\log(1/\alpha)}, & \text{if}~\alpha_n \equiv \alpha, \\
      \Const{p}^{(\mathsf{d})} \sqrt{\alpha_{n} {\cal A}_n \log(1/\alpha_n)}, & \text{if under \Cref{assum:stepsize_2},}
    \end{cases} }
\end{equation}
where $\alpha_{\infty,p}^{(1)}$, $\Const{p}^{(\mathsf{f})}, \Const{p}^{(\mathsf{d)}}$ are given in \eqref{eq:alpha_infty_1_def}, \eqref{eq:ConstH0fixed_improved}, respectively.
\end{theorem}
The theorem shows that the previous bound of $\PE_z^{1/p}[ \normop{H_n^{(0)}}^p ] = {\cal O}( \sqrt{\alpha_n} )$ can be improved to ${\cal O}( \sqrt{ \alpha_{n} {\cal A}_n \log(1/\alpha_n)} )$. Take for example a diminishing step size as $\alpha_n = \Const{a} (n+n_0)^{-1}$,
our result shows that the fluctuation term admits a \emph{clear separation of scales} as
\begin{equation*}
    \vtheta_n = J_n^{(0)} + H_n^{(0)} ~~\text{with}~~ \PE_z^{1/p}[ \normop{J_n^{(0)}}^p ] = {\cal O}( n^{-1 / 2} ),~~\PE_z^{1/p}[ \normop{H_n^{(0)}}^p ] = {\cal O}( n^{-1 } \sqrt{ \log n } ) .
\end{equation*}
\paragraph{Proof Sketch}
We study $J_{n+1}^{(1)}$ first. By \eqref{eq:expansion_recur_gen} and the definition of $J_n^{(0)}$ in \eqref{eq:jn0_main}, we obtain
\begin{equation} \label{eq:Sdef_expansion_main} \textstyle
    J_{n+1}^{(1)} = \sum_{j=1}^{n} \alpha_j S_{j+1:n+1} \funnoise{Z_j}, ~~\text{with} ~~ S_{j+1:n+1} = \sum_{k=j+1}^{n+1} \alpha_k G_{k+1:n} \funcAt{Z_k} G_{j+1:k-1} \eqsp.
\end{equation}
For illustrative purpose, in this proof sketch we will only consider the case when $\{ Z_i \}_{i \geq 1}$ are i.i.d.. Here, we have $\PE[S_{j+1:n+1} \funnoise{Z_j}|Z_{j+1}, \ldots Z_{n+1}] = 0$ and therefore $J_{n+1}^{(1)}$ is a Martingale. It follows:
{\small \begin{equation} \label{eq:jn1_rough}
\PE^{1/p}[\norm{J_{n+1}^{(1)}}^p] \overset{(a)}{\lesssim}   \sqrt{\sum_{j=1}^{n} \alpha_{j}^2 \PE^{2/p}[\norm{S_{j+1:n+1}}^p]} \overset{(b)}{\lesssim}   \sqrt{\sum_{j=1}^{n+1} \alpha_{j}^2 \mathcal A_{j} \prod_{\ell=j+1}^{n+1}(1 - a \alpha_\ell)} \lesssim \sqrt{\alpha_{n+1} \mathcal A_{n+1}},
\end{equation}}where (a) applied the Burkholder inequality~\citep[Theorem 2.10]{hallheydebook} for Martingales, and (b) can be obtained by applying the Rosenthal inequality for i.i.d. random variables to the expectation $\PE^{1/p}[\norm{S_{j+1:n+1}}^p]$ \citep[Theorem 2.12]{hallheydebook}.

Furthermore, we observe that $H_{n+1}^{(1)} = \sum_{j=1}^{n+1} \alpha_j \ProdB_{j+1:n+1} \funcAt{Z_j} J_{j-1}^{(1)}$. Similar to \eqref{eq:holder_Hn0}, we can apply \eqref{eq:jn1_rough} and the H\"{o}lder inequality to obtain $\PE^{1/p}[ \normop{ H_{n+1}^{(1)} }^p ] = {\cal O}(\sqrt{\alpha_{n+1} \mathcal A_{n+1}})$. Combining both bounds yields the conclusion of the theorem.

Unfortunately, in the Markovian case we cannot apply the same arguments directly since $J_{n+1}^{(1)}$ is no longer a martingale. Instead, we first decouple the dependent random variables $\funnoise{Z_j}$ and $S_{j+1:n+1}$. This is done in Lemma~\ref{lem: J1est} in the appendix by using the Berbee's coupling construction exploiting the fact that $V$-uniformly ergodic Markov chains are special cases of $\beta$-mixing processes \citep{riobook}.
We leave the detailed derivations in the appendix for interested readers.

\subsection{Temporal Difference Learning Algorithms}
Following the notation from \cite[Chapter~12]{sutton:book:2018}, we
consider a discounted Markov Reward Process (MRP) denoted by the tuple $( \Xset, \MKQ, \gain, \gamma )$, where $\MKQ$ is the state transition kernel defined on a general state space $(\Xset,\Xsigma)$. We do not assume that $\Xset$ is finite and countable, the only requirement being that $\Xsigma$ is countably generated: we may assume for example that $\Xset=\rset^d$.  For any given state $x \in \Xset$, the scalar $\gain(x)$ represents the reward of being at the state $x$. The reward function is possibly unbounded. Finally, $\gamma \in (0,1)$ is the discount factor.
The value function $V^\star : \Xset \to \rset$ is defined as the expected discounted reward $V^\star(x) = \PE_x [ \sum_{k=0}^\infty \gamma^k \gain( X_k ) ]$.

Let $d \in \nsets$, we associate with every state $x \in \Xset$ a \emph{feature vector} $\psi(x) \in \rset^d$ and approximate $V^\star(x)$ by a linear combination $V_\theta(x)= \psi(x)^{\top} \theta$ (see \cite{tsitsiklis:td:1997,sutton:book:2018}).
Temporal difference learning algorithms may be expressed as
\begin{equation}
\label{eq:TD-lambda}
\theta_{k+1} = \theta_k + \alpha_{k+1} \eligibility_k \{ \gain(X_k) + \gamma \psi(X_{k+1})^\top \theta_k - \psi(X_k)^\top \theta_k \} , \\
\end{equation}
where $\{\eligibility_k \}_{k \in \nset}$ is a sequence of eligibility vectors. For the TD(0) algorithm, $\eligibility_k = \psi(X_k)$. For the TD($\lambda$) algorithm,  $\varphi_k= (\lambda \gamma) \varphi_{k-1} + \psi(X_k)$.
Note that for TD($\lambda$), \eqref{eq:TD-lambda} corresponds to \eqref{eq:lsa} with the extended Markov chain $Z_k = (X_k,X_{k+1},\varphi_k)$ and $\bA(Z_k)=-\eligibility_k(\psi(X_k)^\top - \gamma \psi(X_{k+1})^\top)$, $b(Z_k) =\eligibility_k \gain(X_k)$.
\cite{srikant:1tsbounds:2019} were able to study TD$(\lambda)$ while that $(Z_k)_{k \in \nsets}$ is not necessary uniformly ergodic. Indeed, a core argument in their application is the use of \cite[Lemma 6.7]{bertsekas:tsitsiklis:96} which implies that if $\msz$ is a finite state space and $(X_k)_{k \in\nset}$ is uniformly ergodic, then $\norm{\PE_z[\bA(Z_k)] -A } \leq C \rho^k$ and $\norm{\PE_z[b(Z_k)] -b} \leq C \rho^k$, for any $z \in \Zset$, $k \in \nsets$ and for some $C \geq 0$, $\rho \in \ooint{0,1}$. This is precisely the condition considered by \cite{srikant:1tsbounds:2019} to derive their bounds.  Obviously \cite[Lemma 6.7]{bertsekas:tsitsiklis:96} does not extend to general (unbounded) state space. 

As a replacement, to verify our assumption \Cref{assum:drift}, we consider here a $\tau$-truncated version of the eligibility trace
\begin{equation}
\label{eq:tau-truncation} \txts
\eligibility_k=  \phi_\tau(X_{k-\tau+1:k})   \quad \text{where} \quad  \phi_\tau(x_{0:\tau-1} )  = \sum_{s=0}^{ \tau-1} (\lambda \gamma)^s \psi (x_{\tau-1-s}) \eqsp.
\end{equation}
TD(0) algorithm is a special case of \eqref{eq:tau-truncation} with $\tau=1$ and we recover the TD($\lambda$) algorithm by letting $\tau \rightarrow \infty$. 
The recursion \eqref{eq:TD-lambda} with eligibility vector  defined in \eqref{eq:tau-truncation} is  a special case of \eqref{eq:lsa}. To see this, we define $\State_k= [X_{k-\tau},\dots,X_{k}]^{\top}$ and observe that \eqref{eq:TD-lambda} can be obtained by using in \eqref{eq:lsa} the following matrix/vector, for $z = [x_0,\dots,x_\tau]^{\top}= x_{0:\tau} \in \Xset^{\tau+1}$,
\begin{equation}
\label{eq: TD0 matrix}
    \funcA{z} = \phi_\tau(x_{0:\tau-1}) \{ \psi(x_{\tau-1}) - \gamma \psi(x_\tau) \}^\top, \quad \funcb{z} = \phi_\tau(x_{0:\tau-1}) \gain(x_{\tau-1}) \eqsp.
\end{equation}
Consider the following assumptions.
\begin{assumSUP}
\label{assum:markovQ}
The  Markov kernel $\MKQ: \Xset \times \mathcal X \to \rset_+$ is irreducible and strongly aperiodic. There exist $c > 0, \bb > 0, \delta \in (1/2,1]$, $R_0 \geq 0$,  and $\tilde V: \Xset \to \coint{\rme,\infty}$ such that by setting $\tilde W = \log \tilde V$, $\msc_0 = \{x : \tilde W(x) \leq R_0\}$, $\msc_0^{\complement} = \{x : \tilde W(x) > R_0\}$, we have
\begin{equation}
\label{eq:drift-condition-improv_Q}
\MKQ \tilde V(x) \le \exp\parentheseDeuxLigne{-c \tilde W^{\delta}(x)} \tilde V(x)  \indi{\msc_0^{\complement}}(x) + \bb \indi{\msc_0}(x) \eqsp,
\end{equation}
in addition, for any $R \geq 1$, the level sets $\{ x : \tilde W(x) \leq R\}$ are $(1,\varepsilon_R\nu)$-small for $\MKQ$, with $\varepsilon_R \in \ocint{0,1}$ and $\nu$ being a probability measure on $(\Xset,\mathcal{X})$.\vspace{-.2cm}
\end{assumSUP}
It follows from~\citep[Theorem 15.2.4]{douc:moulines:priouret:2018} that
the Markov kernel $\MKQ$ admits a unique stationary distribution $\pi_0$.
\begin{assumSUP}
\label{assum:markov2}
$\pi_0(\psi \psi^\top)$ is positive definite.
\end{assumSUP}
In the following, we show that under \Cref{assum:markovQ}, \Cref{assum:markov2}, the TD($\lambda$) algorithm with truncated eligibility trace \eqref{eq:TD-lambda} satisfies the assumptions in \Cref{sec:lsa}. In this case, the state-space is set to be $\Zset=\Xset^{\tau+1}$
and the Markov kernel $\MK$ is given, for any $z=x_{0:\tau} \in  \Xset^{\tau+1}$, by
\begin{equation}
\label{eq: trans kernel for TD} \txts
 \MK(x_{0:\tau}; \rmd x'_{0:\tau})= \prod_{\ell=1}^\tau \delta_{x_\ell}(\rmd x'_{\ell-1})\MKQ(x_{\tau}, \rmd x'_{\tau}) \eqsp,
\end{equation}
where $\delta_{x}$ denotes the Dirac measure at $x\in \Xset$.
\begin{enumerate}[noitemsep, leftmargin=5mm]
 \item It follows from \Cref{lem: tdlambda irredu} that $\MK$ is irreducible, aperiodic and has a unique invariant distribution $\pi(\rmd x_{0:\tau})= \pi_0(\rmd x_0) \prod_{\ell=1}^{\tau} \MKQ(x_{\ell-1}, \rmd x_\ell)$. By \Cref{lem: tdlambda drift}, the super-Lyapunov drift condition \eqref{eq:drift-condition-improv} is satisfied with
\begin{align*} \txts
    V(x_{0:\tau}) = \exp \left( c_0 \sum_{i=0}^{\tau-1}(i+1)\tilde{W}^{\delta}(x_i) + \tilde{W}(x_\tau)\right),
\end{align*}
where $c_0$ is defined in \eqref{eq:const_c_0_def_td_lambda}. Hence, \Cref{assum:drift} is verified.
    \item
    Let $\norm{\psi(x)} \le \Const{\psi} W^{\beta/2}(x)$ and for $\LL \geq 1$, $|\gain(x)|\le \Const{\gain, \LL} V^{1/2\LL}(x)$, where $\Const{\psi}, \Const{\gain, \LL}>0$ are some constants.  Then \Cref{assum:almost_bounded} and \Cref{assum:funcb}($\LL$) are satisfied with
\begin{align}
  \bConst{A} = (1 + \gamma) \Const{\psi}^2 / (1-\lambda \gamma), \quad  \bConst{b, \LL} = \Const{\gain, \LL} \Const{\psi} (\beta \LL / \rme)^{\beta/2}/(1-\lambda \gamma).
\end{align}
\item Eq.~\eqref{eq: TD0 matrix} implies
\[ \txts
A = \sum_{\ell=0}^{\tau-1}\PE_{\pi_0}[\psi(X_{\tau-1-\ell}) \{ \psi(X_{\tau-1}) - \gamma \psi(X_{\tau}) \}^{\top}] \eqsp.
\]
Assumption \Cref{assum:Hurwitzmatrices} follows from \Cref{lem: hurwitz TD} in the appendix.
\end{enumerate}
Collecting the above results shows that the assumptions required by \Cref{th:approximation_error} are satisfied, thereby proving that the ${\rm L}_p$ error of TD($\lambda$) algorithm \eqref{eq:TD-lambda} (with truncated eligibility trace) converges according to the rate specified in \eqref{eq:approx_bound_main}.

\paragraph{Conclusions} We have established the $({\rm V},q)$-exponential stability of the sequence of random matrices $\{ \funcA{Z_k} \}_{k \in \nsets}$ under relaxed conditions on the Markov chain and the matrix functions. The results are applied to obtain finite-time $p$-th moment bounds of LSA error, and a family of TD learning algorithms.

\newpage
\bibliography{references,moulines}

\begin{thebibliography}{30}
\providecommand{\natexlab}[1]{#1}
\providecommand{\url}[1]{\texttt{#1}}
\expandafter\ifx\csname urlstyle\endcsname\relax
  \providecommand{\doi}[1]{doi: #1}\else
  \providecommand{\doi}{doi: \begingroup \urlstyle{rm}\Url}\fi

\bibitem[Benveniste et~al.(2012)Benveniste, M{\'e}tivier, and
  Priouret]{benveniste2012adaptive}
A.~Benveniste, M.~M{\'e}tivier, and P.~Priouret.
\newblock \emph{Adaptive algorithms and stochastic approximations}, volume~22.
\newblock Springer Science \& Business Media, 2012.

\bibitem[Bertsekas and Tsitsiklis(1996)]{bertsekas:tsitsiklis:96}
D.~P. Bertsekas and J.~N. Tsitsiklis.
\newblock \emph{Neuro-dynamic programming.}
\newblock Athena Scientific, Belmont, MA, 1996.

\bibitem[Bhandari et~al.(2018)Bhandari, Russo, and Singal]{bhandari2018finite}
J.~Bhandari, D.~Russo, and R.~Singal.
\newblock A finite time analysis of temporal difference learning with linear
  function approximation.
\newblock In \emph{Conference On Learning Theory}, pages 1691--1692, 2018.

\bibitem[Chen et~al.(2020)Chen, Devraj, Busic, and Meyn]{chen2020explicit}
S.~Chen, A.~Devraj, A.~Busic, and S.~Meyn.
\newblock Explicit mean-square error bounds for monte-carlo and linear
  stochastic approximation.
\newblock In \emph{International Conference on Artificial Intelligence and
  Statistics}, pages 4173--4183. PMLR, 2020.

\bibitem[Dalal et~al.(2018)Dalal, Sz{\"o}r{\'e}nyi, Thoppe, and
  Mannor]{dalal:td0:2017}
G.~Dalal, Bal{\'a}zs Sz{\"o}r{\'e}nyi, G.~Thoppe, and S.~Mannor.
\newblock {Finite sample analyses for TD(0) with function approximation}.
\newblock In \emph{Thirty-Second AAAI Conference on Artificial Intelligence},
  2018.

\bibitem[Delyon and Yuditsky(1999)]{Delyon_Yuditsky1999}
B.~Delyon and A.~Yuditsky.
\newblock On small perturbations of stable markov operators: Unbounded case.
\newblock \emph{Theory Probab. Appl.}, 43\penalty0 (4):\penalty0 577--587,
  1999.

\bibitem[Doan(2019)]{doan2019finite}
T.~T Doan.
\newblock Finite-time analysis and restarting scheme for linear two-time-scale
  stochastic approximation.
\newblock \emph{arXiv preprint arXiv:1912.10583}, 2019.

\bibitem[Douc et~al.(2018)Douc, Moulines, Priouret, and
  Soulier]{douc:moulines:priouret:2018}
R.~Douc, E.~Moulines, P.~Priouret, and P.~Soulier.
\newblock \emph{Markov chains}.
\newblock Springer Series in Operations Research and Financial Engineering.
  Springer, Cham, 2018.
\newblock ISBN 978-3-319-97703-4; 978-3-319-97704-1.
\newblock \doi{10.1007/978-3-319-97704-1}.
\newblock URL \url{https://doi.org/10.1007/978-3-319-97704-1}.

\bibitem[Eweda and Macchi(1983)]{eweda:macchi:1983}
E.~Eweda and O.~Macchi.
\newblock Quadratic mean and almost-sure convergence of unbounded stochastic
  approximation algorithms with correlated observations.
\newblock \emph{Ann. Inst. H. Poincar\'{e} Sect. B (N.S.)}, 19\penalty0
  (3):\penalty0 235--255, 1983.
\newblock ISSN 0020-2347.

\bibitem[Fort and Moulines(2003)]{fort:moulines:2003}
G.~Fort and E.~Moulines.
\newblock Convergence of the monte carlo expectation maximization for curved
  exponential families.
\newblock \emph{Annals of Statistics}, 31\penalty0 (4):\penalty0 1220--1259,
  2003.

\bibitem[Guo(1994)]{guo1994stability}
L.~Guo.
\newblock Stability of recursive stochastic tracking algorithms.
\newblock \emph{SIAM Journal on Control and Optimization}, 32\penalty0
  (5):\penalty0 1195--1225, 1994.

\bibitem[Guo and Ljung(1995{\natexlab{a}})]{guo1995exponential}
L.~Guo and L.~Ljung.
\newblock Exponential stability of general tracking algorithms.
\newblock \emph{IEEE Transactions on Automatic Control}, 40\penalty0
  (8):\penalty0 1376--1387, 1995{\natexlab{a}}.

\bibitem[Guo and Ljung(1995{\natexlab{b}})]{guo1995performance}
L.~Guo and L.~Ljung.
\newblock Performance analysis of general tracking algorithms.
\newblock \emph{IEEE Transactions on Automatic Control}, 40\penalty0
  (8):\penalty0 1388--1402, 1995{\natexlab{b}}.

\bibitem[Gupta et~al.(2019)Gupta, Srikant, and Ying]{gupta2019finite}
H.~Gupta, R~Srikant, and L.~Ying.
\newblock Finite-time performance bounds and adaptive learning rate selection
  for two time-scale reinforcement learning.
\newblock In \emph{Advances in Neural Information Processing Systems}, pages
  4706--4715, 2019.

\bibitem[Hall and Heyde(1980)]{hallheydebook}
P.~Hall and C.~Heyde.
\newblock \emph{MartinG.e Limit Theory and Its Application}.
\newblock Academic Press, 1980.

\bibitem[Kaledin et~al.(2020)Kaledin, Moulines, Naumov, Tadic, and
  Wai]{kaledin2020finite}
M.~Kaledin, E.~Moulines, A.~Naumov, V.~Tadic, and Hoi-To Wai.
\newblock Finite time analysis of linear two-timescale stochastic approximation
  with markovian noise.
\newblock In \emph{Conference On Learning Theory}, 2020.

\bibitem[Kontoyiannis and Meyn(2005)]{kontoyiannis:meyn:2005}
I.~Kontoyiannis and S.~Meyn.
\newblock Large deviations asymptotics and the spectral theory of
  multiplicatively regular markov processes.
\newblock \emph{Electronic Journal of Probability}, 10:\penalty0 61--123, 2005.

\bibitem[Kontoyiannis and Meyn(2003)]{kontoyiannis:meyn:2003}
I.~Kontoyiannis and S.~P. Meyn.
\newblock Spectral theory and limit theorems for geometrically ergodic {M}arkov
  processes.
\newblock \emph{Ann. Appl. Probab.}, 13\penalty0 (1):\penalty0 304--362, 2003.
\newblock ISSN 1050-5164.
\newblock \doi{10.1214/aoap/1042765670}.
\newblock URL \url{https://doi.org/10.1214/aoap/1042765670}.

\bibitem[Lakshminarayanan and Szepesvari(2018)]{lakshminarayanan2018linear}
C.~Lakshminarayanan and C.~Szepesvari.
\newblock Linear stochastic approximation: How far does constant step-size and
  iterate averaging go?
\newblock In \emph{International Conference on Artificial Intelligence and
  Statistics}, pages 1347--1355, 2018.

\bibitem[Ljung(2002)]{ljung2002recursive}
Lennart Ljung.
\newblock Recursive identification algorithms.
\newblock \emph{Circuits, Systems and Signal Processing}, 21\penalty0
  (1):\penalty0 57--68, 2002.

\bibitem[Osekowski(2012)]{osekowski:2012}
A.~Osekowski.
\newblock \emph{Sharp Martingale and Semimartingale Inequalities}.
\newblock Monografie Matematyczne 72. Birkhäuser Basel, 1 edition, 2012.
\newblock ISBN 3034803699,9783034803694.

\bibitem[Poznyak(2008)]{poznyak:control}
A.~S. Poznyak.
\newblock \emph{Advanced Mathematical Tools for Automatic Control Engineers:
  Deterministic Techniques}.
\newblock Elsevier, Oxford, 2008.

\bibitem[Priouret and Veretenikov(1998)]{priouret1998remark}
P.~Priouret and A.~Veretenikov.
\newblock A remark on the stability of the {LMS} tracking algorithm.
\newblock \emph{Stochastic analysis and applications}, 16\penalty0
  (1):\penalty0 119--129, 1998.

\bibitem[Rio(2017)]{riobook}
E.~Rio.
\newblock \emph{Asymptotic Theory of Weakly Dependent Random Processes}.
\newblock Springer, 2017.

\bibitem[{Srikant} and {Ying}(2019)]{srikant:1tsbounds:2019}
R.~{Srikant} and L.~{Ying}.
\newblock {Finite-Time Error Bounds For Linear Stochastic Approximation and TD
  Learning}.
\newblock In \emph{Conference on Learning Theory}, 2019.

\bibitem[Sutton(1988)]{sutton:td:1988}
R.~S. Sutton.
\newblock Learning to predict by the methods of temporal differences.
\newblock \emph{Machine Learning}, 3\penalty0 (1):\penalty0 9--44, Aug 1988.
\newblock ISSN 1573-0565.
\newblock \doi{10.1007/BF00115009}.

\bibitem[Sutton and Barto(2018)]{sutton:book:2018}
R.~S. Sutton and Andrew~G. Barto.
\newblock \emph{Reinforcement Learning: An Introduction}.
\newblock The MIT Press, second edition, 2018.

\bibitem[{Tsitsiklis} and {Van Roy}(1997)]{tsitsiklis:td:1997}
J.~N. {Tsitsiklis} and B.~{Van Roy}.
\newblock An analysis of temporal-difference learning with function
  approximation.
\newblock \emph{IEEE Transactions on Automatic Control}, 42\penalty0
  (5):\penalty0 674--690, May 1997.
\newblock ISSN 2334-3303.
\newblock \doi{10.1109/9.580874}.

\bibitem[Varadhan(1984)]{varadhan:1984}
S.~Varadhan.
\newblock \emph{Large deviations and applications}.
\newblock SIAM, 1984.

\bibitem[Xu et~al.(2019)Xu, Zou, and Liang]{xu2019two}
Tengyu Xu, Shaofeng Zou, and Yingbin Liang.
\newblock Two time-scale off-policy td learning: Non-asymptotic analysis over
  markovian samples.
\newblock In \emph{Advances in Neural Information Processing Systems}, pages
  10633--10643, 2019.

\end{thebibliography}

\newpage
\appendix

\section{Formal statement and Proof for \Cref{ex:counterexample}}
\label{sec:proof:theo:counterexample}
\begin{proposition}
  \label[proposition]{theo:counterexample}
Consider the Markov chain $(\State_k)_{k \in\nset}$ defined by $\State_{k+1}=   \State_k -1$,  if $\State_k >1$ and $  \State_{k+1}= \rmY_{k+1}$, if $\State_k=1$ with $\State_0 = 1$, where  $(\rmY_k)_{k \in\nset}$ is an \iid\ sequence and  $\rmY_k \in \Stateset$. Consider the sequence $(\theta_n^{\varepsilon})_{n\in\nset}$ starting from $\theta_0 >0$ and defined by the recursion $\theta_{n+1}^{\varepsilon} = \{1-\alpha A_{\varepsilon}(\State_{n+1})\} \theta_n^{\varepsilon}$ with $\alpha, \varepsilon >0$ and $A_{\varepsilon}$ given by 
\begin{equation}
  \label{eq:def:A_vareps}
    A_{\varepsilon}(z)= \begin{cases}
    1, & \text{ if $z = 1$ \eqsp,} \\
    -\varepsilon, & \text{ otherwise}  \eqsp.
    \end{cases}
\end{equation} 
Then there exists $\bvarepsilon \in \ooint{0,1}$ such that for any $\alpha \in \ooint{0,1}$, $\limsup_{n \to \plusinfty} \expeLigne{\absLigne{\theta_n^{\bvarepsilon}}} = \plusinfty$.
\end{proposition}

\begin{proof}
  Assume that $(\State_k)_{k \in\nset}$ is not geometrically ergodic and let $\alpha >0$.
  First for any $\varepsilon >0$,
  $\int_{\Stateset} A_{\bvarepsilon}(x) \rmd \pi(x) = -\varepsilon
  [1-\pi(\{1\})] + \pi(\{1\})$. Then for any
  $\varepsilon \in\ooint{0,\bvarepsilon}$, setting $\bvarepsilon =\pi(\{1\})$, we get that   $\int_{\Stateset} A_{\bvarepsilon}(z) \rmd \pi(z) >0$.
  In addition, we have by definition of $(\theta_n^{\bvarepsilon})_{n \in\nset}$,
  \begin{equation}
    \label{eq:eq:proof:counterexample}
    \expe{\absLigne{\theta_n^{\bvarepsilon}}} \geq \theta_0 (1+\alpha \bvarepsilon)^n \PP(\rmY_1 > n+1)\eqsp.
  \end{equation}

  By \cite[Theorem~15.1.5]{douc:moulines:priouret:2018}, 
  $(\State_k)_{k \in\nset}$ is not geometrically ergodic and for any
  $\upeta >0$, $\expe{(1+\upeta)^{\rmY_1}} = \plusinfty$. Therefore,
  $\limsup_{n\to \plusinfty} [(1+\alpha\bvarepsilon)^n \PP(\rmY_1 \geq n)] =
  \plusinfty$, otherwise we would obtain that for any $\varepsilon \in \ooint{0, \alpha \bvarepsilon}$, $\expe{(1+\varepsilon)^{\rmY_1}} \leq  \sup_{n\in \nset} [(1+\alpha\bvarepsilon)^n \PP(\rmY_1 \geq n)] \sum_{k=1}^{\plusinfty} [(1+\varepsilon)/(1+\alpha\bvarepsilon)]^k < \plusinfty$, which is absurd. Applying this result to  \eqref{eq:eq:proof:counterexample} completes the proof.
\end{proof} 

\section{Super-Lyapunov drift conditions \Cref{assum:drift}}
We gather the technical results needed for the proof of our main theorems.
Define 
\begin{align}
  \label{eq:def_msc_R}
  \msc_{R} &= \{ z \in \msz \ , : \, W(z) \geq R \} \eqsp, \text{ for any $R \geq 0$} \eqsp, \\
  \label{eq:def_varphi_delta}
  \varphi_{\delta}&: z \mapsto  cW^{\delta}(z) \eqsp. 
\end{align}

\begin{lemma}
  \label[lemma]{lem:drift_conditions_technical}
  Assume \Cref{assum:drift}. Then for any $n \in \nset$, we have
  \begin{equation}
      \label{eq:lem:drift_conditions_technical}
    \MK^n V(z)  \txts \leq \lambda^n V(z) + \bb/(1-\lambda)\eqsp, \quad 
    \MK^n V(z)   \leq   \rme^{-\varphi_{\delta}(z)} V(z) + [\bb/(1-\lambda)] \1_{\msc_{R_1}}(z) \eqsp,
  \end{equation}
  where $\lambda$ is defined in \eqref{eq:def_lambda}, $\msc_R$ in \eqref{eq:def_msc_R} and
  \begin{equation}
    \label{eq:def_R_1}
R_1 =     \inf\defEnsLigne{R\geq R_0 \, : \,  \exp\parentheseLigne{R-cR^{\delta}} > [\bb/(1-\lambda)^2]} \eqsp. 
  \end{equation}
\end{lemma}
\begin{proof}
  We first show the left-hand side inequality in \eqref{eq:lem:drift_conditions_technical}. First, \Cref{assum:drift} and \eqref{eq:def_lambda} shows that $\MK V \leq \lambda V + \bb$ which implies by a straightforward induction that for any $n \in \nset$,
  \begin{equation}
\label{eq:lem:drift_conditions_technical_proof}
    \MK^n V \leq \lambda^n V(z) + \bb \sum_{k=0}^{n-1} \lambda^k \eqsp. 
  \end{equation}
  Using $\sum_{k=0}^{n-1} \lambda^k \leq (1-\lambda)^{-1}$ completes the proof.

  We now show the right-hand side inequality of \eqref{eq:lem:drift_conditions_technical}. \eqref{eq:lem:drift_conditions_technical_proof}  applied for $n-1 \in \nset$ and \eqref{eq:drift-condition-improv} implies that $\MK^n V(z) \leq \lambda^{n-1} \MK V(z) + \bb \sum_{k=0}^{n-2} \lambda^k \leq \rme^{-\varphi_{\delta}(z)} \lambda^{n-1} V(z) + \bb \sum_{k=0}^{n-1} \lambda^k\leq \rme^{-\varphi_{\delta}(z)} V(z) - \rme^{-\varphi_{\delta}(z)} (1-\lambda) V(z) + \bb[1-\lambda]^{-1}$. Then, using by definition of $R_1$ that for any $z \in \msc_{R_1}^{\complement}$, $\rme^{-\varphi_{\delta}(z)} V(z)  \geq \bb[1-\lambda]^{-2}$ completes the proof.
\end{proof}

\begin{lemma}
\label{lem:poly_drift_condition}
Assume \Cref{assum:drift}. Then, for any $\upgamma > 0$, 
\begin{equation}
\label{eq:poly_drift_condition}
\MK W^{\upgamma + 1-\delta}(z) \leq  W^{\upgamma + 1-\delta}(z) - \cupgamma W^{\upgamma}(z)  + \bupgamma \1_{\msc_{\Rupgamma}}(z)\eqsp,
\end{equation}
where the constants $\cupgamma, \Rupgamma$ and $\bupgamma$ are given by: if $ \upgamma \leq \delta$,
\begin{equation}
\Rupgamma = R_0\eqsp,  \qquad \cupgamma = 1\wedge[(\upgamma + 1-\delta) c]  \eqsp, \qquad \bupgamma = \log^{\upgamma + 1-\delta}{\bb}\eqsp,
\end{equation}
and if $\upgamma  > \delta$,
\begin{equation}
\begin{aligned}
\label{eq:constants_poly_drift_large_beta}
&\Rupgamma = R_0 \vee (2(\upgamma + 1-\delta)/c)^{1/\delta} \vee c^{1/(\delta-1)} \eqsp, \quad \bupgamma = \log^{\upgamma+1-\delta}\parentheseDeux{(\bb +\rme^{\upgamma-\delta}) \vee (\exp(\Rupgamma + \rme^{\upgamma-\delta}) )} \\
&\cupgamma = 1\wedge[(\upgamma + 1-\delta) (1-c\Rupgamma^{\delta-1}/2)^{\upgamma-\delta}(c/2)] \eqsp.  
\end{aligned}
\end{equation}
\end{lemma}
\begin{proof}
  We consider separately the cases $\upgamma \leq \delta$ and $\upgamma > \delta$.
  
If $\upgamma \leq \delta$, the function $z \mapsto \log^{\upgamma + 1-\delta}{z}$ is concave. Using Jensen's inequality and \Cref{assum:drift}, we get that 
\begin{align*}
\MK W^{\upgamma + 1-\delta}(z) &\leq ( \MK W(z))^{\upgamma + 1-\delta} \leq (W(z) - c W^{\delta}(z))^{\upgamma + 1-\delta}\1_{\msc_0^{\complement}}(z) + \log^{\upgamma + 1-\delta}(\bb) \1_{\msc_0}(z)  \\
&= W^{\upgamma + 1-\delta}(z)(1 - cW^{\delta-1}(z))^{\upgamma + 1-\delta}\1_{\msc_0^{\complement}}(z) + \log^{\upgamma + 1-\delta}(\bb) \1_{\msc_0}(z)\eqsp.
\end{align*}
Note that \Cref{assum:drift} implies that for any
$z \in \msc_0^{\complement}$,
$1 \leq \MK V(z) \leq V(z) \rme^{-c W^{\delta}(z)}$ and therefore, 
$c W^{\delta-1}(z) \leq 1$ since $\delta \leq 1$. Then,
 Using that $(1-x)^{\upgamma + 1-\delta} < 1-(\upgamma + 1-\delta)
x$ for all $x \in [0,1]$ since $\upgamma +1 - \delta
\leq 1$ and $c W^{\delta-1}(z) \leq 1$ on $\msc_0^{\complement}$,  we get that
\begin{equation*}
  W^{\upgamma + 1-\delta}(z)(1 - cW^{\delta-1}(z))^{\upgamma + 1-\delta}\1_{\msc_0^{\complement}}(z) \leq W^{\upgamma + 1-\delta}(z) - (\upgamma+1-\delta)c W^{\upgamma}(z) \eqsp,
\end{equation*}
which completes the proof for $\upgamma \leq \delta$.

Consider now the case $\upgamma > \delta$ and note that the function $z \mapsto \log^{\upgamma + 1-\delta}{z}$ is concave on $\coint{\exp\parenthese{\upgamma-\delta}, \plusinfty}$ and therefore $\psi_\upgamma : z \mapsto \log^{\upgamma + 1-\delta}(z + \rme^{\upgamma-\delta})$ is concave on $\rset_+$. Using Jensen's inequality, we obtain 
\begin{equation}
  \label{eq:proof:lem:poly_drift_condition_psi}
  \MK W^{\upgamma + 1-\delta}(z) = \MK \log^{\upgamma + 1-\delta}(V(z)) \leq \MK \psi_{\gamma} \circ V(z)  \leq \psi_{\gamma}[\MK V(z)] \eqsp.
\end{equation}
Now by \Cref{assum:drift} and $a+b \leq a(b+1)$ for $a,b \geq 1$, $\rme^c +1 \leq \rme^{c+1}$, we get 
\begin{align*}
  \MK V(z) + \rme^{\upgamma-\delta}& \leq 
  (\exp\parentheseLigne{W(z) - cW^{\delta}(z)} + \rme^{\upgamma-\delta} )\1_{\msc_0^{\complement}}(z) +(\bb +\rme^{\upgamma-\delta}) \1_{\msc_0}(z)\\
                                   & \exp\parentheseLigne{W(z) - cW^{\delta}(z) +\upgamma+1-\delta }\1_{\msc_0^{\complement}}(z) +(\bb +\rme^{\upgamma-\delta}) \1_{\msc_0}(z)\\
  & \leq \exp\parentheseLigne{W(z) - (c/2)W^{\delta}(z) }\1_{\msc_{\Rupgamma}^{\complement}}(z) +(\bb +\rme^{\upgamma-\delta}) \vee (\exp(\Rupgamma + \rme^{\upgamma-\delta}) ) \1_{\msc_{\Rupgamma}}(z) \eqsp,
\end{align*}
where we used for the last inequality that for any $z \not \in \msc_{\Rupgamma}$ and the definitions \eqref{eq:def_msc_R}, \eqref{eq:constants_poly_drift_large_beta}, $W^{\delta}(z) \geq 2\rme^{\upgamma - \delta}/c$ and $\Rupgamma \geq R_0$. Using the previous result in \eqref{eq:proof:lem:poly_drift_condition_psi}, we get that

\begin{equation}
  \label{eq:proof:lem:poly_drift_condition_psi_1}
  \MK W^{\upgamma + 1-\delta}(z) \leq  \parenthese{ W(z) - (c/2)W^{\delta}(z)}^{\upgamma+1-\delta} + \bupgamma\1_{\msc_{\Rupgamma}}(z) \eqsp,
\end{equation}
where $\bupgamma$ is given in \eqref{eq:constants_poly_drift_large_beta}. 
Note that $(1-x)^{\upgamma + 1-\delta} \leq 1 - (\upgamma + 1-\delta)(1- c\Rupgamma^{\delta-1}/2)^{\upgamma -\delta} x$ for all $x \in [0, c\Rupgamma^{\delta-1}/2]$ since $c\Rupgamma^{\delta-1}/2 \leq 1/2$ by definition of $\Rupgamma$ \eqref{eq:constants_poly_drift_large_beta}. Therefore, using that on $\msc_{\Rupgamma}$, we have $0 < cW^{\delta-1}(z)/2 \leq c\Rupgamma^{\delta-1}/2 $, we get
$ \parenthese{ W(z) - (c/2)W^{\delta}(z)}^{\upgamma+1-\delta}= W^{\upgamma+1-\delta}(z)\{1-(c/2)W^{\delta-1}\}^{\upgamma+1-\delta} \leq W^{\upgamma+1-\delta}(z)-\cupgamma W^{\upgamma}(z)$. Plugging this result in \eqref{eq:proof:lem:poly_drift_condition_psi_1} concludes the proof of \eqref{eq:poly_drift_condition} for $\upgamma >\delta$.
\end{proof}

\begin{corollary}
  \label{coro:moment:W}
  Assume \Cref{assum:drift}. Then, for any $\upgamma >0$, it holds that $    \pi(W^{\upgamma}) \leq \bupgamma/\cupgamma$ and $\pi(V) \leq b/(1-\lambda)$.
  where $\bupgamma,\cupgamma$ are given in \Cref{lem:poly_drift_condition}. 
\end{corollary}

\begin{proof}
  As mentioned previously (see \eqref{eq:drift_conseq}), $\MK$ has a unique stationary distribution satisfying $\pi(V) < \plusinfty$. Therefore, since $\VnormD{W}{V} < \plusinfty$, we can take the integral in \eqref{eq:poly_drift_condition} and \eqref{eq:lem:drift_conditions_technical} with respect to $\pi$. Rearranging terms completes the proof. 
\end{proof}


\section{Rosenthal inequality for Markov chains}\label{sec:Rosenthal}
In this section, we state a general weighted Rosenthal inequality for $f$-ergodic Markov chain. This result is a simple adaptation of \cite[Proposition 12]{fort:moulines:2003}. In addition, we apply this result to obtain bounds which will be useful in the proof of our main results.

In all this section, $(Z_k)_{k \in \nset}$ is the canonical Markov chain corresponding to the Markov kernel $\MK$ on the filtered canonical space $(\msz^{\nset}, \mcz^{\otimes \nset}, (\mathcal{F}_n)_{n \in \nset})$, where $\mathcal{F}_n = \sigma(Z_0,\ldots,Z_n)$ for $n \in\nset$. We still denote by $\PP_{\mu}$ and $\PE_{\mu}$ the corresponding probability distribution and expectation with initial distribution $\mu$. In the case $\mu= \updelta_z$, $z \in \msz$, $\PP_{\mu}$ and $\PE_{\mu}$ are  denoted by $\PP_{z}$ and $\PE_{z}$.

\begin{proposition}[Rosenthal's inequality]
  \label[proposition]{propo:rosenthal_adapt_fort_moulines}
  Let $p \geq 2$ and $f,\lyapW,\lyapV : \msz \to \ccint{1,\plusinfty}$ such that $\VnormD{f}{\lyapW} \leq 1$ and $\VnormD{\lyapW^p}{\lyapV} \leq 1$. Let $(\beta_k)_{k \in\nset}$ be a real sequence.
  Assume that $\MK$ has a unique stationary distribution $\pi$ and satisfies for any $z \in \msz$,
  \begin{equation}
    \label{eq:hyp_rosenthal_fort_moulines}
    \sum_{n\in\nset} \VnormD{\updelta_z \MK^n - \pi}{f} \leq \Cf \lyapW(z) \eqsp, \quad
          \sum_{n\in\nset} \VnormD{\updelta_z \MK^n - \pi}{\lyapW^p} \leq \CW \lyapV(z) \eqsp,
        \end{equation}
        for some constants $\Cf,\CW < \plusinfty$.
        Then, for any $g \in \mrl_{\infty}^{f}$, it holds that for any $z \in \msz$,
    \begin{equation}
    \label{eq:rosenthal_fort_moulines}
    \begin{split}
    &\PE_{z}\bigg[\bigg|\sum\nolimits_{k=1}^n \beta_k \{g(Z_k) - \pi(g)\}\bigg|^p\bigg] \\
    &\leq \VnormD{g}{f}^p\Cros{p}\left[ \bigg\{\sum\nolimits_{k=1}^{n} \beta_{k}^2\bigg\}^{p/2} + \bigg\{\sum\nolimits_{k=1}^{n-1}|\beta_k-\beta_{k+1}|\bigg\}^{p} + \beta_1^p + \beta_n^p \right] \lyapV(z)\eqsp,
    \end{split}
\end{equation}
where
\begin{equation}
  \label{eq:def_C_ros}
  \Cros{p} = 6^p \Cf^p\{\CW+\pi(\lyapW^p)\}(p^p +2) \eqsp. 
\end{equation}
\end{proposition}

\begin{proof}
  Let $g \in \mrl_{\infty}^{f}$ and $z \in \msz$. Without loss of generality, we assume
  that $\VnormD{g}{f} \leq 1$. Denote by $S_n = \sum_{k=1}^n \beta_k \{g(Z_k) - \pi(g)\}$. By
  \eqref{eq:hyp_rosenthal_fort_moulines}, the function
  $\hg(x) = \sum_{n \in \nset} \{\MK^ng(x) - \pi(g)\}$ is well
  defined, $\hg \in \mrl_{\infty}^{\lyapW}$,
  \begin{equation}
    \label{eq:bound_norm_Poisson}
    \VnormD{\hg}{\lyapW} \leq \Cf
  \end{equation}
  and is a solution of the
  Poisson equation $\gPois - \MK \gPois = g - \pi(g)$.
Then, we have 
\begin{align*}
  S_n &= M_n +R_{1,n}+ R_{2,n} \eqsp,\\
  M_n &=  \sum_{ k=0}^{n-1}\beta_{k+1}\{\gPois(\State_{k+1}) - \MK \gPois(\State_{k})\} \\
  R_{1,n} &=\sum_{ k=1}^{n-1}(\beta_{k+1}-\beta_{k})\MK \gPois(\State_{k})\eqsp, \quad 
            R_{2,n} = \beta_{1}\MK \gPois(\State_{0}) - \beta_{n}\MK \gPois(\State_{n}) \eqsp. 
\end{align*}
Therefore, by Young inequality, we get that
\begin{equation}
 \label{eq:proof_ros_fort_moulines_S} 
  \expeMarkov{z}{\abs{S_n}^p} \leq 3^{p-1}\{  \expeMarkov{z}{\abs{M_n}^p}+   \expeMarkov{z}{\abs{R_{1,n}}^p}+   \expeMarkov{z}{\abs{R_{1,_n}}^p}\} \eqsp. 
\end{equation}
We now bound each term on the right-hand side. 

First, since $\hg \in \mrl_{\infty}^{\lyapW}$ and
\eqref{eq:hyp_rosenthal_fort_moulines}, note that
$(M_k)_{k \in \nset}$ is a $(\mcf_n)_{n \in\nset}$-martingale with
martingale increment
$(\Delta M_k  = \beta_{k+1}\{\gPois(\State_{k+1}) - \MK \gPois(\State_{k})\} )_{k
  \in\nset}$. Therefore, using  \cite[Theorem 8.6]{osekowski:2012} and Jensen inequality, we have
\begin{align*}
 \expeMarkovLigne{z}{\abs{M_n}^{p}} &\leq \textstyle{ p^p  \expeMarkovLigne{z}{\abs{\sum_{k=0}^{n-1} \Delta M_k^2}^{p/2}}} \\
    & \leq  \txts p^p \{\sum_{k=0}^{n-1} \beta_{k+1}^2\}^{p/2-1}\sum_{k=0}^{n-1} \beta_{k+1}^2 \expeMarkovLigne{z}{\abs{\gPois(\State_{k+1}) - \MK \gPois(\State_{k})}^p} \eqsp. 
\end{align*}
Using \eqref{eq:hyp_rosenthal_fort_moulines} and Jensen inequality, we get
\begin{align}
  \nonumber
  \expeMarkovLigne{z}{\abs{M_n}^{p}} &\leq
                                       \txts 2^{p-1} p^p \{\sum_{k=0}^{n-1} \beta_{k+1}^2\}^{p/2-1}\sum_{k=0}^{n-1} \beta_{k+1}^2 \expeMarkovLigne{z}{\abs{\gPois(\State_{k+1})}^p +\abs{ \MK \gPois(\State_{k})}^p}\\
  \nonumber
  & \leq
    \txts 2^{p} p^p \{\sum_{k=0}^{n-1} \beta_{k+1}^2\}^{p/2-1}\VnormD{\hg}{\lyapW}^p\sum_{k=0}^{n-1} \beta_{k+1}^2 [\abs{\expeMarkovLigne{z}{\lyapW(\State_{k+1})^p} - \pi(\lyapW^p)} + \pi(\lyapW^p) ] \\
    \label{eq:proof_ros_fort_moulines_M}
&\txts   \leq 2^{p} p^p\VnormD{\hg}{\lyapW}^p \{\sum_{k=0}^{n-1} \beta_{k+1}^2\}^{p/2}\{\CW \lyapV(z) + \pi(\lyapW^p)\}  \eqsp. 
\end{align}
Using \eqref{eq:hyp_rosenthal_fort_moulines} and  Jensen inequality, we get that
\begin{align}
  \nonumber
  \txts \expeMarkovLigne{z}{\abs{R_{1,n}}^{p}}  & \txts \leq \{\sum_{k=1}^{n-1}|\beta_k-\beta_{k+1}|\}^{p-1} \sum_{k=1}^{n-1} |\beta_{k} - \beta_{k+1}| \expeMarkov{z}{\abs{\MK \hg(Z_k)}^p} \\
    \label{eq:proof_ros_fort_moulines_R_1}
& \txts  \leq \{\sum_{k=1}^{n-1}|\beta_k-\beta_{k+1}|\}^{p} \VnormD{\hg}{\lyapW}^p   \{\CW \lyapV(z) + \pi(\lyapW^p)\}\eqsp. 
\end{align}
Finally, by \eqref{eq:hyp_rosenthal_fort_moulines} and Young inequality, we get
\begin{equation}
  \label{eq:proof_ros_fort_moulines_R_2}
  \begin{split}
  \PE_{z}[\abs{R_{2,n}}^{p}]   &\leq  2^{p-1} \beta_1^p \abs{\MK \hg(z)}^p +  2^{p-1} \beta_n^p\expeMarkov{z}{\abs{\MK \hg(Z_n)}^p}  \\
  &\leq 2^{p-1} \{\beta_1^p + \beta_n^p \} \VnormD{\hg}{\lyapW}^p   \{\CW \lyapV(z) + \pi(\lyapW^p)\} \eqsp. 
  \end{split}
\end{equation}
Combining \eqref{eq:bound_norm_Poisson}, \eqref{eq:proof_ros_fort_moulines_M}, \eqref{eq:proof_ros_fort_moulines_R_1} and \eqref{eq:proof_ros_fort_moulines_R_2} in  \eqref{eq:proof_ros_fort_moulines_S} completes the proof. 

\end{proof}

\begin{proposition}[Proposition 13, \cite{fort:moulines:2003}]
  \label[proposition]{propo:subgeo_convergence_fort_moulines}
  Assume that $\MK$ is irreducible and aperiodic and satisfies for $\mathbf{W}, \mathbf{f} : \msz \to \coint{1,\plusinfty}$, $\VnormD{\mathbf{f}}{\mathbf{W}} \leq 1$, $b \in \rset_+$ and $\msc \in \mcz$,
  \begin{equation*}
    \MK \mathbf{W} \leq \mathbf{W} - \mathbf{f} + b \1_{\msc} \eqsp.
  \end{equation*}
  Assume in addition that $\msc \cup \{\mathbf{f} \leq 2b\} \subset \msd$, where $\msd$ is a $(m,\epsilon)$-small set and $\sup_{\msd} \mathbf{W} < \plusinfty$. Then, for any distribution $\lambda, \mu$ on $\msz$, $\lambda(\mathbf{f}),\mu(\mathbf{f}) < \plusinfty$, we have
  \begin{equation*}
\txts    \sum_{n\in\nset} \VnormD{\lambda \MK^n - \mu \MK^n }{\mathbf{f}} \leq 8\epsilon^{-1}\{bm+\sup_{\msd} \mathbf{W}\} +2\{\lambda(\mathbf{W})+\mu(\mathbf{W})\} \eqsp, 
  \end{equation*}
\end{proposition}

\begin{proposition}
  \label[proposition]{propo:application_rosenthal}
  Assume \Cref{assum:drift}.
  \begin{enumerate}[leftmargin=5mm, label=\alph*)]
\item \label{propo:application_rosenthal_a} For any $\upgamma >0$,   the inequality \eqref{eq:hyp_rosenthal_fort_moulines} holds   with $f \leftarrow W^{\upgamma}$, $\lyapW \leftarrow W^{\upgamma +1 - \delta}/\cupgamma$ and $\lyapV \leftarrow W^{p(\upgamma+1-\delta)+1-\delta}/[\cupgamma^p \cupgammaD{p(\upgamma+1-\delta)}]$ and
    \begin{equation}
      \label{eq:def_cf_cw_ue}
      \begin{aligned}
                \Cf(\upgamma) & = \uppsi(\upgamma) \eqsp, \qquad  \CW(\upgamma) = \uppsi(p(\upgamma+1-\delta)) \eqsp
\end{aligned}
\end{equation}
where for any $\tupgamma >0$, $\uppsi(\tupgamma) = 8 \varepsilon_{\tRupgammaD{\tupgamma}} \defEnsLigne{\bupgammaD{\tupgamma}/\cupgammaD{\tupgamma} m_{\tRupgammaD{\tupgamma}} + \tRupgammaD{\tupgamma}^{\upgamma+1-\delta}} +2[ \bupgammaD{\tupgamma+1-\delta}/\cupgammaD{\tupgamma+1-\delta} +1]$, 
 $\tRupgammaD{\tupgamma} = \{2 \bupgammaD{\tupgamma}/\cupgammaD{\tupgamma}\}^{1/\tupgamma}\vee \RupgammaD{\tupgamma}$, and  $\RupgammaD{\tupgamma},\bupgammaD{\tupgamma},\cupgammaD{\tupgamma}$ are given in \Cref{lem:poly_drift_condition}.

\item  \label{propo:application_rosenthal_b} For any $\upgamma >0$ and $p \geq 1$, \eqref{eq:rosenthal_fort_moulines} holds with $f \leftarrow W^{\upgamma}$, $\lyapV \leftarrow W^{p(\upgamma+1-\delta)+1-\delta}$ and
    \begin{equation}
  \label{eq:def_C_ros_ue}
  \Cros{p} = 6^p \Cf^p(\upgamma)\{\CW(\upgamma)+\bupgammaD{p(\upgamma+1-\delta)}/[\cupgamma^p \cupgammaD{p(\upgamma+1-\delta)}]\}(p^p +2)/[\cupgamma \cupgammaD{p(\upgamma+1-\delta)}] \eqsp. 
\end{equation}
where $\Cf(\upgamma),\CW(\upgamma)$ are defined in \eqref{eq:def_cf_cw_ue}. 
  \end{enumerate}
\end{proposition}

\begin{proof}
  First note that \ref{propo:application_rosenthal_b} is an easy consequence of \ref{propo:application_rosenthal_a}, \Cref{propo:rosenthal_adapt_fort_moulines} and \Cref{coro:moment:W}.

  We now show \ref{propo:application_rosenthal_a}. Let $\tupgamma >0$. \Cref{lem:poly_drift_condition} shows that 
  \begin{equation*}
\cupgammaD{\tupgamma}^{-1} \MK W^{\tupgamma+1-\delta} \leq  \cupgammaD{\tupgamma}^{-1} W^{\tupgamma+1-\delta}  - W^{\tupgamma} + \bupgammaD{\tupgamma}/\cupgammaD{\tupgamma} \1_{\msc_{\RupgammaD{\tupgamma}}} \eqsp.
  \end{equation*}
  Then, using that for any $R \geq 0$, $\{ W \leq R\}$ is an
  $(\epsilon_R,m_R)$-small set for $\MK$ under \Cref{assum:drift}, $\msc_{\RupgammaD{\tupgamma}} \cap \{ W^{\tupgamma} \leq 2 \bupgammaD{\tupgamma}/\cupgammaD{\tupgamma} \} \subset \msc_{\RupgammaD{\tupgamma}}$, \Cref{propo:subgeo_convergence_fort_moulines} and \Cref{coro:moment:W}, we get that for any $z \in \msz$,
  \begin{equation*}
    \sum_{n\in\nset} \VnormD{\updelta_z \MK^n - \pi}{W^{\tupgamma}} \leq \uppsi(\upgamma) W^{\tupgamma +1-\delta}(z)/\cupgammaD{\tupgamma} \eqsp, 
  \end{equation*}
  where $\uppsi$ is defined by \eqref{eq:def_cf_cw_ue}.
  Applying this result for $\tupgamma \leftarrow \upgamma$ and $\tupgamma \leftarrow p(\upgamma+1-\delta)$ completes the proof. 
\end{proof}
\begin{proposition}
  \label[proposition]{propo:application_rosenthal_V}
  Assume \Cref{assum:drift}.
  \begin{enumerate}[leftmargin=5mm, label=\alph*)]
\item \label{propo:application_rosenthal_V_a} For any $\uptau \geq 1$,   the inequality \eqref{eq:hyp_rosenthal_fort_moulines} holds   with $f \leftarrow V^{1/\uptau}$, $\lyapW \leftarrow V^{1/\uptau}/(1-\lambda^{1/\uptau})$ and $\lyapV \leftarrow V/[(1-\lambda)(1-\lambda^{1/\uptau})]$ and
    \begin{equation}
      \label{eq:def_cf_cw_ue_V}
      \begin{aligned}
        \Cf(\uptau) & = \upphi(\uptau) \eqsp, \qquad  \CW(\uptau) = \upphi(1) \eqsp
\end{aligned}
\end{equation}
with for any $\tuptau >0$, $\upphi(\tuptau) = 8 \varepsilon_{\Rtuptau} \defEnsLigne{\bb^{1/\tuptau}/(1-\lambda^{1/\tuptau}) m_{\Rtuptau} + 2 \bb^{1/\tuptau}/(1-\lambda^{1/\tuptau})} +2[ \bb/(1-\lambda) +1]$, $\Rtuptau  = \log(R_0) \vee \log[2^{\tuptau} \bb/(1-\lambda^{1/\tuptau})^{\tuptau}]$ and $\lambda$ is defined by \eqref{eq:def_lambda}.
\item  \label{propo:application_rosenthal_V_b} For any $p \geq 1$, \eqref{eq:rosenthal_fort_moulines} holds with $f \leftarrow V^{1/p}$, $\lyapV \leftarrow V$ and
    \begin{equation}
  \label{eq:def_C_ros_ue_V}
  \Dros{p} = 6^p \Cf^p(p)\{\CW(p)+b/(1-\lambda)\}(p^p +2)/[(1-\lambda)(1-\lambda^{1/p})] \eqsp,
\end{equation}
where $\Cf(p),\CW(p)$ are defined in \eqref{eq:def_cf_cw_ue_V}. 
  \end{enumerate}
\end{proposition}

\begin{proof}
  First note that \ref{propo:application_rosenthal_V_b} is an easy consequence of \ref{propo:application_rosenthal_V_a}, \Cref{propo:rosenthal_adapt_fort_moulines} and \Cref{coro:moment:W}.
  
  Let $\tuptau \geq 1$.  First, Jensen inequality, the fact
  that $t \mapsto t^{1/\tuptau}$ is sub-additive on $\rset_+$ and
  \Cref{assum:drift} and the definition of $\lambda$
in   \eqref{eq:def_lambda} imply that
  $\MK V^{1/\tuptau} \leq \lambda^{1/\tuptau} V^{1/\tuptau} +
  \bb^{1/\tuptau}\1_{\msc_0} = V^{1/\tuptau} - (1-\lambda^{1/\tuptau})
  V^{1/\tuptau} + \bb^{1/\tuptau}\1_{\msc_0} $. Therefore, 
  since $\lambda \in \coint{0,1}$,
  $\msc_0 \cup \{ V^{1/\tuptau} \leq 2
  \bb^{1/\tuptau}/(1-\lambda^{1/\tuptau})\} \subset \msc_{\Rtuptau}$,
 using \Cref{propo:subgeo_convergence_fort_moulines}  and by \Cref{coro:moment:W}, $\pi(V^{1/\tuptau})\leq \pi(V) \leq b/(1-\lambda)$ we obtain that $\sum_{n \in \nset} \VnormD{\updelta_z \MK^n - \pi}{V^{1/\tuptau}} \leq \upphi(\tau) V^{1/\tuptau}(z)/(1-\lambda^{1/\uptau})$ for any $z \in \msz$. Applying this result  twice for $\tuptau \leftarrow \uptau$ and $\tuptau \leftarrow 1$ completes the proof.  
\end{proof}

\begin{lemma}
\label[lemma]{lem: moments of vareps}
Under assumptions of \Cref{th:approximation_error} for any $1 \le q \le \LL$, $z\in \Zset$ and $j \in \nset$
\begin{equation}
    \begin{aligned}
       \PE_z^{1/q}[\norm{\funcAt{Z_j}}^{q}] & \le \bConst{A} V^{1/\LL}(z),  \\
       \PE_z^{1/q}[\norm{\funcb{Z_{j}} - b}^q] & \le \bConst{b} V^{1/\LL}(z) \eqsp,
      \end{aligned}
\end{equation}
where 
\begin{equation}
\label{eq:AepsBd2}
    \bConst{A}: = (\normop{A} +d  \Const{A} (\beta \LL/\rme)^\beta \{1 + b/(1-\lambda)\}^{1/\LL})
\end{equation}
and
\begin{equation}
\label{eq:BepsBd2}
    \bConst{b}: = (\normop{b} +d  \Const{b} \{1 + b/(1-\lambda)\}^{1/\LL}). 
\end{equation}
\end{lemma}
\begin{proof}
We first note that using 
\begin{align*}
\PE_z^{1/q}[\norm{\funcAt{Z_j}}^{q}] &\le \normop{A} +\PE_z^{1/q}[\norm{\funcA{Z_j}}^{q}] \le \normop{A} +d \Const{A} \sup_{x \geq 1} \frac{\log^\beta x}{x^{1/\LL}} \{\MK^j V^{q/\LL}(z) \}^{1/q} \\
&  \le (\normop{A} +d  \Const{A} (\beta \LL)^\beta \rme^{-\beta}\{1 + b/(1-\lambda)\}^{1/\LL}) V^{1/\LL}(z) .
\end{align*}
Similarly, one may prove the second statement of the lemma. 
\end{proof}\vspace{-.4cm}


\section{Proofs for \Cref{th:expconvproducts}}
\label{sec:detailed_proof_product_mat}
In this section, we provide the core lemmas that are employed for the proof of \Cref{th:expconvproducts} in \Cref{sec:proof}.

\subsection{Technical and preliminary results}

\begin{lemma}[Lyapunov Lemma]
\label[lemma]{lem:lyapunov}
A matrix $A$ is Hurwitz if and only if for any positive symmetric matrix $P=P^\top \succ 0$ there is  $Q =Q^\top \succ 0$ that satisfies the Lyapunov equation
\[
A^\top Q + Q A = -P \,.
\]
In addition, $Q$ is unique.
\end{lemma}
\begin{proof}
See \cite[Lemma 9.1, p. 140]{poznyak:control}.
\end{proof}

\begin{lemma}
\label[lemma]{lem:Hurwitzstability}
Assume that $-A$ is a Hurwitz matrix. Let $Q$ be the unique solution of the Lyapunov equation
$A^\top Q + Q A =  \Id$.
Then, for any $\alpha \in [0, (1/2) \normop{A}[Q]^{-2} \normop{Q}^{-1}]$,
we get $\normop{\Id - \alpha A}[Q]^2 \leq (1 - a \alpha)$
with $a = (1/2) \normop{Q}^{-1}$. In particular, for any $\alpha \in [0, (1/2) \normop{A}[Q]^{-2} \normop{Q}^{-1}]$, $\normop{\Id - \alpha A}[] \leq \sqrt{\qcond}(1 - a \alpha/2)$, where $\qcond = \lambda_{\sf min}^{-1}( Q )\lambda_{\sf max}( Q )$. If in addition $\alpha \le \normop{Q}^2$ then
$ 1 - a \alpha \geq 1/2$.
\end{lemma}
\begin{proof}
For any  $x \in \rset^d\setminus\{0\}$, we get
\[
\frac{x^\top (\Id - \alpha A)^\top Q (\Id - \alpha A) x}{x^\top Q x}
=1 - \alpha \frac{\norm{x}^2}{x^\top Q x} + \alpha^2 \frac{x^\top A^\top Q A x}{x^\top Q x}
\]
Hence, we get that for all $\alpha \in [0, (1/2) \normop{A}[Q]^{-2} \normop{Q}^{-1}]$,
\begin{align}
1 - \alpha \frac{\norm{x}^2}{x^\top Q x} + \alpha^2
\frac{x^\top A^\top Q A x}{x^\top Qx}
&\leq 1 - \alpha  \normop{Q}^{-1} + \alpha^2 \normop{A}[Q]^2 \leq 1 - (1/2) \normop{Q}^{-1} \alpha \,. \nonumber
\end{align}
The proof is completed using that for any $t \in \ccint{0,1}$, $(1-t)^{1/2} \leq 1-t/2$ and that for any matrix $A \in \rset^{d \times d}$, $\norm{A}[Q] \leq \qcond^{1/2} \norm{A}$.
\end{proof}


\begin{lemma}
\label{lem:technical_lemma_1}
Let $\upalpha > 0$ and $(u_i)_{i \geq 1}$ be a sequence of non-negative numbers. Then, for any $n \in \nset$, $n \geq 2$, and any $\epsilon \in \ooint{0,1}$,  
\[
\txts 1 + \upalpha^{n}\prod_{i=1}^{n}u_{i} \leq \exp{\{\upalpha^{1+\epsilon}(1+\epsilon)^{-1}\sum_{i=1}^{n}u_{i}^{1+\epsilon}\}} \eqsp. 
\]
\end{lemma}
\begin{proof}
  First note that for any $\beta \geq 1$ and $t \geq 0$, $1 \leq t^{1/\beta}(t^{-1}+1)$ which implies that $1 + t \leq \exp{\parentheseLigne{\beta t^{1/\beta}}}$.
Using this inequality for $\beta = n/(1+\epsilon)$ and the inequality of arithmetic and geometric means, we get
\begin{align*}
\txts 1 + \upalpha^{n}\prod_{i=1}^{n}u_{i} 
&\txts \leq  \exp{\parentheseLigne{\frac{n}{1+\epsilon}\upalpha^{1+\epsilon}(\prod_{i=1}^{n}u_{i})^{(1+\epsilon)/n}\}} }
 \leq \exp\parentheseLigne{\frac{n}{1+\epsilon}\upalpha^{1+\epsilon}(\frac{1}{n}\sum_{i=1}^{n}u_{i})^{1+\epsilon}} \\
&\txts  \leq \exp{\parentheseLigne{\frac{\upalpha^{1+\epsilon}}{1+\epsilon}\sum_{i=1}^{n}u_{i}^{1+\epsilon}}}\eqsp,
\end{align*}
where the last inequality follows from Jensen's inequality.
\end{proof}

\begin{lemma}
\label[lemma]{lem:trick_2}
Let $(\mathfrak F_\ell)_{\ell \geq 0}$ be some filtration and a sequence of non-negative random variables $(\xi_\ell)_{\ell \geq 0}$ is adopted to this filtration.
Then, for any $N \in \nset$, $p \in \nset$, it holds
\begin{equation}
\label{eq:cond_expectation_bound_product}
 \PE \left[ \prod_{\ell=1}^{N} \xi_{\ell}^p \right] \le   \bigg\{\PE \bigg[\prod_{\ell=1}^{N} \PE[ \xi_\ell^{2p} | \mathfrak F_{\ell-1} ] \bigg]\bigg\}^{1/2} \eqsp. 
\end{equation}
\end{lemma}
\begin{proof}
Denote for any $k \in \nset$, $p \in \nset$,
$$
B_{k,p}= ( \prod_{\ell=1}^k \PE[ \xi_\ell^p | \mathfrak F_{\ell-1} ] )^{-1}, \, B_{0,p} : = 1, \,
y_{k,p} = \xi_{k}^p y_{k-1}, \,  y_{0,p}: = 1 \eqsp.
$$
It is straightforward to check that for any $p \in \nset$,
\begin{align*}
\PE[B_{k+1,p} y_{k+1,p}] & = \PE[B_{k+1,p} \PE[ \xi_{k+1}^p | \mathfrak F_k] y_{k,p} ] = \PE[B_{k,p} y_{k,p}] = \cdots = \PE[ B_{0,p} y_{0,p} ] = 1\eqsp.
\end{align*}
This fact implies that
\begin{align*}
\PE\left[ \prod_{\ell=1}^{N} \xi_{\ell} \right] &= \PE[y_{N}] = \PE[y_{N,p} B_{N,2p}^{1/2} B_{N,2p}^{-1/2}] \\
&\le \{\PE [y_{N,2p} B_{N,2p}]\}^{1/2} \{\PE [B_{N,2p}^{-1}] \}^{1/2} = \bigg\{\PE \bigg[\prod_{\ell=1}^{N} \PE[ \xi_\ell^{2p} | \mathfrak F_{\ell-1} ] \bigg] \bigg\}^{1/2} \eqsp.
\end{align*}
Hence, ~\eqref{eq:cond_expectation_bound_product} is proved.
\end{proof}

\subsection{Core Lemmas}
\begin{lemma}
\label[lemma]{lem:deterministic_part_control} Assume that the conditions of  \Cref{th:expconvproducts} holds. Then, for any $\ell\in\{1,\ldots,N\}$, and $p \geq 1$,
\begin{equation*}
\bigl(1+ \qcond^{\half} \rme^{a}\normop{R_{\ell}}\bigr)^p \leq \exp\{pC^{(0)} h^2 \alpha_{j_{\ell-1}+1}^2 \}\eqsp,
\end{equation*}
where  $R_\ell$ is given in \eqref{eq:RlRlbar_def} and 
\begin{equation}
\label{eq:c_0_definition}
C^{(0)} = (1/2)\qcond^{\half} \normop{A}^2 \exp\parentheseLigne{\normop{A}+a}\eqsp.
\end{equation}
\end{lemma}
\begin{proof}
  Let $\ell\in\{1,\ldots,N\}$, and $p \geq 1$.
Using the definition of $R_\ell$ and since $(\alpha_i)_{i \in \nset}$ is non-increasing, we get 
\begin{align*}
\normop{R_{\ell}} &\leq \sum_{r=2}^{h}\binom{h}{r}\alpha_{j_{\ell-1}+1}^{r}\normop{A}^{r} \leq \alpha_{j_{\ell-1}+1}^2\normop{A}^2\sum_{r=0}^{h-2}\binom{h}{r+2}\alpha_{j_{\ell-1}+1}^{r}\normop{A}^{r} \\
&\leq 2^{-1}\alpha_{j_{\ell-1}+1}^2h^2\normop{A}^2(1+\alpha_{j_{\ell-1}+1}\normop{A})^{h-2} \leq 2^{-1}\alpha_{j_{\ell-1}+1}^2h^2\normop{A}^2 \exp\{\alpha_{j_{\ell-1}+1}h\normop{A}\}\eqsp,
\end{align*}
where we have used for the last two inequalities, the upper bounds $\binom{h}{r+2} \leq \binom{h-2}{r}(h^2/2)$ for any $r \in \{0,\ldots,h-2\}$ and $(1+t) \leq \rme^t$ for any $t \geq 0$.
It yields using that $\alpha_{\infty,p}h \leq 1$ that $\normop{R_{\ell}} \leq 2^{-1}\alpha_{j_{\ell-1}+1}^2h^2\normop{A}^2 \rme^{\norm{A}}$. 
The proof is then completed using the bound $(1+t) \leq \rme^t$ for any $t \geq 0$ again.
\end{proof}

\begin{lemma}
  \label[lemma]{lem:remainder_bound}
  Assume that the conditions of  \Cref{th:expconvproducts} holds. Then, for any $\ell\in\{1,\ldots,N\}$, and $p \geq 1$, almost surely it holds 
\begin{equation*}
\txts \bigl(1 + \qcond^{\half} \rme^{a}\normop{\bar{R}_{\ell}}\bigr)^{2p} \leq \exp{\{2^{h+1}p C^{(1)}  \alpha^{1+\varepsilon}_{j_{\ell-1}+1} \sum_{k=j_{\ell-1}+1}^{j_{\ell}}W^{\delta}(\State_{k-m})\}}\eqsp,
\end{equation*}
where $\bar{R}_{\ell}$ is defined by \eqref{eq:RlRlbar_def} and 
\begin{equation}
\label{eq:c_1_definition}
C^{(1)} = (\qcond^{\half} d \rme^a \Const{A})^{1+\varepsilon}/(1+\varepsilon)\eqsp.
\end{equation}
\end{lemma}
\begin{proof}
    Let $\ell\in\{1,\ldots,N\}$, and $p \geq 1$.
Using the definition of $\bar{R}_\ell$, we consider the following decomposition 
\[
\txts \bar{R}_{\ell} = \sum_{r=2}^{h}\bar{R}^{(r)}_{\ell}, \quad \bar{R}^{(r)}_{\ell} = (-1)^{r}\sum_{(i_1,\dots,i_r)\in\msi_r^\ell}(\prod_{u=1}^{r}\alpha_{i_u})A^{r}\eqsp. 
\]
Then, using that $(1+a+b) \leq (1+a)(1+b)$ for $a,b \geq 0$ and $(\alpha_i)_{i \in \nset}$ is non-increasing yields
\begin{align*}
\txts  (1 + \qcond^{\half} \rme^{a}\normop{\bar{R}_{\ell}})^{2p} & \txts\leq \prod_{r=2}^{h}(1+\qcond^{\half} \rme^{a}\normop{\bar{R}^{(r)}_{\ell}})^{2p} \\
  &\txts\leq \prod_{r=2}^{h}\prod_{(i_1,\dots,i_r) \in \msi_r^\ell}(1 + \qcond^{\half} \rme^{a} \alpha^{r}_{j_{\ell-1}+1}\prod_{k=1}^{r}\normop{ \funcA{\State_{i_k-m}}})^{2p}    
\end{align*}
 Using \Cref{lem:technical_lemma_1}, $r \geq 2$ and $\varepsilon \in \ooint{0,1}$ in \Cref{assum:almost_bounded}, 
\begin{align*}
(1 + \qcond^{\half} \rme^{a}\normop{\bar{R}_{\ell}})^{2p} 
  & \txts \leq \prod_{r=2}^{h}\prod_{(i_1,\dots,i_r)\in\msi_r^{\ell}}\exp\parenthese{\frac{2p(\qcond^{\half} \rme^{a})^{1+\varepsilon}\alpha^{1+\varepsilon}_{j_{\ell-1}+1}}    {1+\varepsilon}\sum_{u=1}^{r}\normop{\funcA{\State_{i_u-m}}}^{1+\varepsilon}} \\
  & \txts \leq \exp\parenthese{\frac{2 p (\qcond^{\half} \rme^{a})^{1+\varepsilon} \alpha^{1+\varepsilon}_{j_{\ell-1}+1}}{1+\varepsilon}\sum_{k=j_{\ell-1}+1}^{j_{\ell}}\normop{\funcA{\State_{k-m}}}^{1+\varepsilon} \sum_{r=2}^{h}\sum_{(i_1,\dots,i_r)\in\msi_r^{\ell}}} \\
  &\leq
\exp\parenthese{\frac{2^{h+1} p (\qcond^{\half} \rme^{a})^{1+\varepsilon} \alpha^{1+\varepsilon}_{j_{\ell-1}+1}}{1+\varepsilon}\sum_{k=j_{\ell-1}+1}^{j_{\ell}}\normop{\funcA{\State_{k-m}}}^{1+\varepsilon}}
    \eqsp.
\end{align*}
The proof follows from \Cref{assum:almost_bounded} which implies that $\normop{\funcA{z}}^{1+\varepsilon} \leq \frobnorm{\funcA{z}}^{1+\varepsilon} \leq d^{1+\varepsilon}\Const{A}^{1+\varepsilon}W^{\delta}(z)$, for any $z\in\msz$.
\end{proof}

Under \Cref{assum:almost_bounded}, define for $n \in \nset$, $n \geq 1$, $(\upalpha_i)_{i \in \nset}$ a non-increasing positive sequence,
\begin{equation}
  \label{eq:def_r_a}
  r_A = \min\{ s \geq 0 \, : \, \beta \leq 2\delta - 1 - (1-\delta)/s\} \eqsp, \qquad \tS_n = \sum_{k = 1}^{n} \upalpha_{k}(\funcA{\State_{k}} - A) \eqsp. 
\end{equation}
\begin{lemma}
\label[lemma]{lem:rosenthal_part}
Assume that the conditions of  \Cref{th:expconvproducts} holds. For any $n \in \nsets$,   $p \geq 1 \vee (r_A/4)$, 
\begin{align*}
\PE_{z}^{\nicefrac{1}{4p}}[(1 + \qcond^{\half} \rme^{a}\normop{\tS_n})^{4p}] \leq \exp\{C^{(2)}_{p}\upalpha_{1} n^{1/2}W^{\delta}(z)\}\eqsp,
\end{align*}
where $r_A,\tS_n$ are defined in \eqref{eq:def_r_a}, and  
\begin{equation}
\label{eq:c_2_p_definition_0}
C^{(2)}_{p} = \qcond^{\half} \rme^{a} d \Const{A}(4\Cros{4p})^{\nicefrac{1}{4p}}\eqsp,
\end{equation}
with $\Cros{4p}$ given in \eqref{eq:def_C_ros_ue}.
\end{lemma}
\begin{proof}
  First by Minkowski's inequality, we get
  \begin{equation}
    \label{eq:proof_rosenthal_part_1}
    \PE_{z}^{\nicefrac{1}{4p}}[(1 + \qcond^{\half} \rme^{a}\normop{\tS_n})^{4p}] \leq 1+ \qcond^{\half} \rme^{a} \PE_{z}^{\nicefrac{1}{4p}}[\normop{\tS_n}^{4p}] \eqsp.
  \end{equation}
  In addition, note that denoting by $[\tS_n]_{i,j}$, the $(i,j)$-th component of $\tS_n$, using the Jensen inequality, we get
  \begin{align}
    \nonumber
\txts \PE_{z}[\normop{\tS_n}^{4p}] & \txts \leq \PE_{z}[(\sum_{i_1,i_2=1}^{d}[\tS_n]^2_{i_1,i_2})^{4p/2}]
                                                \leq d^{4p-2}\PE_{z}[\sum_{i_1,i_2=1}^{d}\lvert[\tS_n]_{i_1,i_2}\rvert^{4p}] \\
\label{eq:proof_rosenthal_part_2}  
&\txts  \leq d^{4p}\max_{i_1,i_2 \in \{ 1, \cdots, d\}}\PE_{z}[\lvert[\tS_n]_{i_1,i_2}\rvert^{4p}]\eqsp.
\end{align}
  
Using \Cref{assum:drift} and applying \Cref{propo:application_rosenthal}-\ref{propo:application_rosenthal_b} with $\upgamma \leftarrow \beta$ and using that $(\upalpha_i)_{i \in \nset}$ is non-increasing, we obtain that 
for any $i_1,i_2\in\{1,\ldots,N\}$,
\begin{equation}
\label{eq:proof_rosenthal_part_3} 
  \PE_{z}[\lvert[\tS_n]_{i_1,i_2}\rvert^{4p}] \leq \Cros{4p}C_A^{4p}(\upalpha_1^{4p}n^{4p/2} + 3 \upalpha_1^{4p}) W^{4p(\beta+1-\delta)+1-\delta} \leq 4 \Cros{4p}C_A^{4p}\upalpha_1^{4p}n^{2p}  W^{4p \delta}(z) \eqsp,
\end{equation}
using for the last inequality that $W(z) \geq 1$ and $4p(\beta+1-\delta)+1-\delta \geq 4p \delta$ since $4p \geq r_A$, $\beta \leq 2\delta - 1 - (1-\delta)/r_A$ by \eqref{eq:def_r_a}.
Combining \eqref{eq:proof_rosenthal_part_1}-\eqref{eq:proof_rosenthal_part_2}-\eqref{eq:proof_rosenthal_part_3}, we get 
  \begin{equation*}
    \PE_{z}^{\nicefrac{1}{4p}}[(1 + \qcond^{\half} \rme^{a}\normop{\tS_n})^{4p}] \leq 1+ \qcond^{\half} \rme^{a} [4\Cros{4p}]^{1/4p} d C_A\upalpha_1n^{\half}  W^{\delta}(z) \eqsp.
  \end{equation*}
Using that $1+t \leq \rme^t$ completes the proof. 
\end{proof}

\begin{corollary}
\label[corollary]{coro:rosenthal_part}
Assume that the conditions of  \Cref{th:expconvproducts} holds. For any $n \in \nsets$,   $p \geq 1$, 
\begin{align*}
\PE_{z}^{\nicefrac{1}{4p}}[(1 + \qcond^{\half} \rme^{a}\normop{\tS_n})^{4p}] \leq \exp\{C^{(2)}_{p}\upalpha_{1} n^{1/2}W^{\delta}(z)\}\eqsp,
\end{align*}
where $r_A,\tS_n$ are defined in \eqref{eq:def_r_a}, and  
\begin{equation}
\label{eq:c_2_p_definition}
C^{(2)}_{p} = \qcond^{\half} \rme^{a} d \Const{A}(4\Cros{4\tilde{p}})^{\nicefrac{1}{4\tilde{p}}}\eqsp, \quad \tilde{p} = \max(p,r_A/4) \eqsp,
\end{equation}
with $\Cros{4\tilde{p}}$ given in \eqref{eq:def_C_ros_ue}.
\end{corollary}

\begin{proof}
  The proof is a simple consequence of  \Cref{lem:rosenthal_part} and  Jensen's inequality. 
\end{proof}

Note that \Cref{assum:drift} implies for any $z \in \msz$,
\begin{equation}
\label{eq:PV_exp_bound}
\MK V(z) \leq \exp\parenthese{W(z) - cW^{\delta}(z) + \tilde{\bb}\1_{\msc_0}(z)}\eqsp,
\end{equation}
where
\begin{equation}
\label{eq:bb_tilde_def}
\tilde{\bb} = \log{\bb} + \sup_{r > 0}\{cr^{\delta} - r\}\eqsp.
\end{equation}
Similarly, \eqref{eq:lem:drift_conditions_technical} implies for any $h \in\nset$ and $z \in\msz$,
\begin{equation}
\label{eq:PV_nsteps_exp_bound}
\MK^{h} V(z) \leq \exp\parenthese{W(z) - cW^{\delta}(z) + \bb^{\prime}\1_{\msc_{R_1}}(z)}\eqsp,
\end{equation}
where 
\begin{equation}
\label{eq:bb_prime_def}
\bb^{\prime} = \log{\{\bb/(1-\lambda)\}} + \sup_{r > 0}\{cr^{\delta}-r\}\eqsp.
\end{equation}

\begin{lemma}
\label[lemma]{lem:supermartingale}
Assume \Cref{assum:drift}. Let $(\upalpha_i)_{i \in \nsets}$ be a non-increasing sequence, such that $0 < \upalpha_i \leq 1$ for any $i \geq 1$.  Then, for any $z \in \Stateset$ and $n \in \nset$,
\begin{equation*}
\txts\PE_{z}[\exp\{c\sum_{k=0}^{n-1}\upalpha_{k}W^{\delta}(\State_{k})\}] \leq 
\exp{\{\tilde{\bb}\sum_{k=0}^{n-1}\upalpha_{k}\}} \exp{\{\upalpha_{1} W(z)\}}\eqsp,
\end{equation*}
where $\tilde{\bb}$ is given in \eqref{eq:bb_tilde_def} and $c$ in \Cref{assum:drift}.
\end{lemma}
\begin{proof}
Define, for $n \geq 0$,
\begin{equation}
\label{eq:M_n_def} 
\txts M_n = \exp\{\upalpha_n W(\State_n) + \sum_{k=0}^{n-1}\upalpha_{k}(cW^{\delta}(\State_{k}) - \tilde{\bb}\1_{\msc_0}(\State_k))\}\eqsp,
\end{equation}
with the convention $\sum_{k=0}^{-1}= 0$. Consider $(\mcf_n)_{n \in\nset}$, the canonical filtration:  $\mathcal{F}_n = \sigma(\State_0,\dots,\State_n)$. Then, we have  for $n \geq 1$,
\begin{align*}
\PE[M_n | \mathcal{F}_{n-1}] = M_{n-1} \exp\{-\upalpha_{n-1}W(\State_{n-1}) + \upalpha_{n-1}(cW^{\delta}(\State_{n-1}) - \tilde{\bb}\1_{\msc_0}(\State_{n-1}))\}\PE[\rme^{\upalpha_n W(\State_n)} | \mathcal{F}_{n-1}]\eqsp.
\end{align*}
Using the Markov property, $\alpha_{n} \leq \alpha_{n-1} \leq 1$, $V \geq 1$, \eqref{eq:PV_exp_bound} and Jensen's inequality, for $n \geq 1$,
\begin{align*}
\PE[\rme^{\upalpha_n W(\State_n)} | \mathcal{F}_{n-1}] &= \MK V^{\upalpha_n}(\State_{n-1}) \leq \MK V^{\upalpha_{n-1}}(\State_{n-1}) \leq (\MK V(\State_{n-1}))^{\upalpha_{n-1}} \\
&\leq \exp\{\upalpha_{n-1}(W(\State_{n-1}) - cW^{\delta}(\State_{n-1}) + \tilde{\bb}\1_{\msc_0}(\State_{n-1}))\}\eqsp.
\end{align*}
Therefore, $(M_n)_{n \geq 0}$ is $(\mathcal{F}_{n})_{n \geq 0}$-supermartingale, and $\PE_{z}[M_n] \leq \PE_{z}[M_0] \leq \rme^{\upalpha_1 W(z)}$.
We conclude the proof upon noting that $\PE_{z}[\exp\{c\sum_{k=0}^{n-1}\upalpha_{k}W^{\delta}(\State_{k})\}] \leq \exp{\{\tilde{\bb}\sum_{k=0}^{n-1}\upalpha_{k}\}}\PE_{z}[M_n]$.
\end{proof}

\begin{lemma}
\label[lemma]{lem:sceleton_supermartingale}
Assume \Cref{assum:drift}. Let $(\upalpha_i)_{i \in \nsets}$ be a non-increasing sequence, such that $0 < \upalpha_i \leq 1$ for any $i \geq 1$.  Then, for any $z \in \Stateset$ and $n \in \nset$, $h \in \nset$,
\begin{equation*}
\txts \PE_{z}[\exp\{c\sum_{k=0}^{n-1}\upalpha_{k}W^{\delta}(\State_{hk})\}] \leq 
\exp\{\bb^{\prime}\sum_{k=0}^{n-1}\upalpha_{k}\} \exp\{\upalpha_{1} W(z)\}\eqsp,
\end{equation*}
where $\bb^{\prime}$ is given in \eqref{eq:bb_prime_def} and $c$ in \Cref{assum:drift}.
\end{lemma}
\begin{proof}
The proof follows the same lines as \Cref{lem:supermartingale}, using \eqref{eq:PV_nsteps_exp_bound} in place of  \eqref{eq:PV_exp_bound}.
\end{proof}

\subsection{Proof of \Cref{th:expconvproducts}}
\label{subsec:complete-proof}
\paragraph{Details on the Step 3.} 
We have all the elements to conclude the proof of the theorem. It is essentially a question of adjusting the constants and combining the different bounds obtained above. To simplify notations, we first introduce the auxiliary quantities:
\begin{align*}
    D_p^{(1)} &=  C^{(1)} 2^h \alpha_{\infty,p}^{1 + \varepsilon} + C_p^{(2)} h^{\half}\alpha_{\infty,p} \\
    D_p^{(2)} &= C^{(0)} h \alpha_{\infty,p} + C^{(1)}  2^h   \alpha_{\infty,p}^\varepsilon \tilde{\bb} + C_p^{(2)} h^{-\half} (\tilde{\bb} - \log(1-\lambda)).
\end{align*}
Substituting \eqref{eq:conc_step_2_1}, \eqref{eq:conc_step_2_2} into \eqref{eq:after_exp} and using that $\sup_{i \in \nsets} \alpha_i \leq \alpha_{\infty,p}$, we get
\begin{align*}
    & \txts \PE_{z_m}^{1/p} [ \norm{  \ProdB_{m+1:n}(Z_{1:n-m}) }^p ] \\
    & \qquad \txts\leq \myqcond \exp( a \alpha_{\infty,p}h )  \exp \big\{ - (a/2) \sum_{i=m+1}^n \alpha_i + D_p^{(2)} h \sum_{\ell=1}^N \alpha_{j_{\ell-1}+1} \big\} \exp [ D_p^{(1)}  W(z_m) ] \eqsp.
\end{align*}
We set the block size and the upper bound to the step sizes as
{\small
\begin{gather}
\label{eq:h_def}
h = \big\lceil \big( 12 C_p^{(2)} ( \tilde{\bb} - \log(1-\lambda) ) / a \big)^2 \big\rceil \eqsp, \\
\label{eq:alpha_infty_main}
\alpha_{\infty,p} = \min \parentheseDeux{\frac{1}{a}, \frac{1}{h},  \frac{1}{2 \normopLigne{A}[Q]^2 \normopLigne{Q}}, \frac{a}{12 h C^{(0)}}, \Big( \frac{a}{12 C^{(1)} 2^h} \Big)^{\frac{1}{\varepsilon}}, \parenthese{\frac{ c \wedge \half }{ 2p C^{(1)}2^{h} }}^\frac{1}{1+\varepsilon}, \frac{ c \wedge 1 }{ 4p C_p^{(2)} h^{\half} } } \eqsp.
\end{gather}}
This yields $D_p^{(1)}  \leq 1/(2p)$, $D_p^{(2)} \leq a/4$. Together with $h \sum_{\ell=1}^N \alpha_{j_{\ell-1}+1} \leq \alpha_{\infty,p} + \sum_{i=m+1}^n \alpha_i$, we get
\begin{equation} \label{eq:final_bd_main} \txts
    \PE_{z_m}^{1/p} [ \norm{  \ProdB_{m+1:n}(Z_{1:n-m}) }^p ] \leq \myqcond \rme^{5 a \alpha_{\infty,p}h/4 }  \exp\{ - (a/4) \sum_{i=m+1}^n \alpha_i \} V^{1/(2p)}(z_m) \eqsp.
\end{equation}
Combining \eqref{eq:final_bd_main}, \eqref{proof:main_theo_1_markov_use}, and Jensen's inequality yields the statement of the theorem with the constant
\begin{equation}
\label{eq:C_st_p_proof}
\Const{\mathsf{st},p} =\qcond^{\half} \exp\parenthese{5 a \alpha_{\infty,p}h/4}\bigl(\lambda^{m/(2p)} + [\bb/(1-\lambda)]^{1/(2p)}\bigr)\eqsp.
\end{equation}


\section{Proofs of \Cref{sec:lsa}}

This section provides the missing lemmas and proofs that were required in \Cref{sec:lsa}. 

\subsection{Technical lemmas}

\begin{lemma} \label[lemma]{lem:bsum}
Let $a > 0$ and $(\alpha_k)_{k \geq 0}$ be a non-increasing sequence such that $\alpha_0 < 1/a$. Then
\[
\sum_{j=0}^{n+1} \alpha_j \prod_{l=j+1}^{n+1} (1 - \alpha_l a) = \frac{1}{a} \left\{1  - \prod_{l=1}^{n+1} (1 - \alpha_l a) \right\}
\]
\end{lemma}
\begin{proof}
Let us denote $u_{j:n+1} = \prod_{l = j}^{n+1} (1 - \alpha_l a)$. Then, for $j \in\{1,\dots,n+1\}$,
$u_{j+1:n+1} - u_{j:n+1} = a \alpha_j u_{j+1:n+1}$. Hence,
\[
\sum_{j=0}^{n+1} \alpha_j \prod_{l=j+1}^{n+1} (1 - \alpha_l a) = \frac{1}{a} \sum_{j=1}^{n+1} (u_{j+1:n+1} - u_{j:n+1}) = a^{-1} ( 1 - u_{1:n+1} ) \,.
\]
\end{proof}

\begin{lemma}
\label[lemma]{lem:bsum2}
Let $b > 0$ and $(\alpha_k)_{k \geq 0}$ be a non-increasing sequence such that $\alpha_0 < 1/(2b)$.
\begin{itemize}
    \item Assume $\alpha_k - \alpha_{k+1} \le \smallConst{\alpha} \alpha_{k+1}^2$ with $\smallConst{\alpha} \le b/2$. Then for $p \in (1, 2]$,
    \[
\sum_{k=1}^{n+1} \alpha_k^{p} \prod_{j=k+1}^{n+1}(1 - b \alpha_j) \le (2/b) \alpha_{n+1}^{p-1}.
\]
    \item 
    Assume $\alpha_k - \alpha_{k+1} \le \smallConst{\alpha} \alpha_{k+1}^2$,   $\alpha_k/\mathcal A_{k+1} \le (2/3)\smallConst{\alpha}$  with $\smallConst{\alpha} \le b/4$. We additionally assume that $\alpha_0 \le (2 \smallConst{\alpha})^{-1}$. Then for any $p \in (1,2], q \in [0,1]$
\[
\sum_{k=1}^{n+1} \alpha_k^p \mathcal A_{k}^q \prod_{j=k+1}^{n+1}(1 - b) \alpha_j) \le (2/b) \alpha_{n+1}^{p-1}\mathcal A_{n+1}^{q}.
\]
\end{itemize}

\end{lemma}
\begin{proof}
For the first part 
\begin{align*}
&\sum_{k=1}^{n+1} \alpha_k^{p} \prod_{j=k+1}^{n+1}(1 - b \alpha_j) = \alpha_{n+1}^{p-1} \sum_{k=1}^{n+1} \alpha_k \prod_{j=k+1}^{n+1} \bigg(\frac{\alpha_{j-1}}{\alpha_j} \bigg)^{p-1}  (1 - b \alpha_j) \\
&\qquad\qquad\qquad \le \sqrt{\alpha_{n+1}} \sum_{k=1}^{n+1} \alpha_k \prod_{j=k+1}^{n+1} (1 + \smallConst{\alpha}\alpha_j) (1 - b \alpha_j) \\
&\qquad\qquad\qquad \le \alpha_{n+1}^{p-1} \sum_{k=1}^{n+1} \alpha_k \prod_{j=k+1}^{n+1} (1 - (b/2) \alpha_j) \le (2/b) \alpha_{n+1}^{p-1},
\end{align*}
where on the last step we used \Cref{lem:bsum}. For the second part, we first note that
$$
\frac{\mathcal A_{j-1}}{\mathcal A_j} \le 1 + \frac{\alpha_{j-1}^2}{\mathcal A_j} \le 1 + (2/3)\smallConst{\alpha} \alpha_{j-1} \le 1 + (2/3) \smallConst{\alpha} \alpha_j + (2/3) \smallConst{\alpha}^2 \alpha_j^2 \le 1 + \smallConst{\alpha} \alpha_j.
$$
Similarly to the first part,
\begin{align*}
&\sum_{k=1}^{n+1} \alpha_k^p \mathcal A_{k}^q\prod_{j=k+1}^{n+1}(1 - b \alpha_j) = \alpha_{n+1}^{p-1} \mathcal A_{n+1}^q\sum_{k=1}^{n+1} \alpha_k \prod_{j=k+1}^{n+1} \bigg(\frac{\alpha_{j-1}}{\alpha_j} \bigg)^{p-1} \bigg(\frac{\mathcal A_{j-1}}{\mathcal A_j} \bigg)^q  (1 - b \alpha_j) \\
&\qquad\qquad\qquad \le \alpha_{n+1}^{p-1} \mathcal A_{n+1}^q \sum_{k=1}^{n+1} \alpha_k \prod_{j=k+1}^{n+1} (1 + \smallConst{\alpha}\alpha_j)^2 (1 - b \alpha_j) \\
&\qquad\qquad\qquad \le \alpha_{n+1}^{p-1} \mathcal A_{n+1}^q\sum_{k=1}^{n+1} \alpha_k \prod_{j=k+1}^{n+1} (1 - (b/2) \alpha_j) \le (2/b) \alpha_{n+1}^{p-1}\mathcal A_{n+1}^{q},
\end{align*}
where we also used Lemma~\ref{lem:bsum}.
\end{proof}

To estimate moments of $\normop{S_{j+1:n+1}}^p$ that was defined in \eqref{eq:Sdef_expansion_main}, we first derive an alternative expression for the term. For this aim we prove the following lemma. 
Define
$$
D_{j:k} : = \sum_{\ell=j}^k \alpha_\ell \funcAt{\State_{\ell}}.
$$
Here we also assume that $D_{j:k}  = 0$ if $j > k$. 
Recall that $S_{j:k} = 0$ if $j > k$ and $G_{j:k} = 0$ if $j > k+1$.  

\begin{lemma}
\label[lemma]{lem:S2representation}
For any $0 \le k \le n$
\begin{align*}
S_{k+1:n+1} = - \sum_{\ell=k+1}^{n+1} \alpha_\ell G_{\ell+1:n+1} A D_{\ell:n+1} G_{k+1:\ell-2} + \sum_{\ell=k+1}^{n+1} \alpha_{\ell-1} G_{\ell+1:n+1}  D_{\ell:n+1} A G_{k+1:\ell-2} .
\end{align*}
\end{lemma}
\begin{proof}
By definition of $D_{k:n+1}$
\begin{align*}
S_{k+1:n+1} = -\sum_{\ell=k+1}^{n+1} G_{\ell+1:n+1}  (D_{\ell:n+1}- D_{\ell+1:n+1}) G_{k+1:\ell-1}. 
\end{align*}
Simple algebraic manipulations lead to 
\begin{align*}
S_{k+1:n+1} &= -\sum_{\ell=k+1}^{n+1} G_{\ell+1:n+1}  D_{\ell:n+1} G_{k+1:\ell-1} + \sum_{\ell=k+1}^{n+1} G_{\ell:n+1}  D_{\ell:n+1} G_{k+1:\ell-2} \\
& = -\sum_{\ell=k+1}^{n+1} (G_{\ell+1:n+1} - G_{\ell:n+1}) D_{\ell:n+1} G_{k+1:\ell-2} \\
& \quad - \sum_{\ell=k+1}^{n+1} G_{\ell:n+1}  D_{\ell:n+1} (G_{k+1:\ell-1} - G_{k+1:\ell-2}).
\end{align*}
Calculating the difference in the brackets we obtain the statement of this lemma. 
\end{proof}

\begin{lemma}
\label[lemma]{lem:D1Ros}
Under assumptions of \Cref{th:approximation_error} for any $2 \le p \le \LL$ and $z \in \Zset$,
$$
\PE_z^{1/p}[\normop{D_{\ell:n+1}}^p] \le 4 d \Cros{p}^{1/p}  (\Const{A} + \normop{A}) V^{1/\LL}(z) \mathcal A_{\ell:n+1}^{1/2}.
$$
\end{lemma}
\begin{proof}
Proof follows from \Cref{propo:rosenthal_adapt_fort_moulines} with $f =\lyapW = V^{1/\LL},\lyapV = V^{p/\LL}$.
\end{proof}

\begin{lemma}
\label[lemma]{lem:S1est}
Under assumptions of \Cref{th:approximation_error} for any $2 \le p \le \LL$ and $z \in \Zset$,
$$
\PE_z^{1/p}[\normop{S_{k+1:n+1}}^p] \le \Const{\mathsf{S},p} \sum_{\ell=k+1}^{n+1} \alpha_\ell  \mathcal A_{\ell:n+1}^{1/2} \prod_{j = k+1}^{n+1} (1 - a \alpha_j)^{1/2} V^{1/\LL}(z),
$$
where
\begin{equation} \label{eq:cs_lemma11}
\Const{\mathsf{S},p}: =  24 \qcond d  \Cros{p}^{1/p}  (\Const{A} + \normop{A}) \normop{A}.
\end{equation}
\end{lemma}
\begin{proof}
Recall that
\begin{align*}
S_{k+1:n+1} = - \sum_{\ell=k+1}^{n+1} \alpha_\ell G_{\ell+1:n+1} A D_{\ell:n+1} G_{k+1:\ell-2} + \sum_{\ell=k+1}^{n+1} \alpha_{\ell-1} G_{\ell+1:n+1}  D_{\ell:n+1} A G_{k+1:\ell-2}.
\end{align*}
Applying Minkowski's inequality and Lemma~\ref{lem:D1Ros} we get
\begin{align*}
\PE_z^{1/p}[\normop{S_{k+1:n+1}}^p] \le \Const{\mathsf{S},p} \sum_{\ell=k+1}^{n+1} \alpha_\ell \prod_{j = \ell + 1}^{n+1} \sqrt{1 - a \alpha_j} \prod_{j = k+1}^{\ell-1} \sqrt{1 - a \alpha_j}\mathcal A_{\ell:n+1}^{1/2} V^{1/\LL}(z).
\end{align*}
\end{proof}

\begin{lemma}\label[lemma]{lem: vectRos}
Denote $\mathfrak F_k: = \sigma\{Z_s, s \geq k\}, k \geq 0$. Let $A_k$ be  a sequence of $d\times d$ random matrices  such that $A_k$ is $\mathfrak F_k$-measurable. Assume that  $Z_k^{*}$ is independent of $\mathfrak F_k$. Then
\begin{align*}
 \PE_z^\frac{1}{p}\big[\big\|\sum_{k=1}^n A_k \funnoise{Z_k^{*}}\big\|_2^p \big] \le    \Const{\mathsf{B},p}  \big(\sum_{k=1}^n \PE^{\frac{2}{p}}[\normop{A_k}^{p}] \big)^{1/2} V^{1/\LL}(z),
\end{align*}
where 
\begin{equation}
    \label{eq:Burkhconst}
    \Const{\mathsf{B},p}: = d^{3/2} \bigg\{\frac{2 \Constb_V \Const{\bar \varepsilon} }{\sqrt{1-\rho}} + 2 \bConst{\bar \varepsilon} (18\sqrt{2} p)
    \bigg\}  .
\end{equation}
\end{lemma}
\begin{proof}
We first reduce the problem to univariate one. Applying Minkowski's inequality we get
\begin{align*}
\PE_z^{1/p}\big[\big\|\sum_{k=1}^n A_k \funnoise{Z_k^{*}}\big\|_2^p \big] & =  \PE_z^{1/p}\big[\big|\sum_{\ell_1=1}^d\big \{\sum_{\ell_2=1}^d \sum_{k=1}^n [A_k]_{\ell_1 \ell_2} [\funnoise{Z_k^{*}}]_{\ell_2}\big \}^2 \big|^{p/2} \big] \\
& \le \bigg\{\sum_{\ell_1=1}^d \PE_z^{2/p}\big[ \big|\sum_{\ell_2=1}^d \sum_{k=1}^n [A_k]_{\ell_1 \ell_2} [\funnoise{Z_k^{*}}]_{\ell_2}\big|^{p} \big]  \bigg \}^{1/2} \\
& \le \bigg\{\sum_{\ell_1=1}^d \bigg\{\sum_{\ell_2=1}^d \PE_z^{1/p}\big[ \big| \sum_{k=1}^n [A_k]_{\ell_1 \ell_2} [\funnoise{Z_k^{*}}]_{\ell_2}\big|^{p} \big] \bigg\}^2 \bigg \}^{1/2} .
\end{align*}
Consider
$$
I: = \PE_z^{1/p}\big[ \big| \sum_{k=1}^n [A_k]_{\ell_1 \ell_2} [\funnoise{Z_k^{*}}]_{\ell_2}\big|^{p}  \big].
$$
We decompose it into two parts, $I \le I_1 + I_2$,
\begin{align*}
I_1&:= \PE_z^{\frac{1}{p}}\big[ \big| \sum_{k=1}^n [A_k]_{\ell_1 \ell_2} ([\funnoise{Z_k^{*}}]_{\ell_2} - \PE_z[[\funnoise{Z_k^{*}}]_{\ell_2}])\big|^{p}  \big], \\
I_2&:= \PE_z^{\frac{1}{p}}\big[ \big| \sum_{k=1}^n [A_k]_{\ell_1 \ell_2}  \PE_z[[\funnoise{Z_k^{*}}]_{\ell_2}]\big|^{p}  \big]
\end{align*}
The term $I_2$ may be estimated as follows
\begin{align*}
    |I_2| &\le 2 \Constb_V \Const{\bar \varepsilon}   \PE_z^{\frac{1}{p}}\big[ \big| \sum_{k=1}^n [A_k]_{\ell_1 \ell_2}  \rho^k \big|^{p}  \big] V^{1/\LL}(z) \\
    & \le \frac{2 \Constb_V \Const{\bar \varepsilon} }{\sqrt{1-\rho}} \PE_z^{\frac{1}{p}}\big[ \big| \sum_{k=1}^n [A_k]_{\ell_1 \ell_2}^2 \big|^{p/2}  \big] V^{1/\LL}(z)\\
    & \le \frac{2 \Constb_V \Const{\bar \varepsilon} }{\sqrt{1-\rho}} \bigg\{ \sum_{k=1}^n \PE_z^{\frac{2}{p}}[|[A_k]_{\ell_1 \ell_2}|^{p}] \bigg\}^{1/2} V^{1/\LL}(z)
\end{align*}

Applying Burkholder's inequality, see~\citep[Theorem 2.10]{hallheydebook}, Minkowski's inequality and \cref{lem: moments of vareps} we obtain
\begin{align*}
\PE^{\frac{1}{p}}\big[ \big| \sum_{k=1}^n [A_k]_{\ell_1 \ell_2} [\xi_k]_{\ell_2}\big|^{p} \big]  &\le (18\sqrt{2} p) \PE^{\frac{1}{p}}\bigg[\bigg\{\sum_{k=1}^n |[A_k]_{\ell_1 \ell_2}|^{2} ([\funnoise{Z_k^{*}}]_{\ell_2} - \PE_z[[\funnoise{Z_k^{*}}]_{\ell_2}])^2 \bigg\}^{p/2} \bigg ] \\
&\le 2 \bConst{\bar \varepsilon}  (18\sqrt{2} p) \bigg\{\sum_{k=1}^n \PE^{\frac{2}{p}}[|[A_k]_{\ell_1 \ell_2}|^{p}] \bigg\}^{1/2} V^{1/\LL}(z)
\end{align*}
Finally,
\begin{align*}
\PE_z^\frac{1}{p}\big[\big\|\sum_{k=1}^n A_k \funnoise{Z_k^{*}}\big\|_2^p \big] \le    \Const{\mathsf{B},p}  \big(\sum_{k=1}^n \PE^{\frac{2}{p}}[\normop{A_k}^{p}] \big)^{1/2} V^{1/\LL}(z).
\end{align*}
\end{proof}

\subsection{Proof of \Cref{th:approximation_expansion}}
Define the following constraint on the step size
\begin{equation}
    \label{eq:alpha_infty_1_def}
    \alpha_{\infty, p}^{(1)}: =\alpha_{\infty, 2p} \wedge \rho \wedge \rme^{-1} \wedge (2 \smallConst{\alpha})^{-1},
\end{equation}
where $\alpha_{\infty, 2p}$ and $\rho$ are defined in \eqref{eq:alpha_infty_main} and \eqref{eq:drift_conseq} respectively, and $\smallConst{\alpha}$ is from \Cref{assum:stepsize_2}. Let us re-state \Cref{th:approximation_expansion} as follows.
\begin{theorem}
\label{th:approximation_expansion_supp}
Let $\LL \geq 32$ and assume \Cref{assum:drift}, \Cref{assum:almost_bounded}, \Cref{assum:Hurwitzmatrices}, and \Cref{assum:funcb}. For any $2 \le p \le \LL/16$, any non-increasing sequence $(\alpha_k)_{k \in \nset}$ satisfying $\alpha_0 \in (0, \alpha_{\infty, p}^{(1)})$ and such that $\alpha_k \equiv \alpha$ or \Cref{assum:stepsize_2} holds,
$z \in \Zset$, $n \in \nset$, it holds
\begin{equation} \label{eq:hn0_tight_supp}
\textstyle{ \PE_z^{1/p}[ \| H_{n}^{(0)} \|^p ] \leq V^{3/\LL + 9/(16p)}(z)
    \begin{cases}
     \Const{p}^{(\mathsf{f})} \alpha \sqrt{\log(1/\alpha)} & \alpha_n \equiv \alpha, \\
      \Const{p}^{(\mathsf{d})} \sqrt{\alpha_{n} {\cal A}_n \log(1/\alpha_n)} & \text{under \Cref{assum:stepsize_2},}
    \end{cases}
}
\end{equation}
where the constants $\Const{p}^{(\mathsf{f})}, \Const{p}^{(\mathsf{d)}}$ are defined as
\begin{equation}
\label{eq:ConstH0fixed_improved}
\Const{p}^{(\mathsf{f})}: =   \Const{\mathsf{H},p}^{(\mathsf{f})} + \Const{\mathsf{J},p}^{(1,\mathsf{f})}, \quad 
\Const{p}^{(\mathsf{d})}: = \Const{\mathsf{H},p}^{(\mathsf{d})} + \Const{\mathsf{J},p}^{(1,\mathsf{d})}  .
\end{equation}
\end{theorem}


\begin{lemma}
\label[lemma]{lem: J1est}
Under conditions of  \Cref{th:approximation_expansion_supp}:
\begin{enumerate}[leftmargin=5mm,noitemsep]
\item If the step sizes are constant $\alpha_k \equiv \alpha$, then
\begin{equation}
\label{eqlem: J1 bound fixed}
\PE_z^{\frac{1}{p}}[\norm{J_{n}^{(1)}}^p] \le \Const{\mathsf{J,p}}^{(1,\mathsf{f})} \alpha\sqrt{\log(1/\alpha)} V^{\frac{2}{\LL} + \frac{1}{4p}}(z),
\end{equation}
where $\Const{\mathsf{J},p}^{(1, \mathsf{f})}$ is defined in~\eqref{eq: ConstJ1f}.
\item If the step sizes $\alpha_k, k \in \nset$, satisfy A\ref{assum:stepsize_2}, then
\begin{equation}
\label{eqlem: J1 bound decreasing}
\PE_z^{\frac{1}{p}}[\norm{J_{n}^{(1)}}^p] \le\Const{\mathsf{J},p}^{(1,\mathsf{d})} \sqrt{\alpha_{n} \mathcal A_{n}} \sqrt{\log(1/\alpha_n)} V^{\frac{2}{\LL} + \frac{1}{4p}}(z),
\end{equation}
where $\Const{\mathsf{J},p}^{(1,\mathsf{d})}$ is defined in~\eqref{eq: ConstJ1d}.
\end{enumerate}
\end{lemma}

\begin{lemma}
\label[lemma]{lem:H1est_notopt}
Under conditions of  \Cref{th:approximation_expansion_supp}:
\begin{enumerate}[noitemsep,leftmargin=5mm]
\item If the step sizes are constant $\alpha_k \equiv \alpha$, then
\begin{equation}
\label{eqlem: H1 bound fixed}
\PE_z^{\frac{1}{p}}[\norm{H_{n}^{(1)}}^p] \le \Const{\mathsf{H},p}^{(\mathsf{f})} \alpha \sqrt{\log(1/\alpha)} V^{\frac{3}{\LL} + \frac{9}{16p}}(z), 
\end{equation}
where $\Const{\mathsf{H},p}^{(\mathsf{f})}$ is defined in~\eqref{eq:ConstH1fixed}.
\item If the step sizes $\alpha_k, k \in \nset$, satisfy A\ref{assum:stepsize_2}, then
\begin{equation}
\label{eqlem: H1 bound decreasing}
\PE_z^{\frac{1}{p}}[\norm{H_{n}^{(1)}}^p] \le\Const{\mathsf{H},p}^{(\mathsf{d})}  \sqrt{\alpha_{n} \mathcal A_{n}} \sqrt{\log(1/\alpha_n)} V^{\frac{3}{\LL} + \frac{9}{16p}}(z),
\end{equation}
where $\Const{\mathsf{H},p}^{(\mathsf{d})}$ is defined in~\eqref{eq:ConstH1d}.
\end{enumerate}
\end{lemma}

\begin{proof}[Proof of Lemma~\ref{lem: J1est}]
For the second term $J_n^{(1)}$, solving the recursion in \eqref{eq:expansion_recur_gen} yields the double summation:
\[
J_{n+1}^{(1)}= -\sum_{k=1}^{n+1} \alpha_k G_{k+1:n+1} \funcAt{\State_{k}} J_{k-1}^{(0)} =  -\sum_{k=1}^{n+1} \alpha_k \sum_{j=1}^{k-1}\alpha_j G_{k+1:n+1} \funcAt{\State_{k}}   G_{j+1:k-1} \funnoise{\State_j}.
\]
Changing the order of summation gives
\begin{equation}
\label{eq:J1alternative}
J_{n+1}^{(1)}= -\sum_{j=1}^{n} \alpha_j \bigg\{\sum_{k=j+1}^{n+1} \alpha_k G_{k+1:n+1} \funcAt{\State_{k}}   G_{j+1:k-1} \bigg \}  \funnoise{\State_j} = \sum_{j=1}^{n} \alpha_j S_{j+1:n+1} \funnoise{\State_j},
\end{equation}
where for $j \le n$ we have defined
\[
S_{j:n}:= -\sum_{k=j}^{n} \alpha_k G_{k+1:n} \funcAt{\State_{k}}   G_{j:k-1}.
\]
Fix a constant $m \geq 1$ (to be determined later), we can further rewrite $S_{j+1:n+1}$ as
\begin{align*}
S_{j+1:n+1}&= -\sum_{k=j+1}^{j+m} \alpha_k G_{k+1:n+1} \funcAt{\State_{k}}   G_{j+1:k-1}    -  \sum_{k=j+m+1}^{n+1} \alpha_k G_{k+1:n+1} \funcAt{\State_{k}}   G_{j+1:k-1} \\
&=G_{j+m+1: n+1} S_{j+1:j+m} + S_{j+m+1:n+1} G_{j+1:j+m}.
\end{align*}
Let $N: = \round{n/m}$. In these notations, we can express $J_{n+1}^{(1)}$ as the sum of three terms:
\begin{align*}
  J_{n+1}^{(1)}&= \underbrace{\sum_{j=1}^{(m-1)N} \alpha_j G_{j+m+1: n+1} S_{j+1:j+m} \funnoise{\State_j}}_{=T_1} + \underbrace{\sum_{j=1}^{(m-1)N} \alpha_j S_{j+m+1:n+1} G_{j+1:j+m} \funnoise{\State_j}}_{=T_2} \\
  & \quad + \underbrace{\sum_{j=(m-1)N + 1}^{n} \alpha_j S_{j+1:n+1} \funnoise{\State_j}}_{=T_3}.
\end{align*}
Denote $\bConst{\bar\varepsilon}: =  \bConst{A} \norm{\theta^\star} + \bConst{b}$. By \Cref{lem: moments of vareps} for any $1 \le q \le \LL$,
\begin{equation} \PE_z^{\frac{1}{q}}[ \norm{\bar \varepsilon(\State_j)}^q]  \le \bConst{\bar\varepsilon} V^{1/\LL}(z).
\end{equation}
Let us consider the first term $T_1$. By the Minkowski inequality, \Cref{lem: moments of vareps} and \Cref{lem:S1est} (see the definition for $\Const{\sf S,p}$ in \eqref{eq:cs_lemma11})
\begin{equation}
\label{eq:T1}
\begin{split}
    \PE_z^{1/p}[\normop{T_1}^p] & \le  \sqrt{\qcond} \bConst{\bar \varepsilon}  \Const{\mathsf{S},p} \sum_{k=1}^{(m-1)N} \alpha_k \sum_{\ell=k+1}^{k+m} \alpha_\ell \mathcal A_{\ell:k+m}^{1/2} \prod_{j = k+1}^{n+1}  \sqrt{1 - a \alpha_j} V^{2/\LL}(z) \\
    & \le \sqrt{\qcond m} \bConst{\bar \varepsilon} \Const{\mathsf{S},p}   \sum_{k=1}^{(m-1)N} \alpha_k \sum_{\ell=k+1}^{k+m} \alpha_\ell^2  \prod_{j = k+1}^{n+1}  \sqrt{1 - a \alpha_j} V^{2/\LL}(z) \\
    &\le \sqrt{\qcond m} \bConst{\bar \varepsilon} \Const{\mathsf{S},p}  \sum_{\ell = 1}^{n+1} \alpha_\ell^2 \prod_{j = \ell+1}^{n+1}  \sqrt{1 - a \alpha_j} \sum_{k=1}^{\ell} \alpha_k \prod_{j = k+1}^{\ell}  \sqrt{1 - a \alpha_j} V^{2/\LL}(z)\\
    & \le \Const{1} \sqrt m \alpha_{n+1} V^{2/\LL}(z),
    \end{split}
\end{equation}
where we have defined
$$
\Const{1}: = 8 a^{-2} \bConst{\bar \varepsilon} \Const{\mathsf{S},p}  \sqrt{\qcond}.
$$
Similar bound holds for $T_3$,
\begin{equation}
\label{eq:T3}
\PE_z^{\frac{1}{p}}[\normop{T_3}^p] \le  \Const{1} \sqrt{m} \alpha_{n+1} V^{2/\LL}(z).
\end{equation}
The second term $T_2$ may be rewritten as $T_{21} + T_{22}$, where
\begin{align*}
    T_{21}&:= \sum_{k=0}^{N-1} \sum_{i=1}^{m} \alpha_{km + i} S_{(k+1)m + i + 1:n+1}^{(1)} G_{km+i+1:(k+1)m+i} \funnoise{\State_{km+i}^{*}}, \\
    T_{22}&:=\sum_{k=0}^{N-1} \sum_{i=1}^{m} \alpha_{km + i} S_{(k+1)m + i + 1:n+1}^{(1)} G_{km+i+1:(k+1)m+i} (\funnoise{\State_{km+i}} - \funnoise{\State_{km+i}^{*}}).
\end{align*}
In the above, the set of r.v. $\State_{km+i}^{*}$ is constructed
for each $i \in [1, m]$, with
$\{\State_{km+i}^{*}\}_{k=0}^{N-1}$ and the following properties
\begin{equation} \label{eq:zstar}
    \begin{split}
    & \text{1.~~$\State_{km+i}^{*}$ is independent of $\mathfrak F_{(k+1)m + i}^{n+1}: = \sigma\{Z_{(k+1)m+i}, \ldots, Z_{n+1}\}$}; \\
    & \text{2.~~}\P_z(\State_{km+i}^{*} \neq \State_{km+i}) \le 2 \Constb_V \rho^m  V(z); \\
    & \text{3.~~$\State_{km+i}^{*}$ and $\State_{km+i}$ have the same distribution},
    \end{split}
\end{equation}
where $\Constb_V, \rho$ are defined in~\eqref{eq:drift_conseq}.
The existence of the r.v.s $Z_{km+i}^*$ is guaranteed by Berbee's lemma, see e.g~\citep[Lemma 5.1]{riobook}.
We also exploit the fact the $V$-uniformly ergodic Markov chains are a special instance of $\beta$-mixing processes. We control $\beta$-mixing coefficient via total variation distance; see \citep[Theorem F.3.3]{douc:moulines:priouret:2018}.

To analyze $T_{21}$ we use   \Cref{lem: vectRos}
\begin{align*}
\PE_z^{1/p}[\norm{T_{21}}^p]  &\le \sum_{i=1}^{m}  \PE_z^{\frac{1}{p}}\bigg[ \bigg \|\sum_{k=0}^{N-1} \alpha_{km + i} S_{(k+1)m + i + 1:n+1} G_{km+i+1:(k+1)m+i} \funnoise{\State_{km+i}^{*}} \bigg\|^p  \bigg] \\
&\le \Const{\mathsf{B},p}  \sqrt{\qcond} \sum_{i=1}^{m} \bigg(\sum_{k=0}^{N-1} \alpha_{km + i }^2 \PE^{\frac{2}{p}}[\|S_{(k+1)m + i + 1:n+1} \|^{p}] \prod_{\ell=km+i+1}^{(k+1)m+i} (1 - a \alpha_\ell) \bigg )^{1/2} V^{1/\LL}(z)\\
&\le \Const{\mathsf{B},p}  \sqrt{\qcond m} \bigg(  \sum_{k=1}^{n+1} \alpha_{k}^2 \PE^{\frac{2}{p}}[\|S_{k+m + 1:n+1} \|^{p}] \prod_{\ell=k+1}^{k+m} (1 - a \alpha_\ell) \bigg )^{1/2} V^{1/\LL}(z),
\end{align*}
where $\Const{\mathsf{B}}$ is defined in~\eqref{eq:Burkhconst}.
Applying Lemma~\ref{lem:S1est} we may estimate the term in the brackets by
\begin{align*}
&\frac{4(\Const{\mathsf{S},p})^2}{a^2}  \sum_{k=1}^{n+1} \alpha_{k}^2 \sum_{\ell=k+1}^{n+1} \alpha_\ell \mathcal A_{\ell:n+1} \prod_{j=\ell+1}^{n+1} \sqrt{1 - a \alpha_j}   \prod_{j=k+1}^{\ell} (1 - a \alpha_j) V^{2/\LL}(z)\\
&\qquad\qquad \le \frac{4 (\Const{\mathsf{S},p})^2}{a^2}  \sum_{\ell=1}^{n+1} \alpha_\ell \mathcal A_{\ell:n+1}  \prod_{j=\ell+1}^{n+1} \sqrt{1 - a \alpha_j} \sum_{k=1}^{\ell} \alpha_k^2 \prod_{j=\ell+1}^{n+1} (1 - a \alpha_j) V^{2/\LL}(z) \\
&\qquad\qquad \le\frac{16 (\Const{\mathsf{S},p})^2}{a^2}    \sum_{\ell=1}^{n+1} \alpha_\ell^2 \mathcal A_{\ell:n+1}  \prod_{j=\ell+1}^{n+1} \sqrt{1 - a \alpha_\ell}  V^{2/\LL}(z).
\end{align*}
Finally
\begin{equation}
\label{eq:T21}
\PE_z^{1/p}[\norm{T_{21}}^p] \le \Const{2} \sqrt{m} \bigg\{  \sum_{k=1}^{n+1} \alpha_{k}^2 {\cal A}_{k:n+1} \prod_{\ell=k+1}^{\ell} \sqrt{1 - a \alpha_\ell} \bigg \}^{1/2} V^{2/\LL}(z),
\end{equation}
where
$$
\Const{2}: = 4 a^{-1} \Const{\mathsf{B},p}  \Const{\mathsf{S},p} \sqrt{\qcond}.
$$
For the term $T_{22}$ we use Minkowski's inequality
\begin{align*}
\PE_z^{1/p}[\norm{T_{22}}^p] &\le \sqrt{\qcond}\sum_{k=0}^{N-1} \sum_{i=1}^{m} \alpha_{km + i} \PE^{\frac{1}{2p}}[\normop{S_{(k+1)m + i + 1:n+1}^{(1)}}^{2p}] \\
&\qquad\qquad\qquad \times \prod_{\ell = km+i+1}^{(k+1)m+i} \sqrt{1-a \alpha_\ell} \PE_z^{\frac{1}{2p}}[\norm{\funnoise{\State_{km+i}} - \funnoise{\State_{km+i}^{*}})}^{2p}].
\end{align*}
Using definition of $\State_{km+i}^*$ and
and the Cauchy-Schwartz inequality
\begin{equation}
\label{eq: zminuszstar}
\begin{split}
&\PE_z^{1/(2p)}[\norm{\funnoise{\State_{km+i}} - \funnoise{\State_{km+i}^{*}})}^{2p}] \\
&\qquad\qquad =  \PE_z^{1/(2p)}[\norm{\funnoise{\State_{km+i}} - \funnoise{\State_{km+i}^{*}}) \mathbbm{1} \{\State_{km+i}^{*} \neq \State_{km+i}  \}}^{2p}] \\
&\qquad\qquad\le 2 \PE_z^{1/(4p)}[\norm{\funnoise{\State_{km+i}}}^{4p} \P_z^{1/(4p)}(\State_{km+i}^{*} \neq \State_{km+i}) \\
&\qquad\qquad\le 4\bConst{\bar \varepsilon} \Constb_V^{1/(4p)} \rho^{m/(4p)}  V^{1/(4p) + 1/\LL}(z),
\end{split}
\end{equation}
where we used \eqref{eq:zstar}.
The last two inequalities, \Cref{lem:S1est} and \Cref{lem:bsum} imply
\begin{equation}
\label{eq:T22}
\begin{split}
\PE_z^{1/p}[\norm{T_{22}}^p] &\le 4\bConst{\bar \varepsilon} \Constb_V^{1/(4p)} \bar\rho^{m} \sum_{k=1}^{n+1} \alpha_{k} \sum_{\ell=k+1}^{n+1} \alpha_\ell \mathcal A_{\ell:n+1}^{1/2} \prod_{\ell=k+1}^{n+1} \sqrt{1 -  a \alpha_\ell} V^{1/(4p) + 2/\LL}(z)  \\
& \le \Const{3} \bar \rho^m \sum_{\ell=1}^{n+1} \alpha_\ell \mathcal A_{\ell:n+1}^{1/2} \prod_{j=\ell+1}^{n+1} \sqrt{1 -  a \alpha_j} V^{1/(4p) + 2/\LL}(z),
\end{split}
\end{equation}
where
$$
\Const{3}: = 8 a^{-1} \qcond^{1/2}  \Const{\mathsf{S},p}
\bConst{\bar \varepsilon} \Constb_V^{1/(4p)}, \quad
\bar \rho : = \rho^{1/(4p)}.
$$
Bounds~\eqref{eq:T1}, \eqref{eq:T3}, \eqref{eq:T21}, \eqref{eq:T22} together imply
\begin{align*}
\PE_z^{1/p}[\norm{J_{n+1}^{(1)}}^p] &\le 3 \Const{1} \sqrt{m} \alpha_{n+1} V^{2/\LL}(z)  + \Const{2} \sqrt{m} \bigg \{ \sum_{k=0}^{n+1} \alpha_k^2 \mathcal A_{k:n+1} \prod_{\ell=k+1}^{n+1} \sqrt{1 -  a \alpha_\ell } \bigg\}^{1/2} V^{2/\LL}(z)  \\
&+ \Const{3} \bar \rho^{m}  \sum_{k=1}^{n+1} \alpha_{k} \mathcal A_{k:n+1}^{1/2} \prod_{\ell=k+1}^{n+1} \sqrt{1 -  a \alpha_\ell} V^{1/(4p) + 2/\LL}(z).
\end{align*}
We distinguish two cases:
\begin{enumerate}[leftmargin=5mm]
\item $\alpha_k \equiv \alpha$ for any $k \in \nset$. Then
\begin{align*}
\PE_z^{\frac{1}{p}}[\norm{J_{n+1}^{(1)}}^p] &\le  2 \Const{1} \sqrt m \alpha V^{2/\LL}(z) +  \Const{2} \sqrt{m} \alpha^2 \bigg \{ \sum_{k=0}^{n+1}  (n-k+2)  (1 -  a \alpha )^{(n-k+1)/2} \bigg\}^{1/2} V^{2/\LL}(z) \\
&+ \Const{3} \alpha^2 \bar \rho^{m}  \sum_{k=1}^{n+1} \sqrt{n-k+2} (1 -  a \alpha)^{(n-k+1)/2} V^{1/(4p) + 2/\LL}(z) \\
& \le \Const{4} \sqrt{m} \alpha V^{1/(4p) + 2/\LL}(z),
\end{align*}
where
$$
\Const{4}:= 2 \Const{1}+ 2 \sqrt{\rme}\Const{2}/a + \sqrt{2\pi} \rme \Const{3} / a^{3/2}
$$
and we took $m$ such that
$$
\bar \rho^m \le \sqrt{\alpha}, \text{ i.e. } m = \bigg \lceil \frac{1}{2} \frac{\log(1/\alpha)}{\log (1/\bar \rho)} \bigg \rceil.
$$
We obtain
\begin{equation}
\label{eq: J1 bound fixed}
\PE_z^{\frac{1}{p}}[\norm{J_{n+1}^{(1)}}^p] \le \Const{\mathsf{J},p}^{(1,\mathsf{f})} \alpha \log^{1/2} (1/\alpha) V^{1/(4p) + 2/\LL}(z),
\end{equation}
where
\begin{equation}
\label{eq: ConstJ1f}
    \Const{\mathsf{J},p}^{(1,\mathsf{f})}:= 2 \sqrt{p} \Const{4} \log^{-1/2} (1/ \rho).
\end{equation}
\item Assume that A\ref{assum:stepsize_2} is satisfied. Then we apply \Cref{lem:bsum2} and obtain
\begin{align*}
\PE_z^{\frac{1}{p}}[\norm{J_{n+1}^{(1)}}^p] &\le \Const{5} \sqrt{m} \sqrt{\alpha_{n+1} \mathcal A_{n+1}} V^{1/(4p) + 2/\LL}(z),
\end{align*}
where
$$
\Const{5}:= (3 \Const{1} + 2 \Const{2}/\sqrt{a} + 4\Const{3}/a) (\sqrt{\smallConst{\alpha}}+1)
$$
and
$$
m = \bigg \lceil \frac{1}{2} \frac{\log(1/\alpha_{n+1})}{\log (1/\bar \rho)} \bigg \rceil.
$$
\end{enumerate}
In both cases, we have
\begin{equation}
\label{eq: J1 bound decreasing}
\PE_z^{\frac{1}{p}}[\norm{J_{n+1}^{(1)}}^p] \le\Const{\mathsf{J},p}^{(1,\mathsf{d})} \sqrt{\alpha_{n+1} \mathcal A_{n+1}} \sqrt{\log(1/\alpha_{n+1})} V^{1/(4p) + 2/\LL}(z),
\end{equation}
where
\begin{equation}
    \label{eq: ConstJ1d}
    \Const{\mathsf{J},p}^{(1,\mathsf{f})}:= 2\sqrt{p} \Const{5} \log^{-1} (1/ \rho).
\end{equation}
\end{proof}

\begin{proof}[Proof of Lemma~\ref{lem:H1est_notopt}]
To estimate $H_n^{(1)}$ we rewrite it as follows
\[
H_{n+1}^{(1)} = -\sum_{\ell=1}^{n+1} \alpha_\ell \ProdB_{\ell+1:n+1} \funcAt{\State_{\ell}} J_{\ell-1}^{(1)}.
\]
Using Minkowski's and Cauchy-Schwarz inequality, 
$$
\PE_z^{1/p}[\norm{H_{n+1}^{(1)}}^p] \le \sum_{\ell = 1}^{n+1} \alpha_\ell \PE_z^{1/(2p)}[\normop{\ProdB_{\ell+1:n+1}}^{2p}] \PE_z^{1/(4p)}[\normop{\funcAt{\State_{\ell}}}^{4p}] \PE_z^{1/4p}[\norm{J_{\ell-1}^{(1)}}^{4p}].
$$
We apply \Cref{th:expconvproducts} to estimate $\PE_z^{1/(2p)}[\normop{\ProdB_{\ell+1:n+1}}^{2p}]$ and \Cref{lem: moments of vareps} to estimate $\PE_z^{1/(4p)}[\normop{\funcAt{\State_{\ell}}}^{4p}]$. These bounds lead
$$
\PE_z^{1/p}[\norm{H_{n+1}^{(1)}}^p] \le \bConst{A} \Const{\mathsf{st},2p} \sum_{\ell = 1}^{n+1} \alpha_\ell \rme^{ - (a/4) \sum_{k=\ell+1}^{n+1}   \alpha_{k}} \PE_z^{1/(4p)}[\norm{J_{k-1}^{(1)}}^{4p}]  V^{1/\LL+1/(4p)}(z).
$$
We again consider two cases:
\begin{enumerate}[noitemsep,leftmargin=5mm]
\item $\alpha_k \equiv \alpha$ for any $k \in \nset$. Then applying~\eqref{eq: J1 bound fixed} we get
$$
\PE_z^{\frac{1}{p}}[\norm{H_{n+1}^{(1)}}^p] \le \Const{\mathsf{J},p}^{(1,\mathsf{f})}  \bConst{A} \Const{\mathsf{st},2p} \alpha^2 \log^{1/2}(1/\alpha) \sum_{k = 1}^{n+1}  \rme^{ - \alpha a (n-k+1)/4}  V^{3/\LL + 9/(16p)}(z).
$$
This expression may be simplified. We come to the inequality
\begin{equation}
\label{eq: H1 bound fixed}
\PE_z^{\frac{1}{p}}[\norm{H_{n+1}^{(1)}}^p] \le \Const{\mathsf{H},p}^{(\mathsf{f})}    \alpha \sqrt{\log(1/\alpha)} V^{3/\LL + 9/(16p)}(z),
\end{equation}
where
\begin{equation}
\label{eq:ConstH1fixed}
\Const{\mathsf{H},p}^{(\mathsf{f})} : = 8 \Const{\mathsf{J}}^{(1,\mathsf{f})}   \bConst{A} \Const{\mathsf{st},2p} / a .  
\end{equation}
\item Assume A\ref{assum:stepsize_2}, then we use~\eqref{eq: J1 bound decreasing} and inequality $\rme^{-x} \le 1 - x/2$ valid for $0 \le x \le 1$,
\begin{align*}
\PE_z^{\frac{1}{p}}[\norm{H_{n+1}^{(1)}}^p] &\le \Const{\mathsf{J},p}^{(1,\mathsf{d})}  \bConst{A} \Const{\mathsf{st},2p} (1 + \smallConst{\alpha} \alpha_{\infty,p}^{(2)}) \sqrt{\log(1/\alpha_{n+1})} \\
&\qquad\qquad\times \sum_{k = 1}^{n+1} \alpha_k \rme^{ - (a/4) \sum_{\ell=k+1}^{n+1}   \alpha_{\ell}} \sqrt{\alpha_{k} \mathcal A_{k}} V^{3/\LL + 9/(16p)}(z) \\
&\le \Const{\mathsf{J},p}^{(1,\mathsf{d})}  \bConst{A} \Const{\mathsf{st},2p} (1 + \smallConst{\alpha} \alpha_{\infty,p}^{(2)}) \sqrt{\log(1/\alpha_{n+1})} \\
& \qquad\qquad\times \sum_{k = 1}^{n+1} \alpha_k  \sqrt{\alpha_{k} \mathcal A_{k}}  \prod_{\ell=k+1}^n(1 - (a/8)   \alpha_{\ell})  V^{3/\LL + 9/(16p)}(z).
\end{align*}
Applying \Cref{lem:bsum2} we get
\begin{equation}
 \PE_z^{\frac{1}{p}}[\norm{H_{n+1}^{(1)}}^p] \le \Const{\mathsf{H},p}^{(\mathsf{d})}  \sqrt{\alpha_{n+1} \mathcal A_{n+1}}  \sqrt{\log(1/\alpha_{n+1})}  V^{3/\LL + 9/(16p)}(z),
\end{equation}
where
\begin{equation}
\label{eq:ConstH1d}
\Const{\mathsf{H},p}^{(\mathsf{d})}: = 16 \Const{\mathsf{J}}^{(1,\mathsf{d})}  \bConst{A} \Const{\mathsf{st},2p} (1 + \smallConst{\alpha} \alpha_{\infty,p}^{(2)})/a .
\end{equation}
\end{enumerate}
\end{proof}

\section{Temporal-Difference Learning}
We preface the proof by a a well-known elementary sufficient condition for a matrix $-A$ to be Hurwitz. We give the proof for completeness.
\begin{lemma}
\label[lemma]{lem:useful-algebraic-0}
Let $A$ be a $d \times d$ matrix. Assume that for all $x \in \rset^d$, $x^\top A x > 0$, then for any $\ell \in \{1, \dots, d\}$,  $\operatorname{Re}\lambda_\ell(A)  > 0$,
where $\lambda_\ell(A)$, $\ell \in \{ 1 , \ldots, d\}$ are the eigenvalues of $A$.
\end{lemma}
\begin{proof}
Fix $\ell = 1, \ldots, d$ and let $\lambda = \lambda_\ell(A) = \mu + \rmi \nu$ and $z = x + \rmi y$ be the eigenvector of $A$ corresponding to $\lambda$. Then
\[
       (A - \lambda \Id) (x + \rmi y) = (A - \mu \Id) x + \nu y + \rmi(-\nu x + (A - \mu \Id) y) = 0.
\]
This implies that
    \begin{equation*}
        \begin{cases}
            x^\top (A - \mu \Id) x = - \nu x^\top y \\
            y^\top (A - \mu \Id) y = \nu y^\top x .
        \end{cases}
    \end{equation*}
    Taking the sum of these equations we get $x^\top (A - \mu \Id) x + y^\top (A - \mu \Id) y = 0$, or
    $$
    \mu = \frac{x^\top A x + y^\top A y}{x^\top x + y^\top y} > 0.
    $$
\end{proof}
Recall that,
\begin{equation}
\label{eq:definition-A}
A = \sum_{\ell=0}^{\tau-1}\PE_{\pi_0}[\psi(X_{\tau-1-\ell}) \{ \psi(X_{\tau-1}) - \gamma \psi(X_{\tau}) \}^{\top}] \eqsp,
\end{equation}
for $\tau \in \nset^*$,
\begin{lemma}
\label[lemma]{lem: hurwitz TD}
Assume \Cref{assum:markov2}. Then for any $\ell = 1 , \ldots, d$
\begin{equation}
    \label{eq: hurwitz cond}
       \operatorname{Re}\lambda_\ell(A)  > 0.
    \end{equation}
\end{lemma}
\begin{proof}
We show that $x^\top A x >0$ for any $x \in \rset^d$ and then apply \Cref{lem:useful-algebraic-0}. Fix $x \in \rset^d$ and denote
$$
\rho(\ell) = \PE_{\pi_0}[x^\top \psi(X_0) \psi^\top(X_\ell) x] \{\PE_{\pi_0}[x^\top \psi(X_0) \psi(X_0)^\top x] \}^{-1}.
$$
Then
\begin{align*}
    x^\top A x = \PE_{\pi_0}[x^\top \psi(X_0) \psi(X_0)^\top x] \biggl \{ \sum_{\ell=0}^{\tau-1}  (\lambda \gamma)^{\ell} (\rho(\ell) - \gamma \rho(\ell+1)) \biggr\}.
\end{align*}
The sum in the brackets could be rewritten as
\begin{align*}
  \sum_{\ell=0}^{\tau-1} (\lambda \gamma)^{\ell} (\rho(\ell) - \gamma \rho(\ell+1)) = 1 - \gamma \bigl\{ (1 - \lambda) \sum_{\ell=1}^{\tau-1} (\lambda \gamma)^{\ell-1} \rho(\ell) + (\lambda \gamma)^{\tau-1} \rho(\tau) \bigr\}.
\end{align*}
Since by the Cauchy-Schwartz inequality $|\rho(\ell)| \le 1$, we obtain
\begin{align*}
  \sum_{\ell=0}^{\tau-1} (\lambda \gamma)^{\ell}  (\rho(\ell) - \gamma \rho(\ell+1)) &\geq 1 - \gamma \bigl\{ (1 - \lambda) \sum_{\ell=1}^{\tau-1} (\lambda \gamma)^{\ell-1}+ (\lambda \gamma)^{\tau-1}  \bigr\} \\
  & = 1 - \gamma \bigl\{ (1 - \lambda) (1 - (\lambda \gamma)^{\tau-1}) (1 -\lambda \gamma)^{-1}  + (\lambda \gamma)^{\tau-1}  \bigr\} \\
  & = \frac{1-\gamma}{1- \lambda \gamma} \{ 1 - (\lambda \gamma)^{\tau} \}.
\end{align*}
Finally,
$$
    x^\top A x \geq \frac{1-\gamma}{1- \lambda \gamma} \{ 1 - (\lambda \gamma)^{\tau} \} \PE_{\pi_0}[x^\top \psi(X_0) \psi(X_0)^\top x] > 0,
    $$
    where we applied \Cref{assum:markov2}.
 \end{proof}

\begin{lemma}
\label[lemma]{lem: tdlambda irredu}
Assume \Cref{assum:markovQ}. Then the Markov kernel $\MK$ defined in \eqref{eq: trans kernel for TD}, is irreducible and aperiodic.
\end{lemma}
\begin{proof} Recall that the Markov kernel $\MK$ is irreducible if it admits an accessible small set. We are going to construct such set.

Since the Markov kernel $\MKQ$ is strongly aperiodic, it admits an accessible $(1,\varepsilon \nu)$-small set $C$ with $\nu(C) > 0$ (see \cite[Definition~9.3.5]{douc:moulines:priouret:2018}). Let us take $\tilde{C} = \Xset \times \dots \times \Xset \times C$ and check that it is accessible and small for $\MK$. Note that, for $k \geq \tau$,
\begin{align*}
\PP_{(x_{-\tau},\dots,x_0)}(\State_{k} \in \tilde{C}) = \PP_{x_0}(X_{k-\tau} \in C) = \MKQ^{k-\tau}(x_0,C)\eqsp.
\end{align*}
Since $C$ is accessible for $\MKQ$, for any $x_0 \in \Xset$ we can choose $k$, such that $\MKQ^{k-\tau}(x_0,C) > 0$, showing that $\tilde{C}$ is accessible for $\MK$. To check that $\tilde{C}$ is small, note that for any $D_1,\dots,D_{\tau+1} \in \mathcal{X}$, and $(x_{-\tau},\dots,x_0) \in \tilde{C}$,
\begin{multline*}
\MK^{\tau+1}((x_{-\tau},\dots,x_0),D_1 \times \dots \times D_{\tau+1}) = \int \prod_{k=0}^{\tau} \MKQ(x_k,\rmd x_{k+1}) \prod_{k=1}^{\tau+1}\indi{D_k}(x_k) \\
\geq \int \prod_{k=0}^{\tau} \MKQ(x_k,\rmd x_{k+1}) \prod_{k=1}^{\tau+1}\indi{D_k \cap C}(x_k)
\overset{(a)}{\geq} \varepsilon^{\tau+1} \prod_{k=1}^{\tau+1}\nu(D_k \cap C)
\geq \bigl(\varepsilon\nu(C)\bigr)^{\tau+1}\nu_{C}(D_1 \times \dots \times D_{\tau+1}) \eqsp,
\end{multline*}
where (a) follows from $(\tau + 1)$ applications of the fact that $C$ is $(1,\varepsilon \nu)$ small for $\MKQ$ and
\begin{equation}
\label{eq:gamma_C_definition}
\nu_{C}(D_1 \times \dots \times D_{\tau+1}) = \prod_{k=1}^{\tau+1}\nu(D_k \cap C) / \nu(C) \eqsp. \
\end{equation}
Hence, $\tilde{C}$ is $\bigl(\tau+1, (\varepsilon\nu(C))^{\tau+1}\nu_C\bigr)$-small and accessible. This implies that the Markov kernel $\MK$ is irreducible. To check that $\MK$ is aperiodic, we first note that, due to \cite[Lemma~9.3.3]{douc:moulines:priouret:2018}), there exists such $n_0 \in \nset$, that for any $k \geq n_0$, set $C$ is $(k,\varepsilon_{k}\nu)$-small for $\MKQ$ with $\varepsilon_{k} > 0$. Hence, for any $k \geq n_0 + \tau$, 
\begin{equation*}
\inf_{x_{-\tau:0} \in \tilde{C}}\MK^{k}(x_{-\tau:0}, \tilde{C}) = \inf_{x_{0} \in C}\PP_{x_0}(X_{k-\tau} \in C) \geq \varepsilon_{k - \tau}\nu(C) > 0\eqsp, 
\end{equation*}
yielding that the Markov kernel $\MK$ is aperiodic.
\end{proof}

\begin{lemma}
\label[lemma]{lem: tdlambda drift}
Assume \Cref{assum:markovQ}. For $\tau \in \nset^*$ let us set
\begin{equation}
\label{eq:V_func_def_td_lambda}
V(x_{0:\tau}) = \exp \left(c_0 \sum_{i=0}^{\tau-1}(i+1)\tilde{W}^{\delta}(x_i) + \tilde{W}(x_\tau) \right)\eqsp,
\end{equation}
and $W(x_{0:\tau}) := \log{V(x_{0:\tau})}$, where $c_0$ is defined in \eqref{eq:const_c_0_def_td_lambda}. Then the Markov  kernel $\MK$ (see \eqref{eq: trans kernel for TD}) satisfies the drift condition
\begin{equation}
\MK V(x_{0:\tau}) \leq \rme^{-c_{\MK}W^{\delta}(x_{0:\tau})} V(x_{0:\tau})\1_{\{W(x_{0:\tau}) \geq R_{\MK}\}} + \bb_{\MK}\1_{\{W(x_{0:\tau}) < R_{\MK}\}}\eqsp,
\end{equation}
with the constants $c_{\MK}, \bb_{\MK}$, and $R_{\MK}$ defined in \eqref{eq:b_R_prime_prime_def_td_lambda}. Moreover, for any $R \geq 1$ the sublevel sets $\{x_{0:\tau}: W(x_{0:\tau}) \leq R\}$ are $\bigl(\tau+1, (\varepsilon_R\nu(C_R))^{\tau+1}\nu_{C_R}\bigr)$-small, where the measure $\nu_{C_R}$ is defined in \eqref{eq:gamma_C_R_definition}.
\end{lemma}
\begin{proof}
Let us introduce the function $V_{\beta}(x_0,\dots,x_\tau) = \rme^{\beta c\sum_{i=0}^{\tau-1}(i+1)\tilde{W}^{\delta}(x_i) + \tilde{W}(x_\tau)}$ where $c$ is defined in \eqref{eq:drift-condition-improv_Q} and  $\beta \in \ooint{0,1/\tau}$ is a parameter to be chosen later. Then
\begin{align*}
\MK V_{\beta}(x_{0:\tau}) &= \idotsint \rme^{\beta c \sum_{i=0}^{\tau-1}(i+1)\tilde{W}^{\delta}(x_i^{\prime})} \rme^{\tilde W(x_\tau^{\prime})} \biggl\{\prod_{i=0}^{\tau-1}\delta_{x_{i+1}}(\rmd x_i^{\prime})\biggr\} \MKQ(x_\tau,\rmd x_\tau^{\prime}) \\
&\overset{(a)}{\leq}
\idotsint \rme^{\beta c \sum_{i=0}^{\tau-1}(i+1)\tilde{W}^{\delta}(x_i^{\prime})} \biggl(\rme^{-c\tilde W^{\delta}(x_\tau)}\tilde V(x_\tau) + \bb \biggr)\biggl\{\prod_{i=0}^{\tau-1}\delta_{x_{i+1}}(\rmd x_i^{\prime})\biggr\} \\
&=
\rme^{\beta c \sum_{i=1}^{\tau-1}i \tilde{W}^{\delta}(x_i)}\rme^{-(1-\beta\tau)c \tilde W^{\delta}(x_\tau)}\tilde{V}(x_\tau) + \bb \rme^{\beta c \sum_{i=1}^{\tau}i\tilde{W}^{\delta}(x_i)} \\
&=
\rme^{-\beta c \sum_{i=0}^{\tau-1}\tilde{W}^{\delta}(x_i)-(1-\beta\tau)c \tilde W^{\delta}(x_\tau)}V_{\beta}(x_{0:\tau}) + \bb \rme^{\beta c \sum_{i=1}^{\tau}i\tilde{W}^{\delta}(x_i)}\eqsp,\end{align*}
where (a) follows from \Cref{assum:drift}. The next step is to show how to select $\beta$ and $\tilde{c}$ in order to ensure that
\[
\rme^{-\beta c \sum_{i=0}^{\tau-1}\tilde{W}^{\delta}(x_i)-(1-\beta\tau)c \tilde W^{\delta}(x_\tau)} \leq \rme^{-\tilde{c}W^{\delta}_{\beta}(x_{0:\tau})}
\]
where $ W_{\beta}(x_{0:\tau}) = \log V_{\beta}(x_{0:\tau})$. For this purpose, we first notice that, we have
\begin{align*}
W^{\delta}_{\beta}(x_{0:\tau}) &\leq \sum_{i=0}^{\tau-1}((i+1)\beta c)^{\delta}\tilde W^{\delta}(x_i) + \tilde W^{\delta}(x_\tau) \\
&=
\frac{\sum_{i=0}^{\tau-1}(1-\beta\tau)c((i+1)\beta c)^{\delta}\tilde W^{\delta}(x_i) + (1-\beta\tau)c \tilde W^{\delta}(x_\tau)}{(1-\beta\tau)c}\,.
\end{align*}
Let us select $\beta = \beta_0$, where
\begin{equation}
\label{eq:beta_equation_td_lambda}
\beta_0 = \inf\{\beta \in (1/(2\tau),1/\tau) \Big| (1-\tau\beta)(\tau \beta c)^{\delta} \leq \beta \}\,.
\end{equation}
Then, setting $\tilde{c} = (1-\tau \beta_0)c$, we get
\begin{equation*}
\tilde{c}W^{\delta}_{\beta_0}(x_{0:\tau}) \leq \beta_0 c \sum_{i=0}^{\tau-1}\tilde W^{\delta}(x_i) + (1-\tau \beta_0)c\tilde W^{\delta}(x_\tau) \eqsp.
\end{equation*}
Define now
\begin{equation}
\label{eq:const_c_0_def_td_lambda}
c_0 = \beta_0 c\eqsp,
\end{equation}
and put $V(x_{0:\tau}) = V_{\beta_0}(x_{0:\tau}),\,W(x_{0:\tau}) = W_{\beta_0}(x_{0:\tau})$ . Then
\begin{equation}
\label{eq:preliminary_drift_td_lambda}
\begin{split}
\MK V(x_{0:\tau})
&\leq
\rme^{-\tilde{c}W^{\delta}(x_{0:\tau})} V(x_{0:\tau}) + \bb \rme^{c_0 \sum_{i=1}^{\tau}i\tilde{W}^{\delta}(x_i)} \\
&\leq
\rme^{-\tilde{c}W^{\delta}(x_{0:\tau})} V(x_{0:\tau}) + \bb \rme^{c_0\tau W^{\delta}(x_{0:\tau})}\,.
\end{split}
\end{equation}
Let us fix $R_1 = \inf\bigl\{r > 0 \big\vert r - (\tilde{c} + \tau c_0)r^{\delta} - \ln{\bb} > 0)\bigr\}$.
Note that, for $(x_{0:\tau}) \in \{W(x_{0:\tau}) > R_1\}$,
\begin{align*}
\bb \rme^{c_0 W^{\delta}(x_{0:\tau})} = \rme^{\ln{\bb} + c_0\tau W^{\delta}(x_{0:\tau})} \leq \rme^{-\tilde{c}W^{\delta}(x_{0:\tau})}V(x_{0:\tau})\eqsp.
\end{align*}
Hence,
{\small \begin{align*}
\biggl(\rme^{-\tilde{c}W^{\delta}(x_{0:\tau})} V(x_{0:\tau}) + \bb \rme^{c_0\tau W^{\delta}(x_{0:\tau})}\biggr)\1_{\{W(x_{0:\tau}) \geq R_1 \vee R_2\}} &\leq 2\rme^{-\tilde{c}W^{\delta}(x_{0:\tau})} V(x_{0:\tau})\1_{\{W(x_{0:\tau}) \geq R_1 \vee R_2\}} \\
&\leq \rme^{-(\tilde{c}/2)W^{\delta}(x_{0:\tau})} V(x_{0:\tau})\1_{\{W(x_{0:\tau}) \geq R_1 \vee R_2\}}\eqsp,
\end{align*}
}
where $R_2 = (2\log{2}/\tilde{c})^{1/\delta}$. Now \eqref{eq:preliminary_drift_td_lambda} implies
\begin{align*}
\MK V(x_{0:\tau}) \leq \rme^{-c_{\MK}W^{\delta}(x_{0:\tau})} V(x_{0:\tau})\1_{\{W(x_{0:\tau}) \geq R_{\MK}\}} + \bb_{\MK}\1_{\{W(x_{0:\tau}) < R_{\MK}\}}\eqsp,
\end{align*}
where we have defined
\begin{equation}
\label{eq:b_R_prime_prime_def_td_lambda}
c_{\MK} = \tilde{c}/2 = (1-\beta_0\tau )c/2, \quad
\bb_{\MK} = \sup_{0 < r < R_1 \vee R_2}\bigl\{\rme^{-\tilde{c}r^{\delta} + r} + \bb\rme^{\beta_0 c r^{\delta}}\bigr\}, \quad R_{\MK} = R_1 \vee R_2 \eqsp.
\end{equation}
Now let us define, for $R \geq 1$, the sublevel sets $C_R = \{x \in \Xset: \tilde{W}(x) \leq R\},\, \tilde{C}_R = \{x_{0:\tau} \in \Xset^{\tau+1}: W(x_{0:\tau}) \leq R\}$. To check that $\tilde{C}_R$ is small, we proceed similarly to \Cref{lem: tdlambda irredu}. For any $D_1,\dots,D_{\tau+1} \in \mathcal{X}$, and $x_{0:\tau} \in \tilde{C}_R$,
\begin{align*}
&\MK^{\tau+1}(x_{0:\tau},D_1 \times \dots \times D_{\tau+1}) = \int \prod_{k=0}^{\tau} \MKQ(x_{\tau + k},\rmd x_{\tau + k+1}) \prod_{k=1}^{\tau+1}\indi{D_k}(x_{\tau + k}) \\
&\quad \geq \int \prod_{k=0}^{\tau} \MKQ(x_{\tau + k},\rmd x_{\tau + k+1}) \prod_{k=1}^{\tau+1}\indi{D_k \cap C_{R}}(x_{\tau + k})
\overset{(a)}{\geq} \varepsilon_{R}^{\tau+1} \prod_{k=1}^{\tau+1}\nu(D_k \cap C_R) \\
&\quad \geq \bigl(\varepsilon_R \nu(C_R)\bigr)^{\tau+1}\nu_{C_R}(D_1 \times \dots \times D_{\tau+1}) \eqsp,
\end{align*}
where (a) follows from $(\tau + 1)$ applications of the fact that $C_R$ is $(1,\varepsilon_R \nu)$-small for $\MKQ$ and
\begin{equation}
\label{eq:gamma_C_R_definition}
\gamma_{C_R}(D_1 \times \dots \times D_{\tau+1}) = \prod_{k=1}^{\tau+1}\nu(D_k \cap C_R) / \nu(C_R) \eqsp. \
\end{equation}
Hence, $\tilde{C}_R$ is $\bigl(\tau+1, (\varepsilon_R\nu(C_R))^{\tau+1}\nu_{C_R}\bigr)$-small for the Markov kernel $\MK$.
\end{proof}

\end{document}